\documentclass{article} 
\usepackage{iclr2027_conference,times}

\usepackage[utf8]{inputenc} 
\usepackage[T1]{fontenc}    
\usepackage{hyperref}       
\usepackage{url}            
\usepackage{booktabs}       
\usepackage{amsfonts}       
\usepackage{nicefrac}       
\usepackage{microtype}      
\usepackage{xcolor}         
\usepackage{amsmath,amsthm}
\usepackage{enumitem}
\usepackage{thm-restate}
\usepackage{amssymb}
\usepackage{wrapfig}

\usepackage[most]{tcolorbox}
\newtheorem{theorem}{Theorem}[section]
\newtheorem{proposition}[theorem]{Proposition}

\newtheorem{remark}[theorem]{Remark}

\usepackage{tikz}
\usetikzlibrary{arrows.meta,positioning,calc,decorations.pathreplacing}

\definecolor{figblue}{RGB}{38,92,158}
\definecolor{figgreen}{RGB}{38,108,78}
\definecolor{figgray}{RGB}{112,120,132}

\newcommand{\cC}{\mathcal{C}}

\newcommand{\cR}{\mathcal{R}}
\newcommand{\cS}{\mathcal{S}}

\newcommand{\cX}{\mathcal{X}}

\newcommand{\R}{\mathbb{R}}

\newcommand{\E}{\mathbb{E}}
\newcommand{\Prob}{\mathbb{P}}

\DeclareMathOperator*{\argmin}{arg\,min}

\newcommand{\Var}{\mathrm{Var}}
\newcommand{\Cov}{\mathrm{Cov}}
\newcommand{\Tr}{\mathrm{Tr}}
\newcommand{\err}{\mathrm{err}}
\newcommand{\sign}{\textnormal{sign}}

\newcommand{\Ur}{\mathcal{U}_r}

\title{Which Tasks Survive Self-Supervised \\ Learning?}

\author{Achleshwar Luthra \hspace{1em} Lucas Bryant \hspace{1em} Tracy Zhu \hspace{1em} Tomer Galanti\thanks{Correspondence to \texttt{galanti@tamu.edu}}  \\
Department of Computer Science and Engineering\\
Texas A\&M University \\
}

\iclrfinalcopy 
\begin{document}

\maketitle

\begin{abstract}
Same-instance self-supervised learning (SSL) learns representations by enforcing consistency across two views of the same underlying instance. This principle alone, however, does not determine which downstream tasks remain recoverable from the learned representation. We study this question through \emph{semantic recoverability}, defined as the amount of a task's posterior score captured by the represented function space. We show that, for centered and whitened representations, recoverability exactly determines directional class-distance-normalized variance (CDNV), controls few-shot nearest-centroid classification, and governs the strength of task-relevant semantic directions. The population linear probe and centroid axis coincide, and multiple well-recovered tasks approach a factorial centroid geometry. We then analyze a canonical two-view SSL objective and show that its population optimum spans the leading cross-view-stable modes of the associated two-view operator. This yields a closed-form spectral characterization of semantic recoverability: a downstream task is preserved to the extent that its posterior lies in the selected spectral subspace. We validate these predictions on synthetic and real datasets across several SSL methods, testing the predicted relationships among recoverability, directional geometry, spectral structure, and few-shot transfer. Together, these results give a task-level account of what information survives same-instance SSL and how the retained information appears in downstream geometry and transfer.
\end{abstract}

\section{Introduction}
\label{sec:intro}

Self-supervised learning (SSL)~\citep{balestriero2023cookbookselfsupervisedlearning} learns representations without labels and performs well across vision, language, speech, and multimodal learning~\citep{pmlr-v119-chen20j,He_2020_CVPR,zbontar2021barlow,he2022masked,oquab2024dinov2learningrobustvisual,gao-etal-2021-simcse,reimers-gurevych-2019-sentence,schneider2019wav2vecunsupervisedpretrainingspeech,10.5555/3495724.3496768,10.1109/TASLP.2021.3122291,pmlr-v162-baevski22a,pmlr-v139-radford21a,pmlr-v139-jia21b,Zhai_2023_ICCV,tschannen2025siglip2multilingualvisionlanguage}. A common principle is simple: two views of the same instance should receive compatible representations. Yet two views can share many kinds of information: semantic content, style, background, nuisance structure, or accidental regularities. Agreement alone therefore does not reveal which downstream distinctions survive.

This leads to a basic task-level question: given a representation learned only from paired views, which semantic tasks can still be recovered from it? For example, a representation learned from color-augmented face images might preserve whether a person is smiling while losing information about hair color. Existing theory provides two pieces of the answer. Operator and spectral analyses characterize which cross-view-stable directions are preferred by idealized SSL objectives~\citep{zhai2025contextures,zhai2025contexturesmechanismrepresentationlearning,pmlr-v139-ermolov21a,weng2022an}, while geometric analyses show that SSL representations can develop strong label-aligned directional structure without global class collapse~\citep{luthra2025selfsupervisedcontrastivelearningapproximately,wang2020understanding,garrido2023on,balestriero2022contrastive,pmlr-v202-qiu23a}. We therefore ask:

\begin{tcolorbox}[colback=blue!5!white, colframe=blue!10!black, arc=1pt]
\centering
\textbf{\em Which downstream tasks survive same-instance SSL, and what geometry and few-shot behavior follow from their recoverability?}
\end{tcolorbox}

Our starting point is that this question is task-dependent. A representation can preserve one semantic partition while discarding another. For a fixed downstream task, the relevant object is the part of its posterior score that remains visible in the represented function space, as illustrated in Fig.~\ref{fig:Bf_illustration}. We call this quantity \emph{captured posterior energy}, \(B(F)\), and use it to measure semantic recoverability independently of how the representation was learned.

Our theory addresses two questions: why semantic recoverability matters, and when SSL makes a task recoverable. Our contributions are threefold. {\bf (i)} We formulate task survival through captured posterior energy, separating recoverability from any particular SSL objective. {\bf (ii)} For centered and whitened representations, we derive an exact relation between recoverability and directional class-distance-normalized variance (CDNV)~\citep{luthra2025selfsupervisedcontrastivelearningapproximately}, establish probe--centroid alignment, obtain a direct finite-sample guarantee for few-shot nearest-centroid classification, and characterize the joint centroid geometry of multiple well-recovered tasks. {\bf (iii)} For a canonical two-view population SSL objective, we show that the learned subspace is spanned by the leading cross-view-stable modes of the associated two-view operator. Consequently, semantic recoverability is exactly the spectral overlap between the task posterior and the selected subspace, yielding a criterion for which downstream tasks survive SSL. We further extend this subspace perspective to covariance-regularized and prediction-based SSL surrogates in the appendix.

\subsection{Related Work}\label{sec:related}

Our work connects two lines of SSL theory that are usually studied separately: task-relevant representation geometry and the population subspaces selected by two-view objectives.

{\bf Directional geometry and downstream transfer.\enspace}
SSL representations can support strong downstream classification and retrieval, while remaining globally anisotropic and retaining substantial variance in nuisance directions~\citep{pmlr-v119-chen20j,Caron_2021_ICCV,shaul2023reverse,weng2025clusteringpropertiesselfsupervisedlearning,oquab2024dinov2learningrobustvisual,wang2020understanding,wang2021understanding,chen2021intriguing,NEURIPS2021_27debb43,arora2019theoreticalanalysiscontrastiveunsupervised}. This motivates task-aware geometric measures rather than global collapse. In particular, directional CDNV isolates within-class variation along the semantic decision direction and has been shown to track linear-probe and few-shot transfer across SSL methods~\citep{luthra2025selfsupervisedcontrastivelearningapproximately}. Related work on neural collapse and few-shot transfer connects class geometry to generalization~\citep{doi:10.1073/pnas.2015509117,han2022neural,galanti2022on,galanti2022improved,galanti2023generalizationboundsfewshottransfer,10377311,pmlr-v119-goldblum20a}. Our goal is complementary: we ask what property of a label-free learned subspace predicts whether this favorable task geometry will appear.

{\bf Spectral and operator views of SSL.\enspace}
Many SSL methods combine cross-view agreement with mechanisms that prevent representational collapse~\citep{wang2020understanding,zbontar2021barlow,bardes2022vicreg,garrido2023on,balestriero2022contrastive,pmlr-v139-ermolov21a}. At the population level, whitening-style formulations connect naturally to CCA, maximal correlation, and conditional expectation operators~\citep{hotelling1936relations,Hannan_1961,breiman1985estimating,lai2000kernel,6788402,fukumizu07a,pmlr-v28-andrew13,pmlr-v48-michaeli16,asoodeh2015maximalcorrelationmutualinformation}. Recent analyses characterize the subspace selected by such objectives through the leading spectrum of an induced two-view operator~\citep{zhai2025contextures,zhai2025contexturesmechanismrepresentationlearning}. We use this characterization for computing which task posteriors are recoverable from the selected subspace.

{\bf Relation to broader SSL theory.\enspace}
A broader literature studies why SSL captures useful structure through mutual-information views~\citep{NEURIPS2019_ddf35421,pmlr-v108-mcallester20a,Tschannen2020On}, alignment and uniformity~\citep{wang2020understanding,wang2021understanding,chen2021intriguing}, latent-variable recovery and augmentation structure~\citep{arora2019theoreticalanalysiscontrastiveunsupervised,tosh2021contrastive,zimmermann2021contrastive,pmlr-v151-ash22a,NEURIPS2021_2dace78f,NEURIPS2021_27debb43,pmlr-v162-shen22d,pmlr-v162-awasthi22b,pmlr-v162-saunshi22a}, and optimization, projection heads, and sample complexity~\citep{haochen2023a,pmlr-v195-parulekar23a,pmlr-v139-wen21c,tian2023understanding,NEURIPS2020_4c2e5eaa,gupta2022understandingimprovingroleprojection,gui2023unravelingprojectionheadscontrastive,alon2024optimal}. Our question is narrower and task-conditional: given the information retained by a learned representation, which downstream posterior scores survive, and what geometric and few-shot consequences follow?

\section{Problem Setup}
\label{sec:setup}

We consider a latent instance that generates an observed sample and two conditionally independent views, together with a semantic task defined at the instance level. A representation is learned from paired views without access to the task labels and is then evaluated by how well the task can be recovered from a small labeled support set.

\subsection{Latent Instance Model and Two-View Generation}
\label{subsec:latent_instance_model}

Let \(C\) denote the \emph{latent instance} (for example, the underlying object, scene, or document) taking values in a measurable space \(\cC\), with distribution \(P_C\). The semantic label is determined at the instance level: $Y ~=~ y(C)\in\{\pm 1\}$, $\Prob(Y=+1)~=~\Prob(Y=-1)~=~\tfrac{1}{2}$. Thus \(C\) may contain much more information than \(Y\): many latent instances can share the same label.

We let \(X^{(0)}\) denote the sample associated with \(C\), drawn as \(X^{(0)}\sim P_0(\cdot\mid C)\). A positive pair is formed by drawing two independent views, \(X^{(1)},X^{(2)}\stackrel{\mathrm{i.i.d.}}{\sim}P(\cdot\mid X^{(0)})\). In vision, for example, \(C\) may encode image content, while each view is produced by a random augmentation such as cropping, color jitter, or blur. Let \(P_X\) denote the marginal law of a single view.

\subsection{Downstream Few-Shot Learning}
\label{subsec:directional_geometry}

Let \(F:\cX\to\R^r\) be a representation. We evaluate its downstream quality through the expected \(m\)-shot error of a classifier built on top of \(F\). Given a support set with \(m\) labeled examples per class, with all support examples sampled independently, $S_m=\{(X_{s,i},s): s\in\{\pm1\}$, $X_{s,i}\sim P(\cdot\mid Y=s)$, ($i\in[m]$)$\}$, let \(\hat g^{\mathrm{A}}_{S_m}\) denote the classifier returned by a learning rule \(\mathrm{A}\) from the embedded support set. We define $\textnormal{err}^{\mathrm{A}}_{m}(F)
:=
\E_{S_m}\!\left[
\Prob_{(X,Y)}\!\big(\hat g^{\mathrm{A}}_{S_m}(F(X))\neq Y\big)\right]$,
where \((X,Y)\) is drawn from the single-view distribution induced by the latent-instance model.

We focus on nearest-class-centroid (NCC) classification: $g^{\mathrm{NCC}}_{S_m}(z) := \argmin_{s\in\{\pm1\}} \|z-\widehat \mu_s\|_2$, where $\widehat \mu_s := \frac{1}{m}\sum_{i=1}^m F(X_{s,i})$. This gives an operational notion of task survival: a useful representation should make the task learnable from only a small labeled support set. The next section identifies a population quantity that controls when this happens.

\section{Semantic Recoverability and Its Consequences}
\label{sec:semantic_signal}

\definecolor{figgray}{RGB}{75,82,92}
\definecolor{figblue}{RGB}{25,95,210}



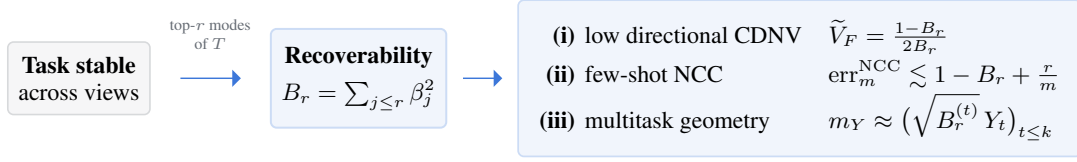
\begin{figure}[t]
\centering
\usetikzlibrary{shapes.arrows, calc}

\begin{tikzpicture}[
  font=\footnotesize,
  stage/.style={
    rounded corners=2.2pt,
    draw=figgray!22,
    line width=0.6pt,
    fill=figgray!5,
    align=center,
    inner xsep=4pt,
    inner ysep=5.5pt
  },
  key/.style={
    rounded corners=2.2pt,
    draw=figblue!22,
    line width=0.75pt,
    fill=figblue!6,
    align=center,
    inner xsep=4pt,
    inner ysep=6.5pt
  },
  resultbox/.style={
    rounded corners=2.2pt,
    draw=figblue!22,
    line width=0.65pt,
    fill=figblue!6,
    inner xsep=8pt,
    inner ysep=6pt
  },
  blockarrow/.style={
    single arrow,
    single arrow head extend=2.2pt,
    fill=figblue!85,
    draw=none,
    minimum width=6.5pt,   
    minimum height=1cm,    
    inner ysep=0pt,
    anchor=center
  },
  tag/.style={
    font=\tiny,
    text=figgray!85,
    align=center
  }
]

\node[stage] (A) at (0,0)
  {\textbf{Task stable}\\across views};

\node[key] (B) at (3.70,0)
  {\textbf{Recoverability}\\[4pt]
   $B_r=\sum_{j\le r}\beta_j^2$};

\node[resultbox, anchor=west] (C) at (5.80,0)
{%
\setlength{\tabcolsep}{0pt}%
\renewcommand{\arraystretch}{1.45}%
\begin{tabular}{@{}r@{\hspace{0.4em}}l@{\hspace{1.1em}}l@{}}
\textbf{(i)}   & low directional CDNV
  & $\widetilde V_F=\frac{1-B_r}{2B_r}$\\
\textbf{(ii)}  & few-shot NCC
  & $\mathrm{err}^{\mathrm{NCC}}_m\lesssim 1-B_r+\frac{r}{m}$\\
\textbf{(iii)} & multitask geometry
  & $m_Y\approx\bigl(\sqrt{B_r^{(t)}}\,Y_t\bigr)_{t\le k}$
\end{tabular}%
};

\node[blockarrow, minimum height=0.80cm]
  at ($(A.east)!0.5!(B.west)$) {};

\node[tag] at ($(A.east)!0.5!(B.west) + (0,0.58)$)
  {top-$r$ modes\\of $T$};

\node[blockarrow, minimum height=0.52cm]
  at ($(B.east)!0.5!(C.west)$) {};

\end{tikzpicture}

\caption{%
Conceptual flow of the theory. Same-instance SSL retains the most cross-view-stable
modes of the two-view operator $T$ (Prop.~\ref{prop:same_instance_ssl_posterior}), so a
task survives to the extent that its posterior overlaps that subspace
(Cor.~\ref{cor:posterior_decomp_ssl_basis}). Under centered whitening, $B_r$ fixes the
directional collapse law (Prop.~\ref{prop:Br_probe_dcdnv}) and few-shot nearest-centroid
guarantee (Thm.~\ref{thm:ncc_bound_direct_via_Br}); the task-specific $B_r^{(t)}$ determine
the multitask centroid geometry (Thm.~\ref{thm:axis_and_centroids}).}
\label{fig:theory_flow}
\end{figure}

\subsection{Captured Posterior Energy}
\label{subsec:posterior_projection}

For the balanced task, define the posterior score \(\eta(x):=\E[Y\mid X=x]\). The posterior contains the information about \(Y\) available from a single view. For a representation \(F=(F_1,\dots,F_r)\), let \(\cS_F:=\mathrm{span}\{F_1,\dots,F_r\}\subset L^2(P_X)\) denote its represented function space. We define the \emph{captured posterior energy} \(B(F):=\|\Pi_{\cS_F}\eta\|_{L^2(P_X)}^2\), where \(\Pi_{\cS_F}\) is the orthogonal projection onto \(\cS_F\).

This definition does not require SSL or whitening. It asks directly how much of the Bayes-relevant posterior score remains in the function space represented by \(F\). The dependence on the downstream task is intentional: the same representation may preserve one semantic partition and discard another. If one wants a literal fraction of the single-view task signal that is retained, it is \(B(F)/\|\eta\|_{L^2(P_X)}^2\). In particular, large \(B(F)\) means that most of the posterior is expressible through the representation, while small \(B(F)\) means that the task is largely absent from the represented subspace.

The quantity is invariant to invertible changes of coordinates in feature space because such changes leave \(\cS_F\) unchanged. Whitening enters below for another reason: it converts this intrinsic subspace quantity into exact Euclidean statements about feature geometry and few-shot classification.

\subsection{Why Recoverability Matters Under Whitening}
\label{sec:recoverability_consequences}

Assume for the remainder of this section that \(F\) is centered and whitened, so that \(\E[F(X)]=0\) and \(\E[F(X)F(X)^\top]=I_r\). The coordinate functions of \(F\) are then orthonormal in \(L^2(P_X)\), and the intrinsic recoverability quantity has the simple feature-space form \(B(F)=\|\E[YF(X)]\|_2^2\).

Let \(\mu_s:=\E[F(X)\mid Y=s]\), \(\Sigma_s:=\Cov(F(X)\mid Y=s)\), and \(\Delta:=\mu_+-\mu_-\). When \(\Delta\neq0\), let \(u:=\Delta/\|\Delta\|_2\). The CDNV (\(V_F\)) and directional CDNV (\(\tilde{V}_F\)) are
\begin{equation}
\label{eq:problem_dcdnv_vector_binary}
V_F := (\Tr(\Sigma_+)+\Tr(\Sigma_-))/\|\Delta\|_2^2,
\qquad
\tilde V_F := (u^\top \Sigma_+u+u^\top \Sigma_-u)/\|\Delta\|_2^2.
\end{equation}
When \(\Delta=0\), we set \(V_F=\tilde V_F=+\infty\). Under whitening, total within-class variation can remain large even when variation along the semantic direction is small.

\subsubsection{Exact directional-collapse law}
\label{subsec:sslnc1}

\begin{restatable}{proposition}{cdnv}
\label{prop:Br_probe_dcdnv}
Let \(F:\cX\to\R^r\) be centered and whitened, and let \(Y\in\{\pm1\}\) be balanced. Then \(V_F=(r-B(F))/(2B(F))\) and \(\tilde V_F=(1-B(F))/(2B(F))\), with both quantities equal to \(+\infty\) when \(B(F)=0\).
\end{restatable}

Prop.~\ref{prop:Br_probe_dcdnv} gives an exact population law: directional CDNV vanishes as recoverability approaches one and diverges as recoverability vanishes. Thus the representation need not collapse globally for the task to become geometrically simple; it only needs to collapse along the semantic direction that survives in the represented subspace.

\begin{remark}[Directional vs.\ global collapse]
\label{rem:directional_vs_global}
Even at \(B(F)=1\), the global CDNV is \(V_F=(r-1)/2\), large whenever \(r\ge3\). Whitening fixes the total representation variance, so within-class spread remains in directions orthogonal to the task axis. Global collapse is therefore neither expected nor necessary. When the label is fully determined by a single view, \(Y=g(X)\), we have \(B(F)=1\) if and only if \(\eta\in\cS_F\), giving perfect directional collapse despite potentially large global CDNV.
\end{remark}

\subsubsection{Probe--centroid alignment}
\label{subsec:sslnc3}

A second consequence of whitening is that the population squared-loss probe and the centroid decision axis coincide.

\begin{restatable}[MSE probe and decision-axis alignment]{corollary}{duality}
\label{cor:nc3_mse_duality}
Let \(F:\cX\to\R^r\) be centered and whitened, and let \(Y\in\{\pm1\}\) be balanced. Define \(\mu_s:=\E[F(X)\mid Y=s]\) for \(s\in\{\pm1\}\), and let \(\Delta:=\mu_+-\mu_-\). Then the unique population MSE probe, \((w^\star,b^\star)\in\argmin_{w\in\R^r,b\in\R}\E[(Y-(w^\top F(X)+b))^2]\), satisfies \(b^\star=0\) and \(w^\star=\E[YF(X)]=\tfrac12\Delta\). In particular, \(\mu_-=-\mu_+\), and \(\Delta=2\mu_+=2w^\star\).
\end{restatable}

Cor.~\ref{cor:nc3_mse_duality} shows that linear probing and NCC use the same separating direction. Moreover, \(\|w^\star\|_2^2=B(F)\), so the strength of the population probe is itself a measurement of semantic recoverability.

\subsubsection{Few-shot transfer guarantees}
\label{subsec:sslnc4}

The geometric identities above translate recoverability into an operational downstream guarantee. In the whitened setting, the same scalar that controls directional geometry also controls the expected error of an \(m\)-shot nearest-centroid classifier.

\begin{restatable}[Few-shot NCC bound via captured posterior energy]{theorem}{nccBound}
\label{thm:ncc_bound_direct_via_Br}
Let \(F:\cX\to\R^r\) be centered and whitened, and let \(Y\in\{\pm1\}\) be balanced. Then, for every \(m\ge1\),
\begin{small}
\[
\mathrm{err}^{\mathrm{NCC}}_{m}(F)
\le
1-B(F)+\frac{r-B(F)}{m}
+\frac{1-B(F)}{1-B(F)+2mB(F)}
=
\frac{2\tilde V_F}{1+2\tilde V_F}
+\frac{(r-1)+2r\tilde V_F}{m(1+2\tilde V_F)}
+\frac{\tilde V_F}{\tilde V_F+m},
\]
\end{small}
where the second expression is understood by continuity when \(B(F)=0\).
\end{restatable}

The bound separates two effects. The leading term reflects task information not captured by the representation, while the remaining terms arise from estimating class centroids from finitely many examples. Whitening is what makes this a one-scalar Euclidean statement: without it, coordinate rescaling could change \(\|\E[YF]\|_2^2\) and Euclidean NCC would no longer coincide with the natural population geometry. The coordinate-invariant analogue is obtained by rewhitening, or equivalently by using a Mahalanobis metric in the original feature space.

\subsubsection{Multitask semantic geometry}
\label{subsec:sslnc2}

Recoverability also determines how several semantic tasks are organized jointly. We use a superscript to index tasks, reserving the subscript \(r\) for representation rank. Consider \(k\) balanced binary tasks \(Y_1,\dots,Y_k\in\{\pm1\}\), and write \(Y:=(Y_1,\dots,Y_k)\). For each task \(t\), let \(\eta_t(x)=\E[Y_t\mid X=x]\), \(B^{(t)}(F)=\|\Pi_{\cS_F}\eta_t\|_{L^2(P_X)}^2\), and \(\varepsilon_t=\E[\Var(Y_t\mid X)]=1-\|\eta_t\|_{L^2(P_X)}^2\). Let \(w_t=\E[Y_tF(X)]\). Assuming \(B^{(t)}(F)>0\), define \(u_t=w_t/\|w_t\|_2\) and \(Z_t(X)=\langle u_t,F(X)\rangle\).

\begin{restatable}[Near-orthogonality and centroid geometry]{theorem}{rectangle}
\label{thm:axis_and_centroids}
Let \(F:\cX\to\R^r\) be centered and whitened. Assume each \(Y_t\in\{\pm1\}\) is balanced and \(B^{(t)}(F)>0\) for all \(t\); arbitrary dependence among \(Y_1,\dots,Y_k\) is allowed. For \(i\neq j\), define \(\rho_{ij}:=\E[Y_iY_j]\). Then, for any \(i\neq j\),
\[
|u_i^\top u_j|
~\le~
\left(|\rho_{ij}|
+\sqrt{\varepsilon_i\varepsilon_j}
+\sqrt{\big(\|\eta_i\|_{L^2}^2-B^{(i)}(F)\big)\big(\|\eta_j\|_{L^2}^2-B^{(j)}(F)\big)}\right)/
\sqrt{B^{(i)}(F)B^{(j)}(F)}.
\]
Moreover, letting \(m_Y:=\E[Z(X)\mid Y]\in\R^k\),
\[
\E\Big[\big\|m_Y-\big(\sqrt{B^{(1)}(F)}Y_1,\dots,\sqrt{B^{(k)}(F)}Y_k\big)\big\|_2^2\Big]
~\le~
\sum_{t=1}^k \big(1-B^{(t)}(F)\big).
\]
\end{restatable}

\looseness=-1 
Thm.~\ref{thm:axis_and_centroids} shows that well-recovered tasks induce strong semantic axes that become nearly orthogonal when task correlations and residual posterior uncertainty are small. Their joint centroids approach \(\big(\sqrt{B^{(t)}(F)}Y_t\big)_{t\le k}\), yielding an approximately axis-aligned hyperrectangle whose side lengths are set by the task-specific capture values. Thus recoverability controls the strength of each semantic direction, while cross-task dependence determines how closely the overall geometry approaches a factorial product.

The results so far deliberately make no claim about how \(F\) was learned. They say that \emph{if} a representation captures a task, then \(B(F)\) determines its directional geometry, its few-shot behavior, and its role in multitask centroid structure. We now turn to the complementary question: {\em why should a same-instance SSL objective produce a large value of \(B(F)\) for some tasks and not others?}

\section{Recoverability Under Same-Instance SSL}
\label{subsec:ssl_specialization}

The preceding results show that the recoverability quantity \(B(F)\) characterizes the task-relevant geometry of a representation and controls its few-shot behavior. The next question is whether, for a representation learned by same-instance SSL, this quantity can itself be identified explicitly from the SSL objective. We show that this is indeed the case: under the whitening-based population objective, the learned representation spans the leading eigenspace of the two-view operator, and \(B(F)\) becomes the spectral overlap between the downstream posterior and these view-stable modes.

\subsection{A Canonical Two-View Population Objective}
\label{sec:population_ssl_objective}

Many SSL methods~\citep{NEURIPS2020_f3ada80d,zbontar2021barlow,bardes2022vicreg,pmlr-v119-chen20j,Caron_2021_ICCV,pmlr-v139-ermolov21a,Tao_2022_CVPR} combine instance-level agreement with an anti-collapse mechanism that promotes global spread, decorrelation, or isotropy~\citep{wang2020understanding,garrido2023on,balestriero2022contrastive}. To obtain a clean population characterization of the represented subspace, we study the canonical whitening-based surrogate (W-MSE)~\citep{pmlr-v139-ermolov21a,weng2022an}
\begin{equation}
\label{eq:ssl_pop_obj_multi}
\max_{F:\cX\to\R^r}
\E\!\left[\langle F(X^{(1)}),F(X^{(2)})\rangle\right]
\quad
\text{subject to}
\quad
\E[F(X)]=0,
\qquad
\E[F(X)F(X)^\top]=I_r.
\end{equation}
The objective rewards agreement across two views of the same instance, while whitening rules out collapse and fixes the feature metric. Its role in our theory is specific: it gives an analyzable rule for which \(r\)-dimensional function space is selected from the two-view distribution.

Strict whitening is the cleanest setting for this identification and for the exact geometry above. App.~\ref{app:soft_whitening} studies a covariance-regularized surrogate related to Barlow Twins~\citep{zbontar2021barlow} and VICReg~\citep{bardes2022vicreg}, while App.~\ref{app:rrr_byol_jepa} gives a corresponding reduced-rank regression view of prediction-based objectives such as BYOL~\citep{NEURIPS2020_f3ada80d}, SimSiam~\citep{Tao_2022_CVPR}, and JEPA~\citep{DBLP:conf/cvpr/AssranDMBVRLB23}.

\begin{figure}
    \centering
    \includegraphics[width=0.95\linewidth]{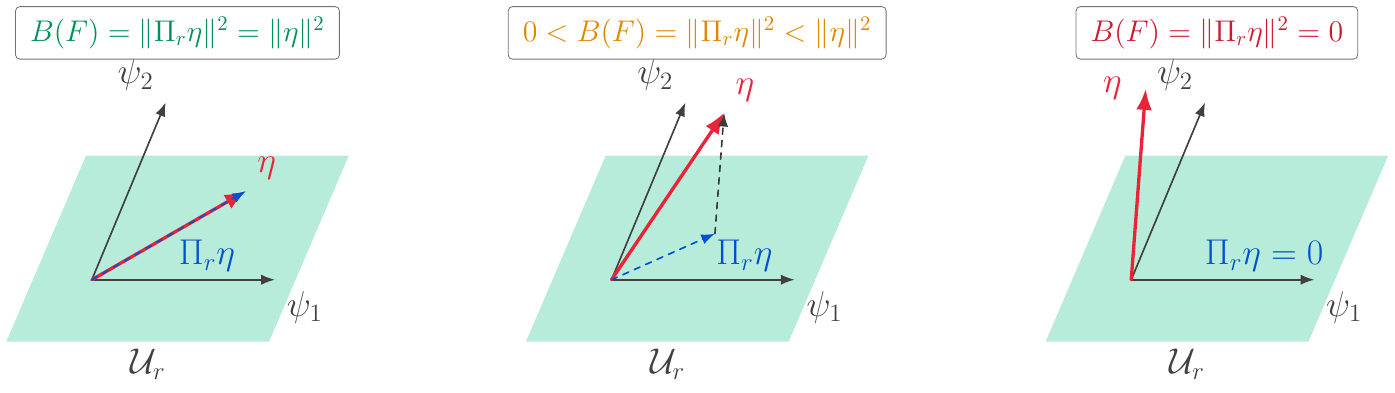}
\caption{\textbf{Geometric interpretation of semantic recoverability.}
The shaded plane is the SSL-selected subspace \(\mathcal U_r=\operatorname{span}\{\psi_1,\psi_2\}\) of leading two-view eigenfunctions. The task posterior \(\eta\) decomposes into its projection \(\Pi_r\eta\) onto this subspace and an orthogonal residual. From left to right, the posterior is fully captured, partially captured, or discarded. The captured posterior energy \(B_r=\|\Pi_r\eta\|_{L^2(P_X)}^2\) measures the retained task-relevant information.}
    \label{fig:Bf_illustration}
\end{figure}

\subsection{The Two-View Operator and the Selected Subspace}
\label{subsec:operator_viewpoint}

Define the two-view conditional expectation operator \(T:L^2(P_X)\to L^2(P_X)\) by $(Tf)(x):=\E[f(X^{(2)})\mid X^{(1)}=x]$. Thus \(T\) maps a function of one view to its conditional expectation from another view of the same instance. In the conditionally i.i.d. two-view model, exchangeability gives \(\langle f,Tg\rangle=\langle Tf,g\rangle\), while conditional independence given \(X^{(0)}\) gives \(\langle f,Tf\rangle=\E[(\E[f(X)\mid X^{(0)}])^2]\ge0\). Hence \(T\) is self-adjoint and positive semidefinite, and an eigenvalue measures the cross-view stability of the corresponding mode.

\begin{restatable}{proposition}{posterior}
\label{prop:same_instance_ssl_posterior}
Let \(L_0^2(P_X):=\{f\in L^2(P_X):\E[f(X)]=0\}\), and assume that \(T\) restricted to \(L_0^2(P_X)\) admits an orthonormal eigenbasis \((\psi_j)_{j\ge1}\), with eigenvalues \(\lambda_1\ge\lambda_2\ge\cdots\ge0\). Assume \(\lambda_r>0\) and \(\lambda_r>\lambda_{r+1}\). Then \(\psi_1,\dots,\psi_r\) are feasible for \eqref{eq:ssl_pop_obj_multi} and achieve the optimal value \(\sum_{j=1}^r \lambda_j\). Moreover, if \(F:\cX\to\R^r\) is any optimal solution of \eqref{eq:ssl_pop_obj_multi}, then its coordinates span the same \(r\)-dimensional subspace as \(\psi_1,\dots,\psi_r\).
\end{restatable}

Prop.~\ref{prop:same_instance_ssl_posterior} says that the population objective keeps the \(r\) most cross-view-stable modes. The representation is determined only up to an orthogonal change of basis within this selected subspace. This result identifies what SSL preserves without referring to any downstream labels; the task enters only when we ask how much of its posterior lies in that subspace.

\subsection{Closed-Form Recoverability at the SSL Optimum}
\label{subsec:spectral_Br}

Let \(F^\star\) be any population-optimal solution of \eqref{eq:ssl_pop_obj_multi}, and define \(\Ur:=\mathrm{span}\{\psi_1,\dots,\psi_r\}\subset L^2(P_X)\), with \(\Pi_r\) denoting the orthogonal projection onto \(\Ur\). For a balanced binary label with posterior score \(\eta(x)=\E[Y\mid X=x]\), write \(\beta_j:=\langle\eta,\psi_j\rangle_{L^2(P_X)}\).

\begin{restatable}{corollary}{decomp}
\label{cor:posterior_decomp_ssl_basis}
Under the assumptions of Prop.~\ref{prop:same_instance_ssl_posterior}, \(\Pi_r\eta=\sum_{j=1}^r\beta_j\psi_j\), and $B(F^\star)
= \|\Pi_r\eta\|_{L^2(P_X)}^2 = \sum_{j=1}^r \beta_j^2 =:B_r$.
\end{restatable}

For multiple downstream tasks, we write \(B_r^{(t)}\) for the rank-\(r\) capture of task \(t\); equivalently, \(B_r^{(t)}=B^{(t)}(F^\star)\) for an optimal rank-\(r\) representation.

This is the closed-form answer to when recoverability is large under the SSL model. The task is preserved when its posterior score is concentrated on the modes that are most stable across the two views, and it is discarded when its posterior lies primarily in view-variant directions outside the selected subspace. The two-view objective itself remains label-free; labels enter only through the overlap used to evaluate a particular downstream task.

Combining this spectral identity with the previous section immediately turns \(B_r\) into task-level predictions. A large spectral overlap implies low directional CDNV, strong few-shot NCC performance, and pronounced semantic axes; across several tasks, the values \(B_r^{(t)}\) determine the side lengths of the approximate hyperrectangle. In the recoverable regime where the augmentations preserve the task and the increasing view-stable subspaces capture its posterior, \(B_r\) approaches one as the representation dimension grows. Thus the overall mechanism is: SSL selects view-stable structure, a task survives according to its posterior overlap with that structure, and the amount that survives determines the downstream geometry and transfer behavior.

\section{Experiments}
\label{sec:exp}

\begin{figure}[t]
    \centering
    \begin{tabular}{cccc}
    \multicolumn{4}{c}{
            \includegraphics[width=0.70\linewidth]{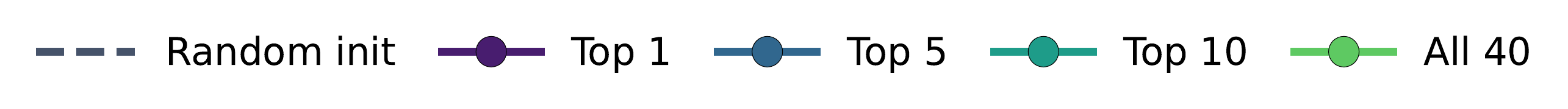}
        } \\[-1mm]
         \includegraphics[width=0.22\linewidth]{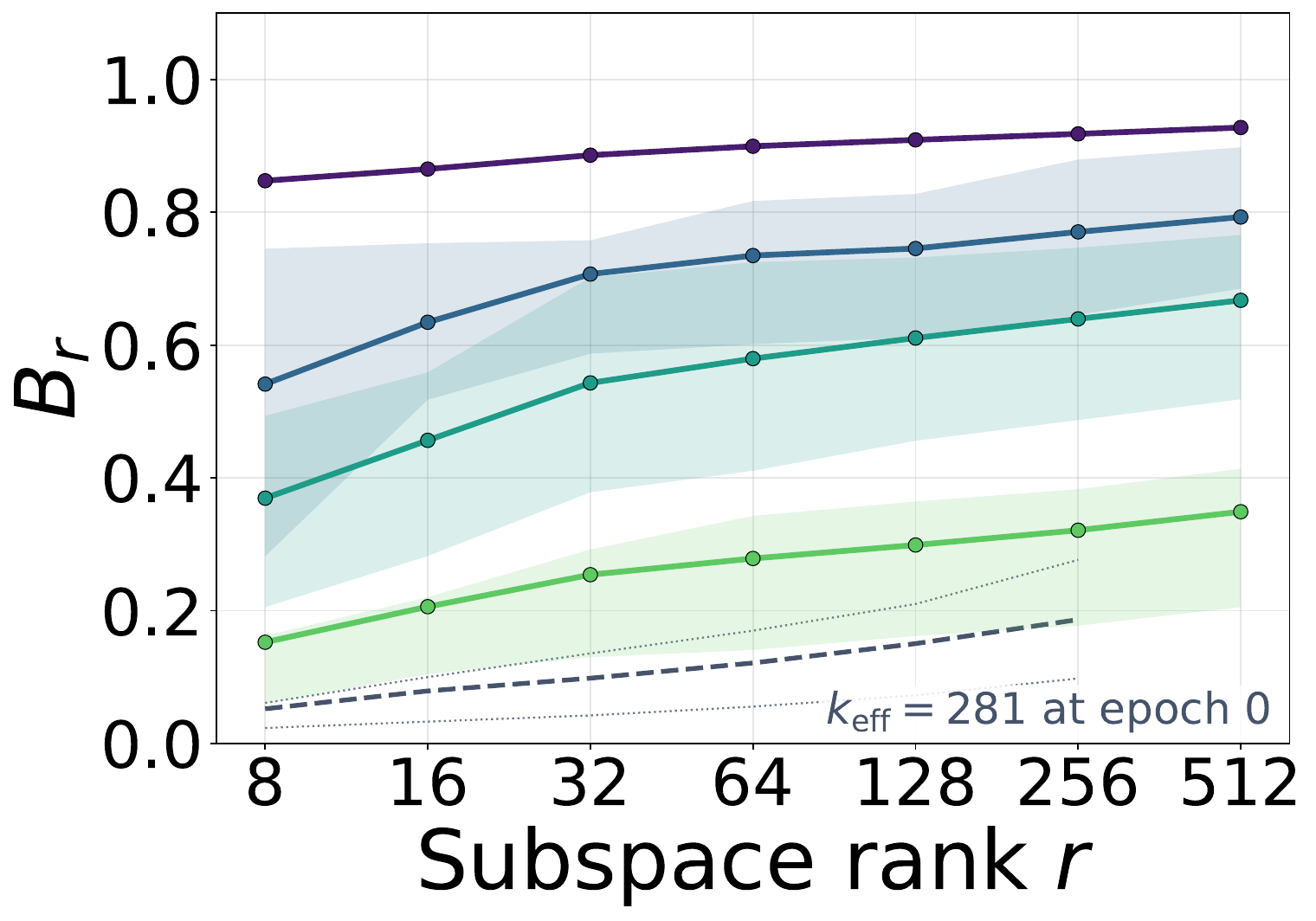} &
         \includegraphics[width=0.22\linewidth]{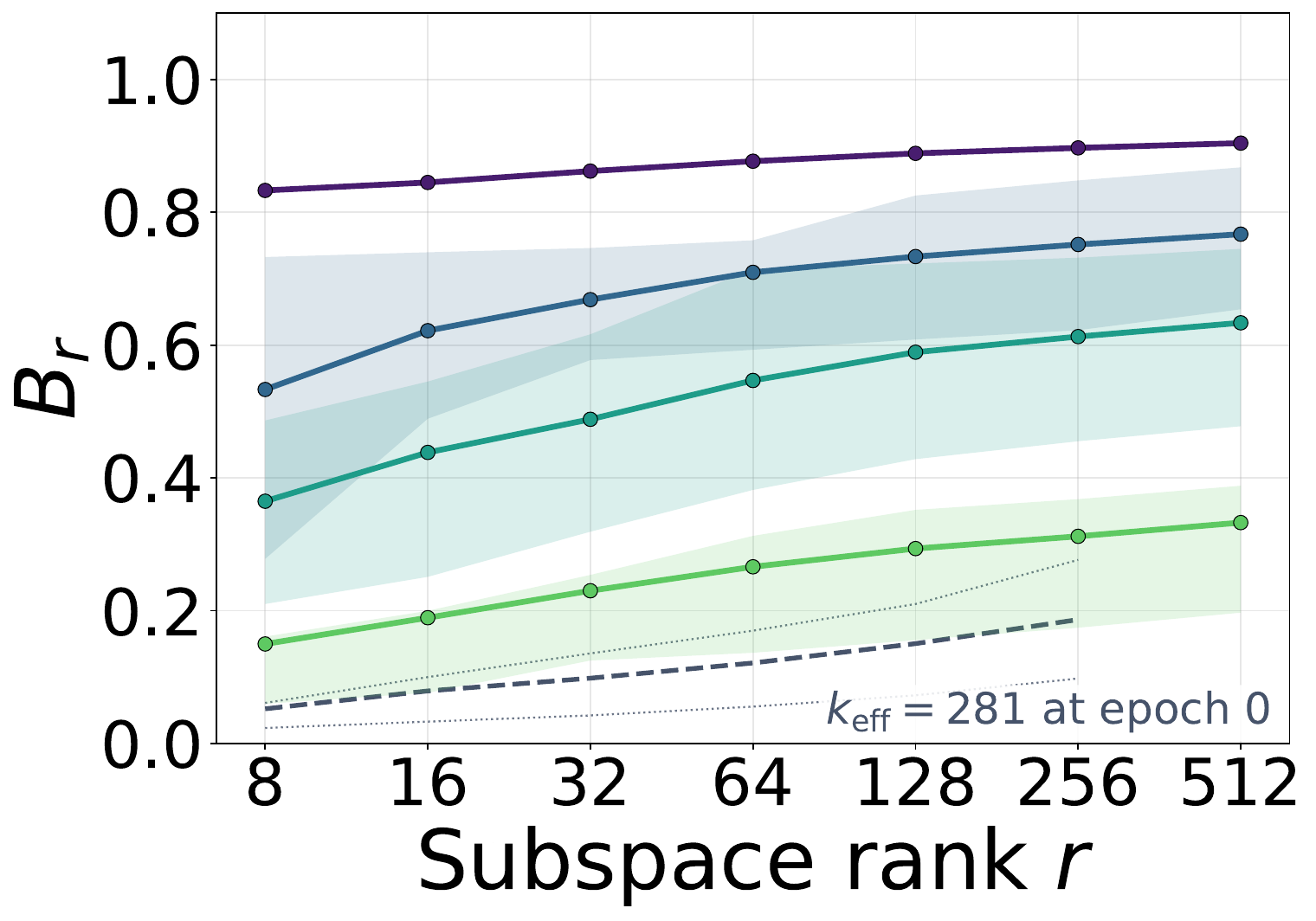} &
         \includegraphics[width=0.22\linewidth]{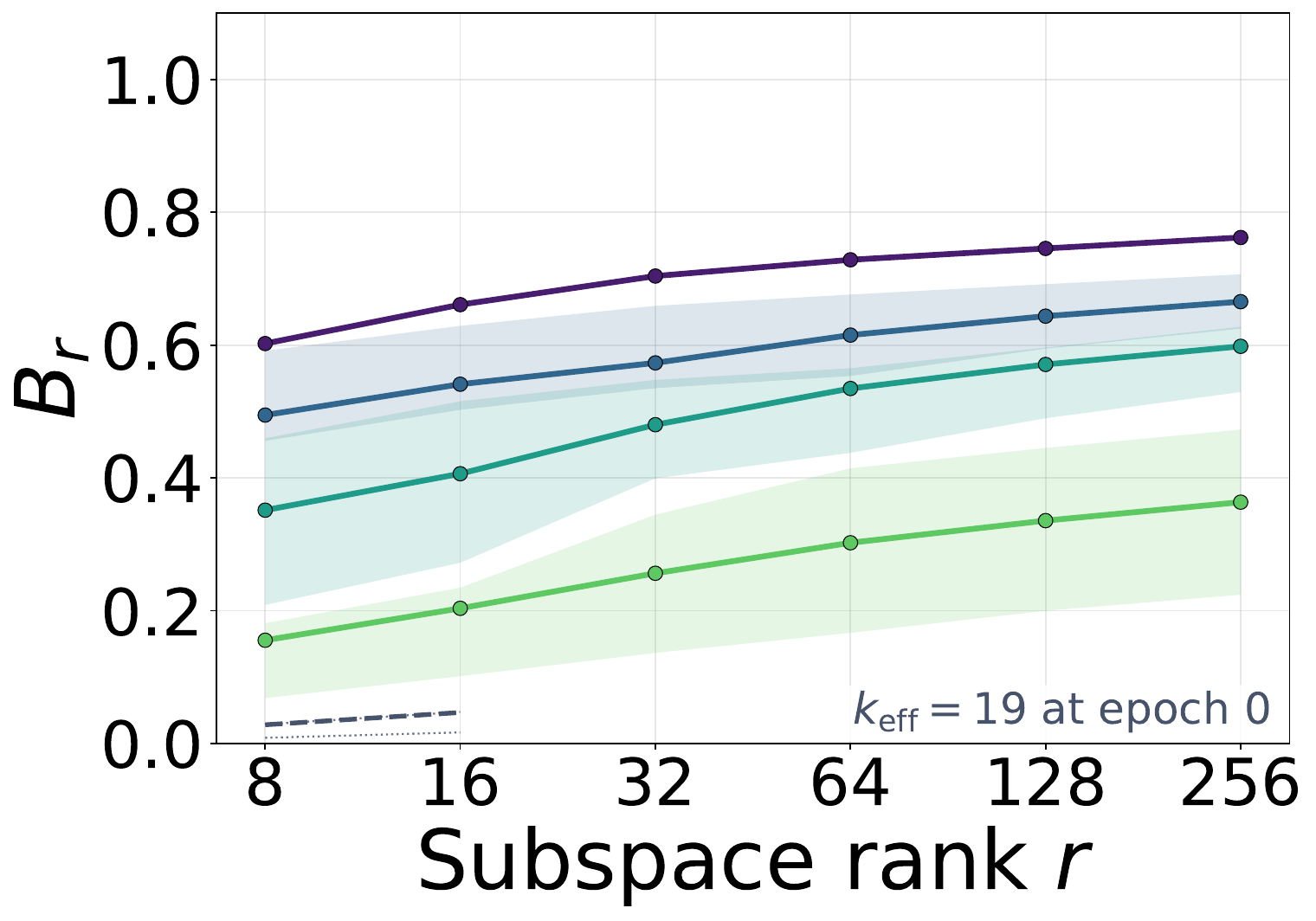} &
         \includegraphics[width=0.22\linewidth]{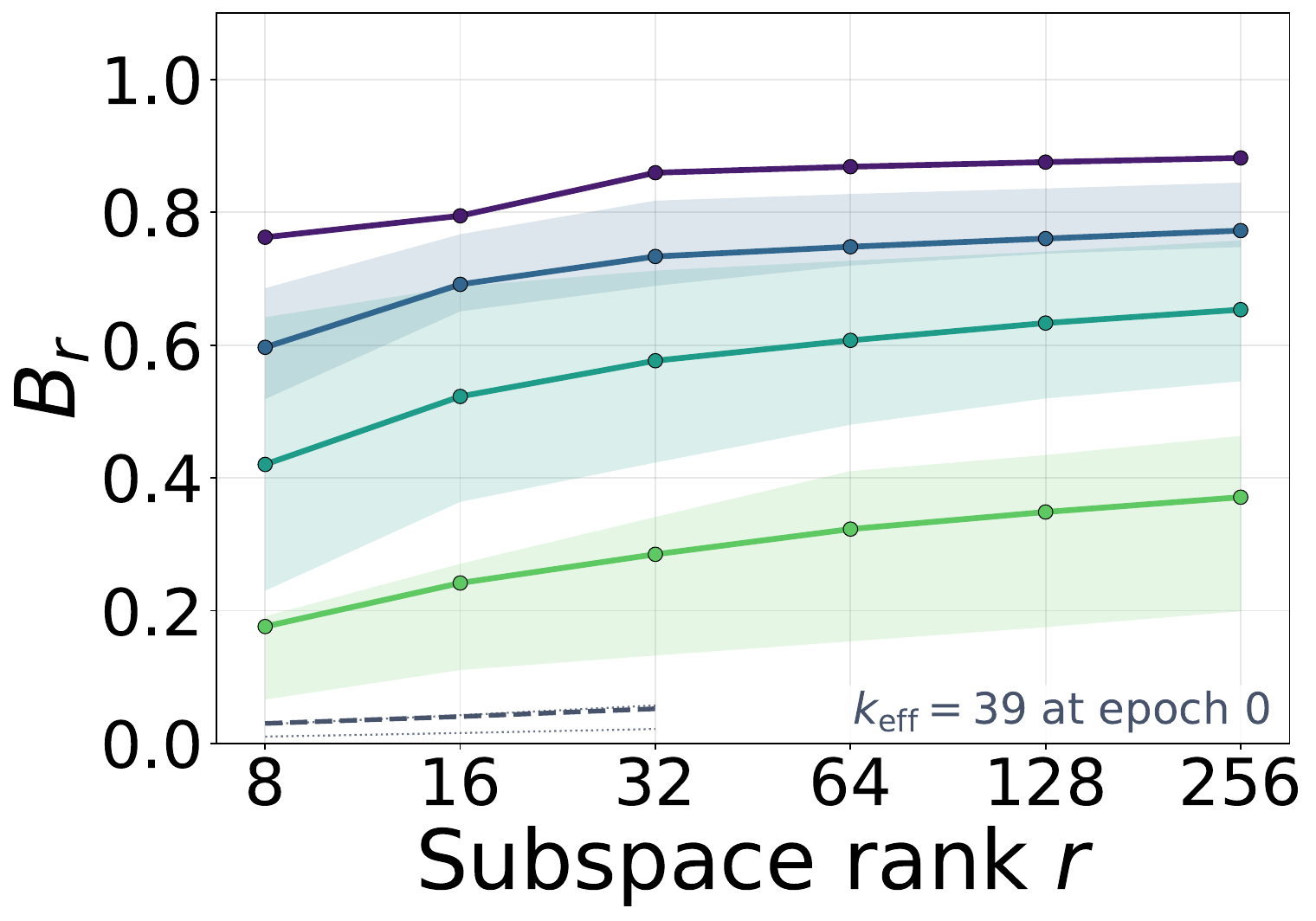} \\
         {\small {\bf (a) W-MSE-CelebA}} &
         {\small {\bf (b) VICReg-CelebA}} &
         {\small {\bf (c) Barlow-IM1K}} &
         {\small {\bf (d) I-JEPA-IM1K}}
    \end{tabular}
    \caption{{\bf Captured posterior energy on CelebA. \enspace}
    Mean \(B_r^{(t)}\) over the 1, 5, and 10 most recoverable CelebA attributes and for all 40 attributes, shown across spectral ranks \(r\). Shaded regions denote the corresponding interquartile ranges. The dashed curve is the mean for a matched randomly initialized encoder, with dotted curves showing its interquartile range.}
    \label{fig:captured_Br}
\end{figure}

\begin{figure}[t]
    \centering
    \begin{tabular}{cccc}
    \multicolumn{4}{c}{
            \includegraphics[width=0.70\linewidth]{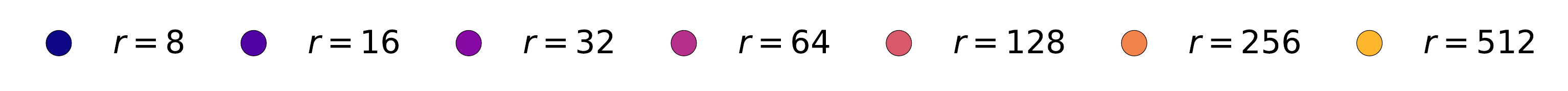}
        } \\[-1mm]
         \includegraphics[width=0.22\linewidth]{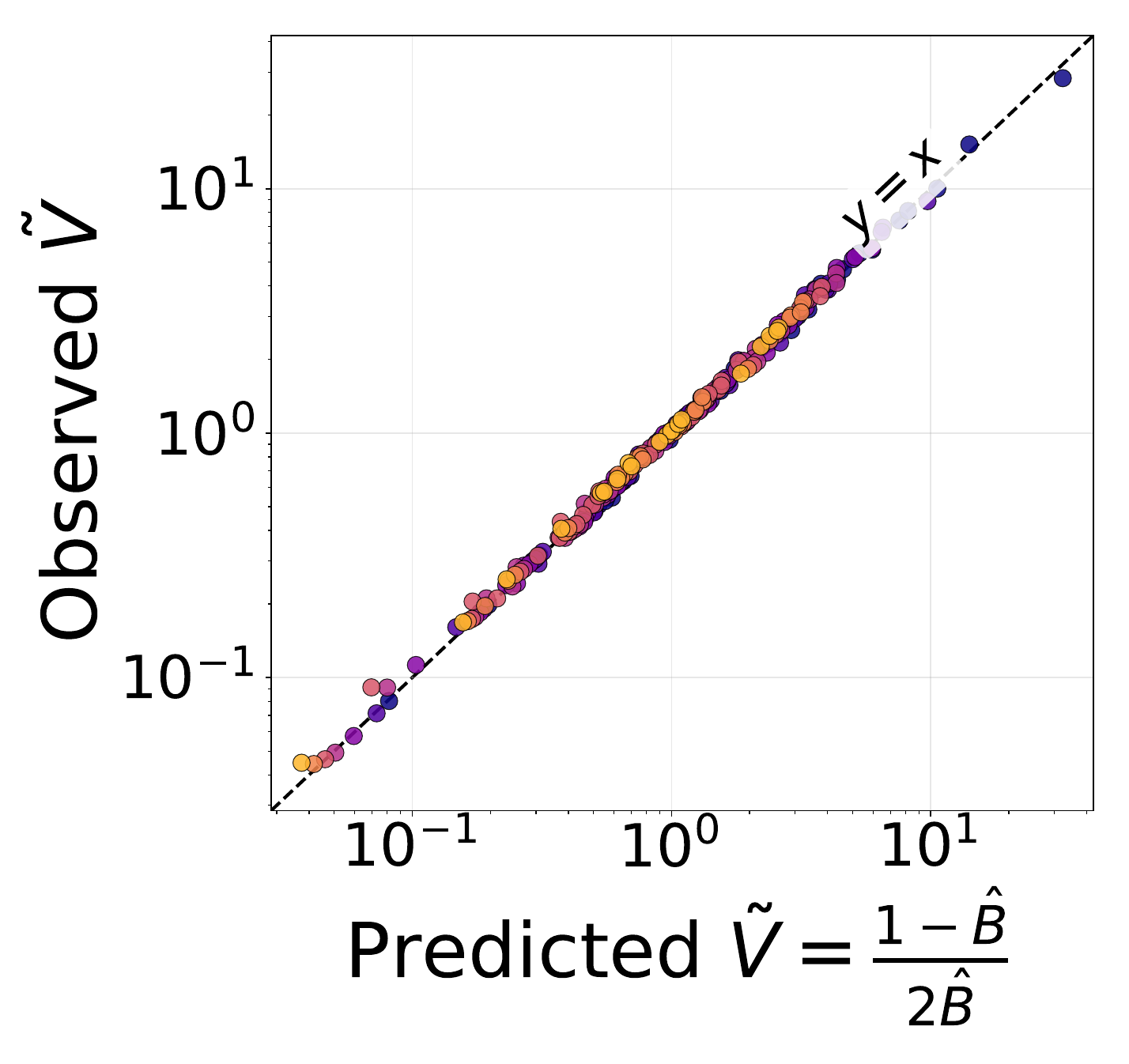} &
         \includegraphics[width=0.22\linewidth]{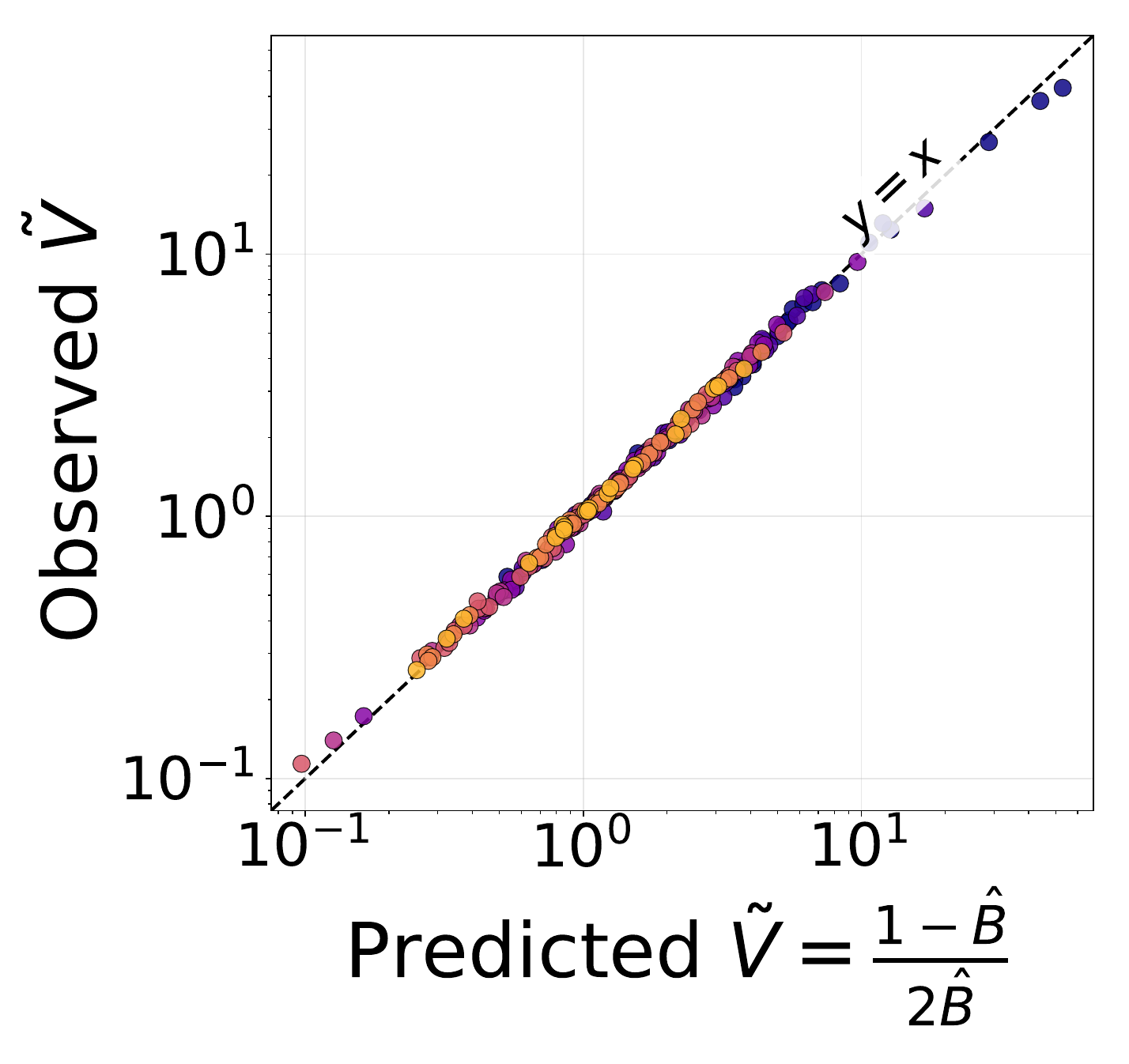} &
         \includegraphics[width=0.22\linewidth]{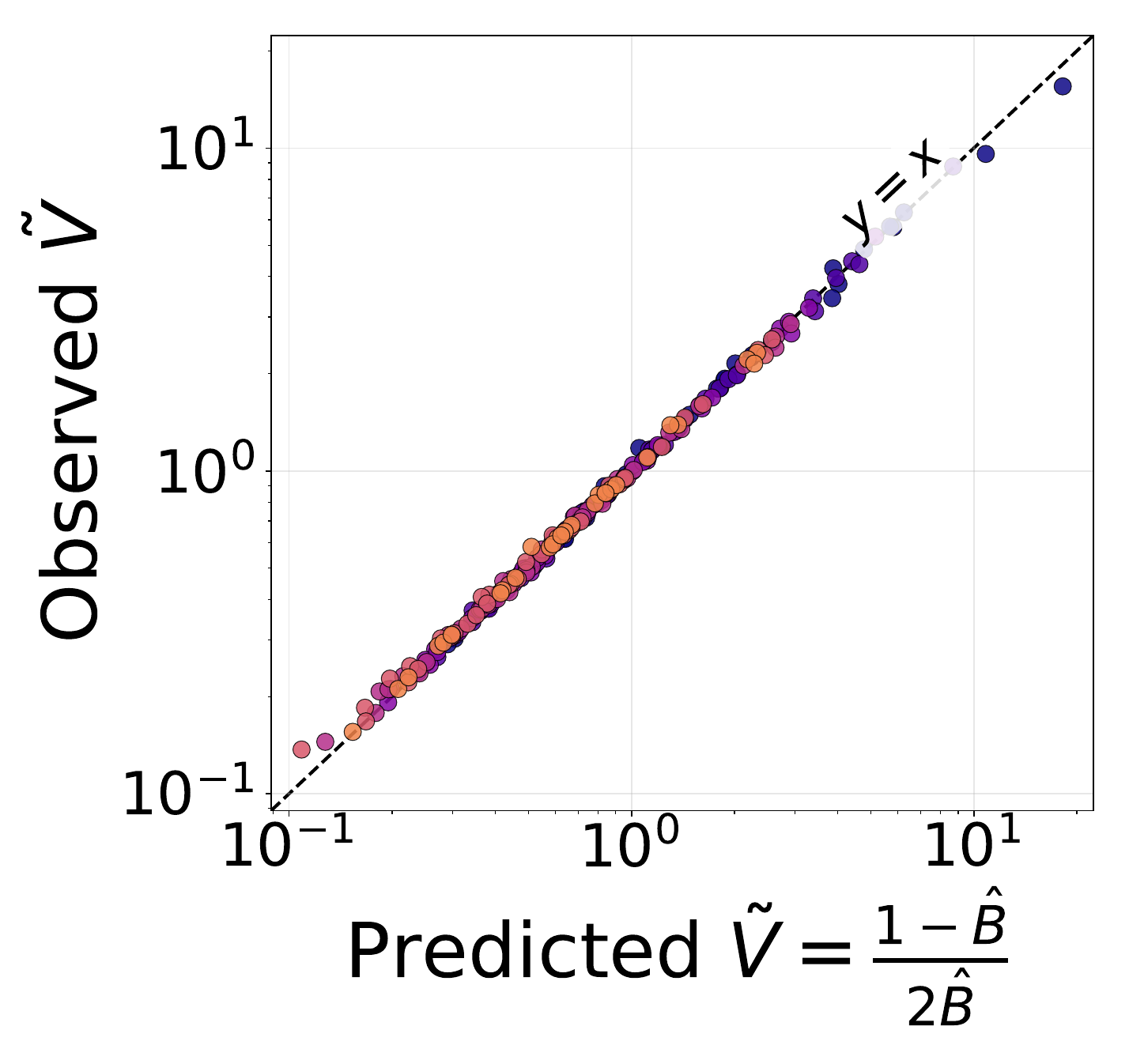} &
         \includegraphics[width=0.22\linewidth]{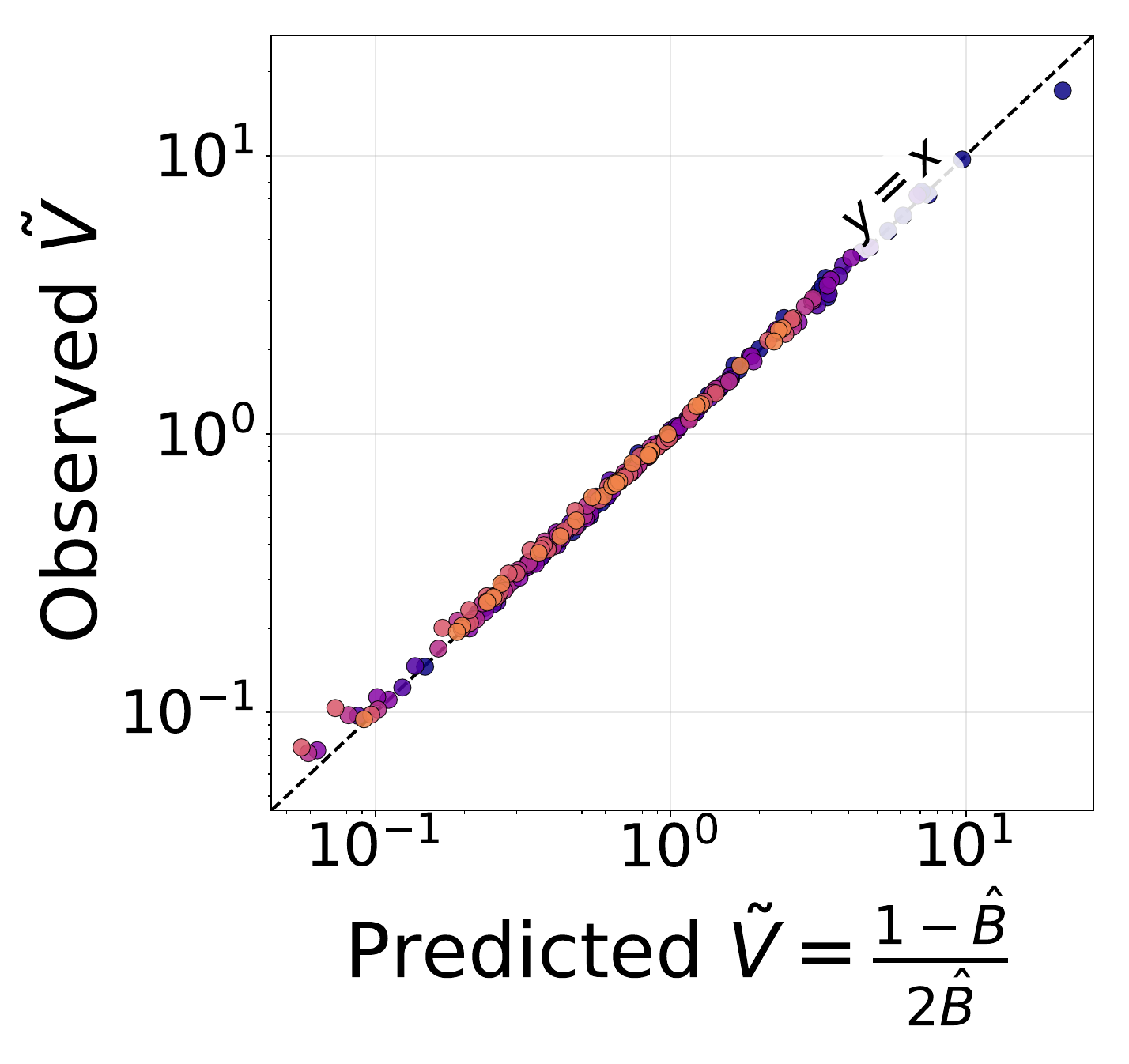} \\
         {\small {\bf (a) W-MSE-CelebA}} &
         {\small {\bf (b) VICReg-CelebA}} &
         {\small {\bf (c) Barlow-IM1K}} &
         {\small {\bf (d) I-JEPA-IM1K}}
    \end{tabular}
    \caption{{\bf Predicted versus observed directional CDNV on CelebA. \enspace}
    Each point is one CelebA attribute at a given rank \(r\). The prediction
    \(\widehat{\tilde V}_{\mathrm{pred}}=(1-\widehat B_A^{(t)})/(2\widehat B_A^{(t)})\)
    is computed on split A. Class centroids, the semantic direction, and observed directional CDNV are estimated independently on held-out split B. The dashed line is \(y=x\), and each panel contains all 40 attributes.}
    \label{fig:Br_predicts_geom}
\end{figure}

\begin{figure}[t]
    \centering
    \begin{tabular}{cc}
    \multicolumn{2}{c}{
            \includegraphics[width=0.70\linewidth]{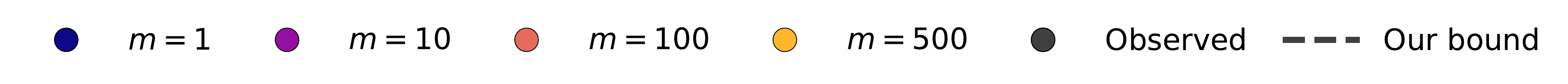}
        } \\[-1mm]
         \includegraphics[width=0.46\linewidth]{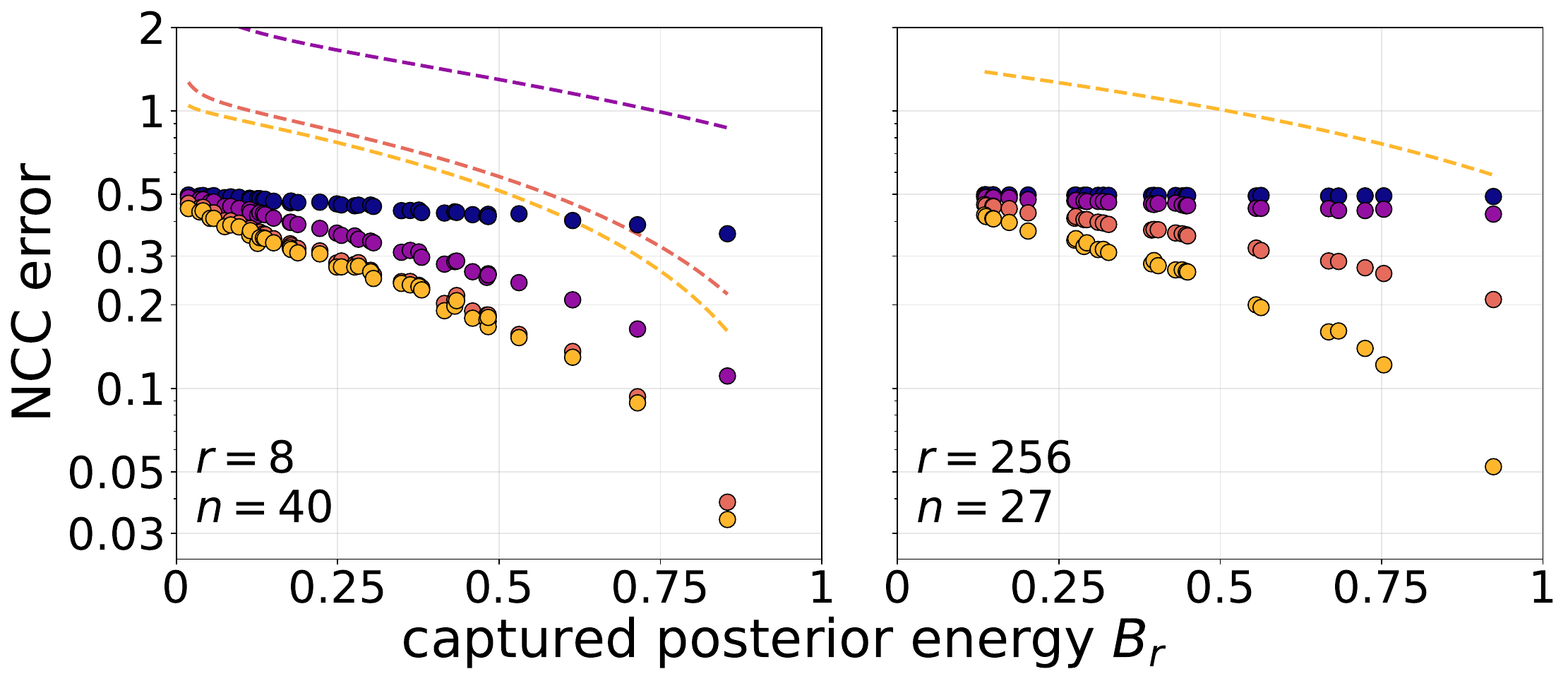} &
         \includegraphics[width=0.46\linewidth]{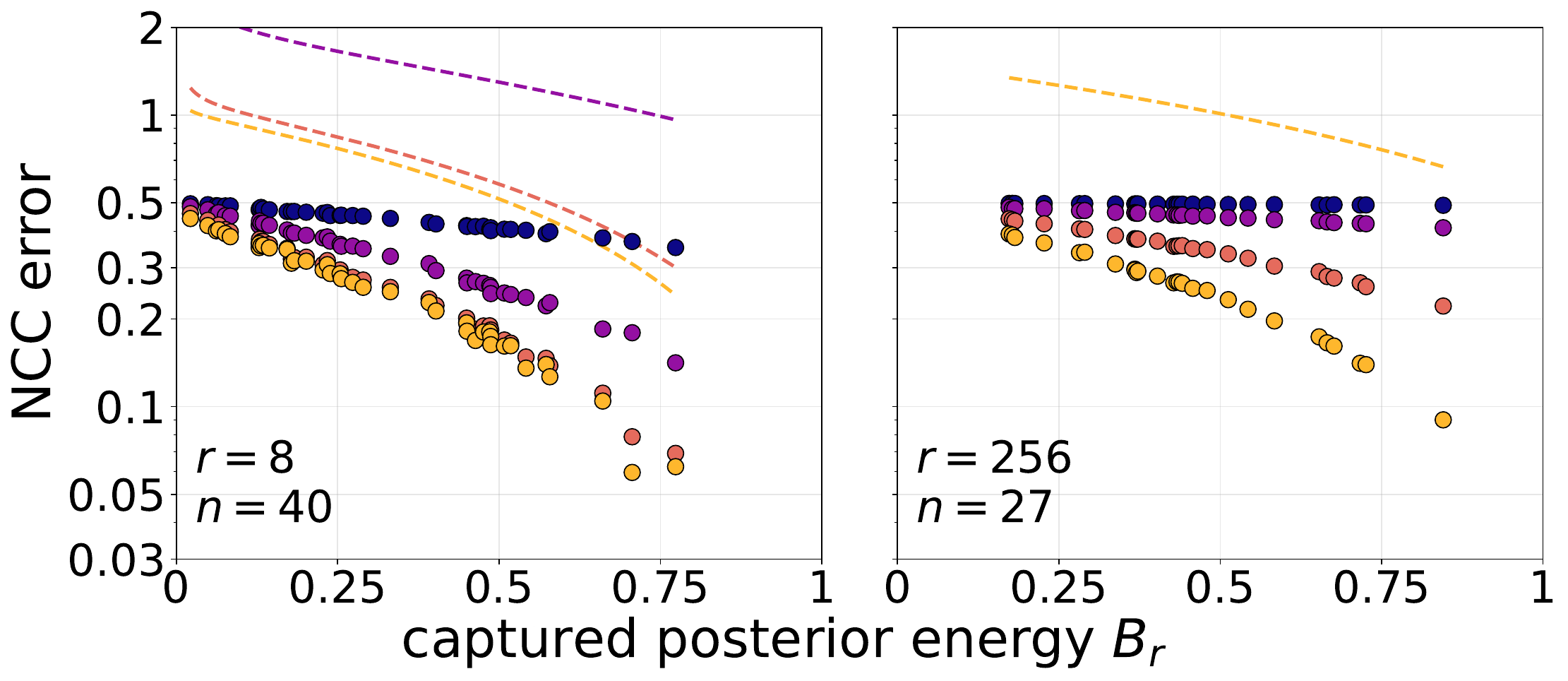} \\
         {\small {\bf (a) W-MSE-CelebA}} &
         {\small {\bf (b) I-JEPA-IM1K}}
    \end{tabular}
    \caption{{\bf Captured posterior energy controls few-shot NCC error. \enspace}
    Empirical \(m\)-shot NCC error versus task-specific \(B_r^{(t)}\) across CelebA attributes, for several ranks \(r\) (panels) and support sizes \(m\) (colors). Each point corresponds to a downstream attribute, and dashed curves are the corresponding upper bounds from Thm.~\ref{thm:ncc_bound_direct_via_Br}.}
    \label{fig:Br_predicts_ncc}
\end{figure}

\begin{figure}[t]
    \centering
    \setlength{\tabcolsep}{2pt}
    \begin{tabular}{@{}cccc@{}}
        \includegraphics[width=0.16\linewidth]{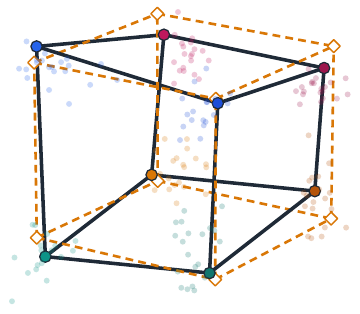} &
        \includegraphics[width=0.16\linewidth]{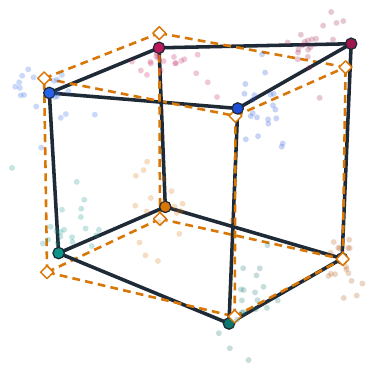} &
        \includegraphics[width=0.30\linewidth]{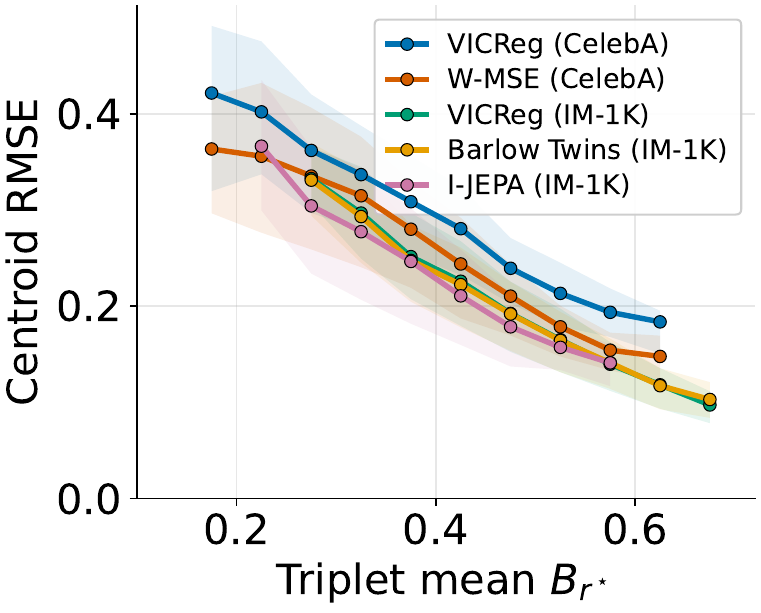} &
        \includegraphics[width=0.30\linewidth]{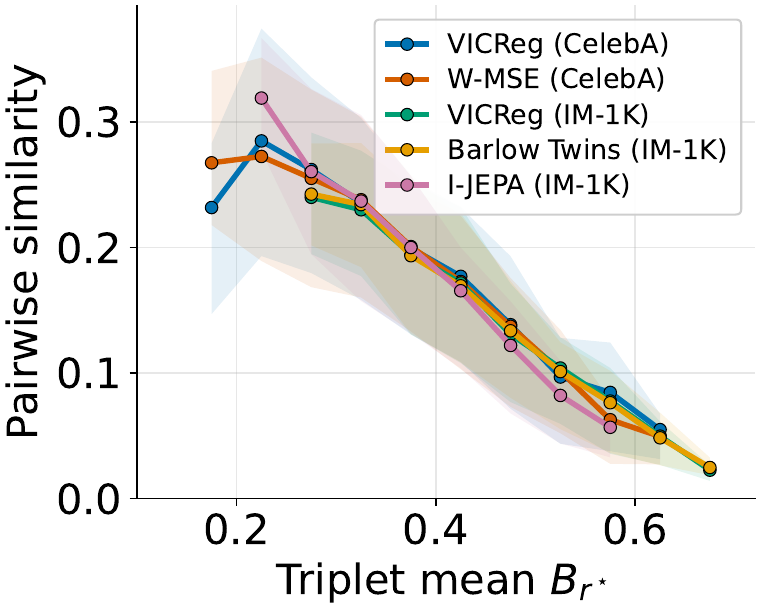} \\
        {\small {\bf (a) W-MSE-CelebA}} &
        {\small {\bf (b) I-JEPA-IM1K}} &
        {\small {\bf (c) Centroid RMSE}} &
        {\small {\bf (d) Task-axis overlap}}
    \end{tabular}

    \caption{{\bf Multitask centroid geometry on CelebA. \enspace}
(a--b) Illustrative held-out joint centroids for three attributes, compared with the hyperrectangle predicted by Thm.~\ref{thm:axis_and_centroids}; triplets are selected on the training split.
(c--d) Results over all 1646 eligible triplets per encoder, without recoverability or orthogonality screening. Larger mean task recoverability is associated with lower normalized centroid RMSE and smaller maximum pairwise task-axis overlap. Curves show bin means and shaded regions show interquartile ranges; selection and normalization details are in App.~\ref{app:celeba_extra}.}
    \label{fig:hyperrec_multitask}
\end{figure}

\begin{figure}[t]
    \centering
    \begin{tabular}{cccc}
        \includegraphics[width=0.22\linewidth]{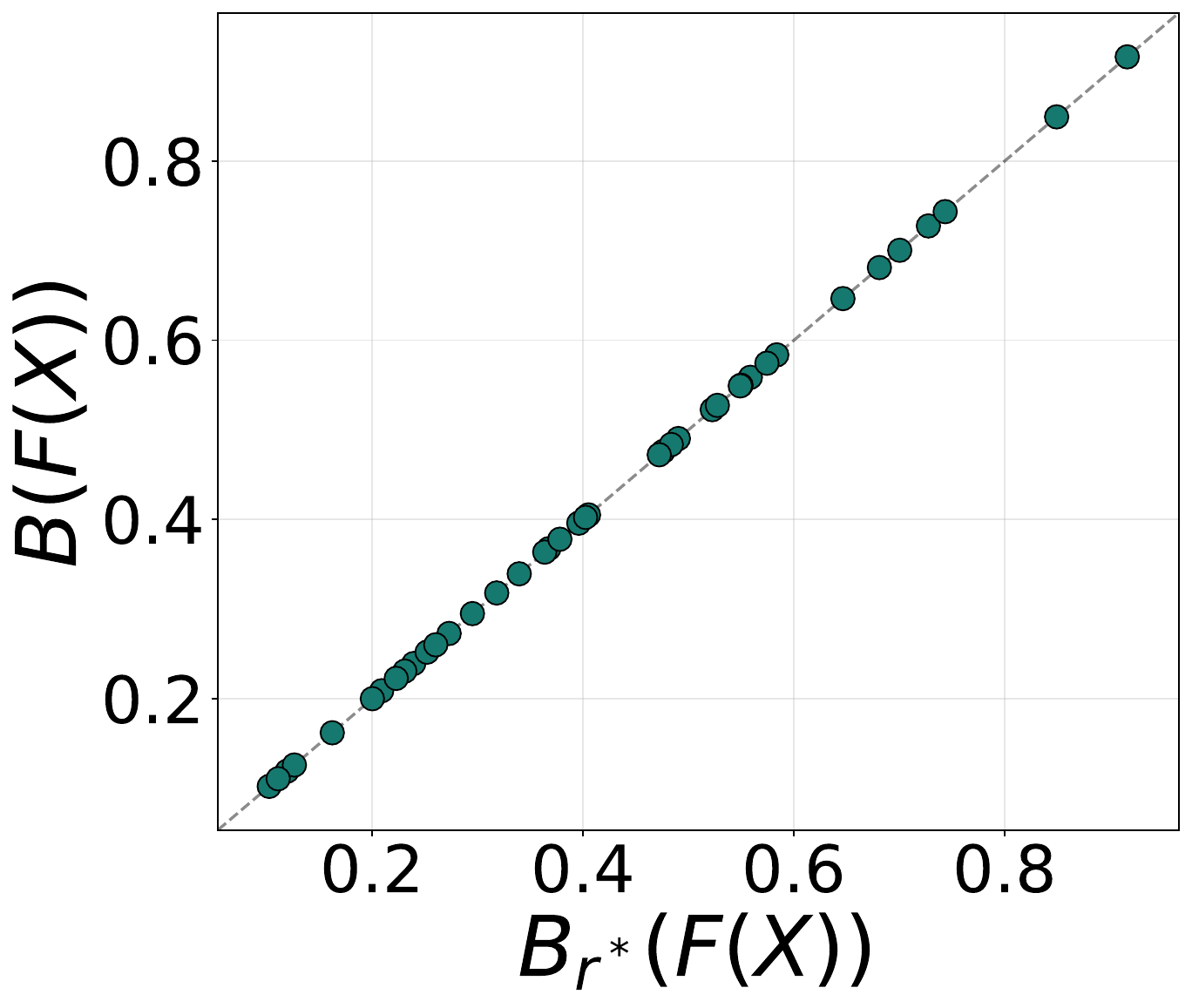} &
        \includegraphics[width=0.22\linewidth]{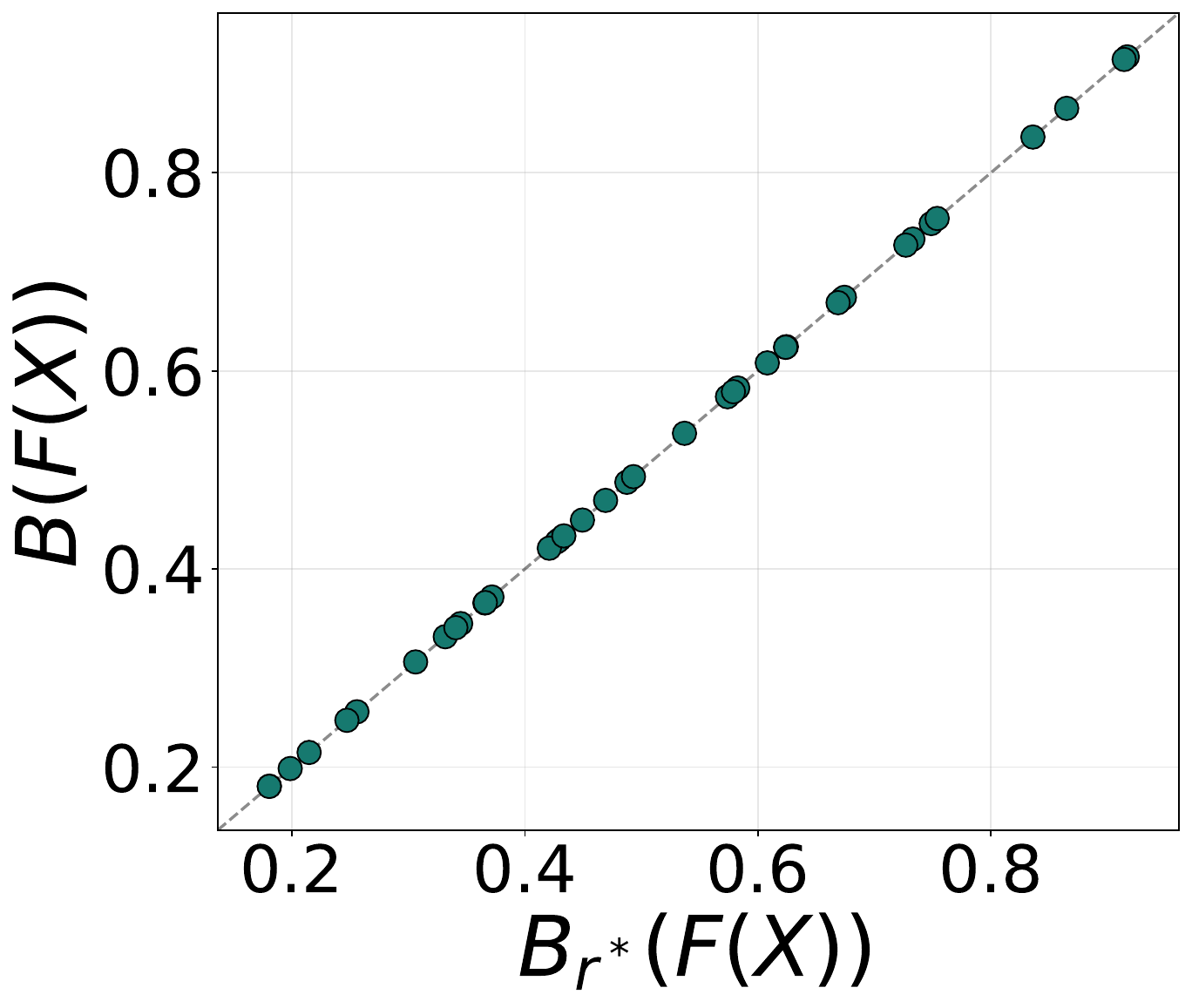} &
        \includegraphics[width=0.22\linewidth]{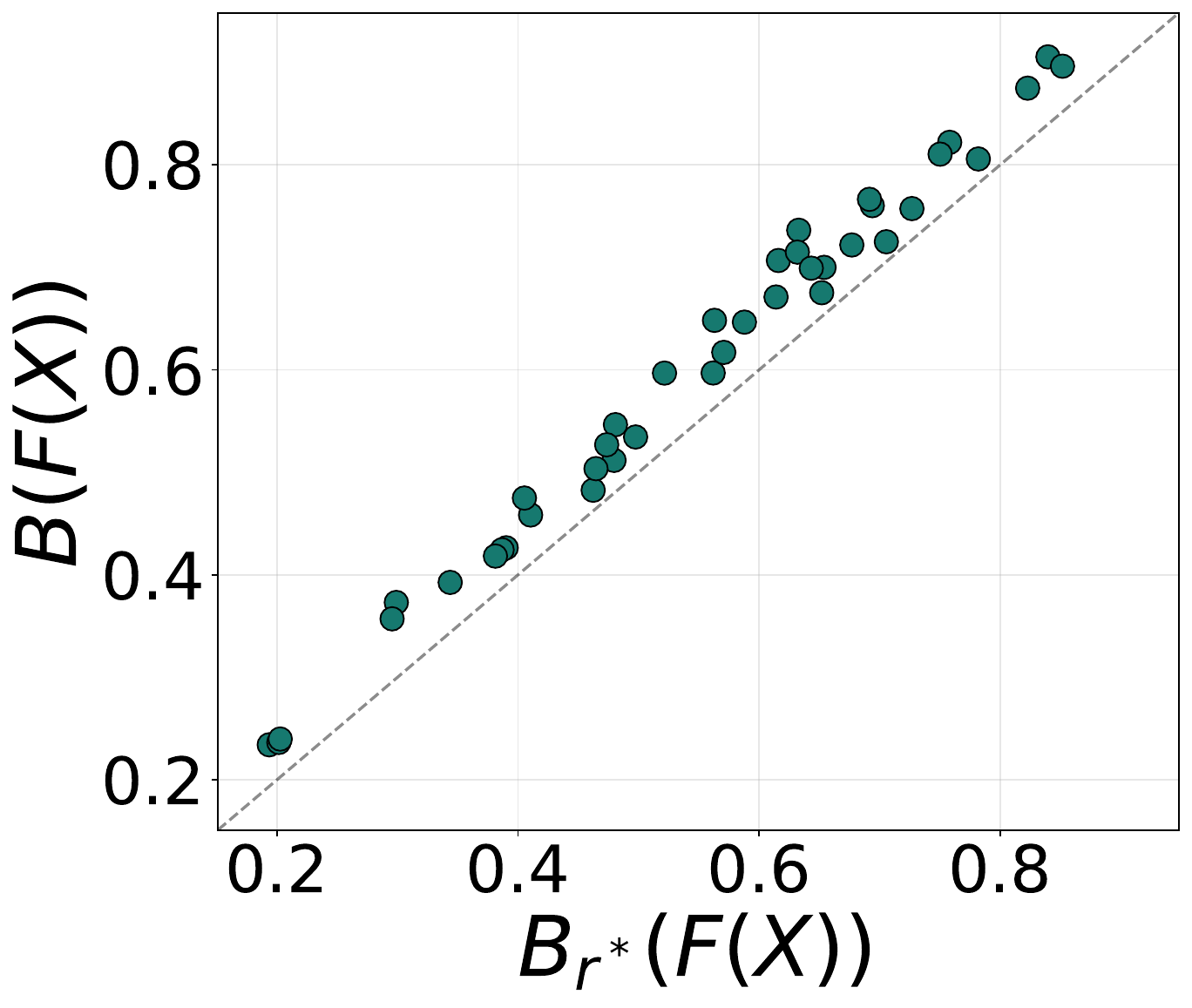} &
        \includegraphics[width=0.22\linewidth]{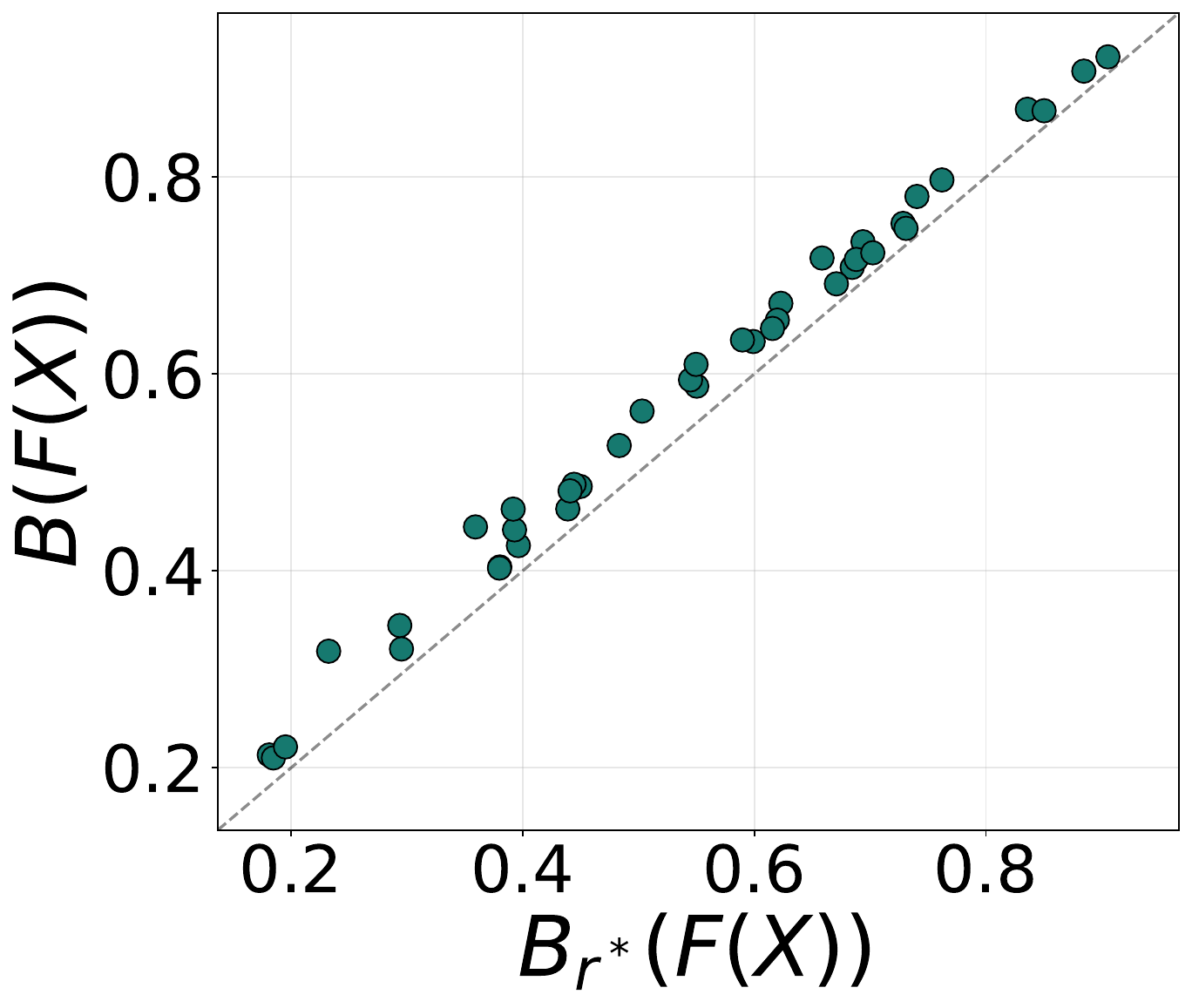} \\
        {\small {\bf (a) W-MSE-CelebA}} &
         {\small {\bf (b) VICReg-CelebA}} &
         {\small {\bf (c) Barlow-IM1K}} &
         {\small {\bf (d) I-JEPA-IM1K}}
    \end{tabular}
    \caption{{\bf Recoverability from the paired-view spectral subspace. \enspace}
    Directly measured task recoverability \(\widehat B^{(t)}(F)\) versus its reconstruction \(\widehat B_{r^\star}^{(t)}\) from the label-free paired-view spectral modes. Each point is one CelebA attribute; the dashed line is perfect agreement. For ImageNet-pretrained encoders, the retained subspace is reduced-rank and the median absolute discrepancy is below \(0.05\) across models; additional encoders are shown in App.~\ref{app:celeba_extra}.}
    \label{fig:bf_equals_br}
\end{figure}

\subsection{Settings}
\label{subsec:exp_setup}

We test the main empirical consequences of the theory on natural and synthetic data. Full dataset, preprocessing, and training details are given in App.~\ref{app:exp_details}.

{\bf Datasets.\enspace} The main experiments use CelebA~\citep{liu2015deeplearningfaceattributes}, whose 40 binary attributes provide a collection of downstream semantic tasks. For each attribute, we balance the two classes by subsampling so that the empirical problem matches the binary setting of Sec.~\ref{sec:setup}. Controlled experiments on dSprites~\citep{dsprites17} and 3DShapes~\citep{3dshapes18} are reported in App.~\ref{app:synthetic_results}.

{\bf SSL methods and rank-$r$ representations.\enspace}
We train W-MSE~\citep{pmlr-v139-ermolov21a} and VICReg~\citep{bardes2022vicreg} from scratch on CelebA with ResNet-50~\citep{He_2016_CVPR} backbones. We also evaluate publicly available ImageNet-1K~\citep{5206848}-pretrained VICReg and Barlow Twins~\citep{zbontar2021barlow} ResNet-50 encoders and an I-JEPA~\citep{DBLP:conf/cvpr/AssranDMBVRLB23} ViT-H/14 encoder. To vary representation rank, we estimate a paired-view spectral basis from unlabeled data and retain its leading $r$ directions. The retained features are then centered and whitened before the label-dependent measurements of \(B_r\) below.

{\bf Estimating semantic recoverability.\enspace}
For a balanced downstream task, let \(\hat\beta=\widehat{\E}[YF(X)] \) and \( \widehat G=\widehat{\E}[F(X)F(X)^\top]\).
We estimate captured posterior energy by $\widehat B(F)=\hat\beta^\top \widehat G^{-1}\hat\beta$. This form accounts for finite-sample deviations from exact whitening. When \(\widehat G=I_r\), it reduces to \(\widehat B(F)=\|\widehat{\E}[YF(X)]\|_2^2\). Unless noted otherwise, all reported recoverability values use this estimator.

\subsection{Results}
\label{subsec:exp_results}

{\bf Recoverability is task-dependent.\enspace}
Because the rank-$r$ subspaces are nested, \(B_r^{(t)}\) is nondecreasing in \(r\). Fig.~\ref{fig:captured_Br} nevertheless shows substantial variation across CelebA attributes: the most recoverable tasks approach \(B_r^{(t)}\approx1\) at moderate ranks, while the average over all 40 attributes remains lower. SSL representations also exceed random-initialization baselines across a broad range of ranks. Tasks differ markedly in both saturation rate and attainable recoverability. The same pattern appears for additional pretrained encoders in App.~\ref{app:celeba_extra} and on dSprites and 3DShapes in App.~\ref{app:synthetic_results}.

{\bf Recoverability predicts directional geometry.\enspace}
Prop.~\ref{prop:Br_probe_dcdnv} gives the exact relation \(\tilde V_F=(1-B(F))/(2B(F))\) for centered, whitened representations. We test it with a split-sample protocol: recoverability is estimated on split A and directional CDNV independently on split B. Fig.~\ref{fig:Br_predicts_geom} shows close agreement over more than two orders of magnitude, across ranks and SSL representations. Additional encoder results are given in App.~\ref{app:celeba_extra}.

{\bf Recoverability controls few-shot transfer.\enspace}
Thm.~\ref{thm:ncc_bound_direct_via_Br} predicts that, at fixed rank and shot count, larger \(B_r^{(t)}\) yields lower NCC error, while increasing the support size \(m\) reduces centroid-estimation error. Fig.~\ref{fig:Br_predicts_ncc} shows both trends across CelebA attributes: empirical error decreases with recoverability, and the bound tightens as \(m\) grows. The bound remains above the observed error throughout the evaluated configurations and becomes more informative for well-recovered tasks. App.~\ref{app:celeba_extra} reports the remaining encoders and Tab.~\ref{tab:bound_comparison_male} shows comparisons with earlier neural-collapse-based transfer bounds~\citep{galanti2023comparative, luthra2025selfsupervisedcontrastivelearningapproximately, luthra2026directionalneuralcollapseexplains}.

{\bf Multiple semantic factors.\enspace}
Thm.~\ref{thm:axis_and_centroids} predicts two related effects: well-recovered tasks should have nearly orthogonal semantic axes when cross-task correlations and residual uncertainty are small, and their joint centroids should approach the corresponding factorial hyperrectangle. Fig.~\ref{fig:hyperrec_multitask}(a--b) shows two illustrative triplets selected on the training split using fixed recoverability and near-orthogonality thresholds (see App.~\ref{app:celeba_extra}); the displayed centroids are then measured on held-out test images. Panels (c--d) use every eligible triplet, with no additional recoverability or orthogonality screening. Across encoders, higher mean task recoverability is associated with both lower centroid RMSE and smaller task-axis overlap, consistent with the geometry predicted by the theorem.

{\bf The spectral subspace recovers semantic energy.\enspace}
Cor.~\ref{cor:posterior_decomp_ssl_basis} identifies recoverability at the population optimum with posterior energy in the leading two-view spectral subspace. We test this empirically by comparing \(\widehat B^{(t)}(F)\), measured directly from the learned representation, with \(\widehat B_{r^\star}^{(t)}\), reconstructed from the estimated paired-view spectral modes at the label-free rank \(r^\star\) defined in App.~\ref{app:celeba_extra}. Fig.~\ref{fig:bf_equals_br} shows close agreement. For CelebA-trained W-MSE and VICReg, the retained basis spans nearly the full projector space, so near-exact agreement is expected. For ImageNet-pretrained encoders, the lower-dimensional spectral subspace still recovers most of the full-representation semantic energy, with median absolute discrepancies below \(0.05\) across models. The subspace is estimated without downstream labels. For I-JEPA, the paired-view operator is a surrogate analysis kernel rather than the masking distribution used in pretraining; more details can be found in App.~\ref{app:exp_details}.

\section{Discussion and Limitations}
\looseness=-1 Our results identify {\em semantic recoverability} as the key quantity linking same-instance SSL to downstream geometry and transfer. The task-specific captured posterior energy \(B_r^{(t)}\) provides a unified account of directional neural collapse, probe--centroid alignment, multitask centroid structure, and few-shot performance. These results are derived primarily for population-optimal, centered and whitened representations. We also analyze a covariance-regularized surrogate, but extending the theory to finite-sample training and imperfect optimization remains an important direction for future work. Our main analysis also focuses on balanced binary tasks and conditionally i.i.d.\ two-view generation. Extending the framework to multiclass tasks and more general view dependencies is another natural direction.

\newpage

\section*{AI use statement}

In this work, we used generative AI tools (ChatGPT and Claude Code) to assist with developing and refining theoretical frameworks, formulating mathematical claims, discussing proof strategies, and providing feedback on experimental methodology. These tools also assisted with experimental design and implementation, including code generation and refinement, as well as research brainstorming, literature exploration, manuscript organization, drafting and editing, and LaTeX presentation.

All AI-assisted theoretical arguments, mathematical derivations, experimental designs, and generated or modified code were reviewed by the authors for correctness and consistency. Relevant literature was independently examined, and all experimental results and their interpretation were evaluated by the authors. The authors take full responsibility for the final content of the paper, including its text, claims, code, results, and other artifacts produced with the assistance of generative AI.

\section*{Ethics statement}

This work studies theoretical and empirical properties of representations learned through self-supervised learning. It does not involve new human-subject studies or the collection of new personal data. The experiments use existing datasets and pretrained models for research purposes. The work does not develop or evaluate a high-risk application.

\section*{Reproducibility statement}

We state the assumptions and formal results in the main text and provide complete proofs in the appendix. The experimental sections and supplementary material document the datasets, pretrained representations, evaluation protocols, preprocessing, training procedures, and implementation details needed to reproduce our empirical results. We also specify the procedures used to estimate semantic recoverability, directional CDNV, spectral structure, and few-shot transfer performance.

\bibliography{refs}
\bibliographystyle{iclr2027_conference}

\newpage
\appendix
\section{Implementation Details}
\label{app:exp_details}

{\bf Datasets.\enspace}
CelebA~\citep{liu2015deeplearningfaceattributes} contains 202,599 RGB face images at resolution \(178\times218\), covering 10,177 identities and 40 binary attributes. We use the standard identity-disjoint split of 162,770 training, 19,867 validation, and 19,962 test images.

The dSprites dataset contains 737,280 binary \(64\times64\) images generated from six latent factors. We retain two of the three shapes, the four most extreme scale levels, and the eight most extreme values of each position coordinate, leaving orientation unrestricted. This gives 20,480 images, split randomly into 18,432 training and 2,048 held-out test images. Binary downstream tasks are formed by thresholding selected generative factors.

3DShapes contains 480,000 RGB \(64\times64\) images generated from six latent factors. We keep six of the ten object-hue levels and the four most extreme scale levels, giving 144,000 images. We use 129,600 images for training and 14,400 for held-out testing. As with dSprites, binary downstream tasks are obtained by thresholding selected latent factors; the factors used in each analysis are indicated in the corresponding figure.

{\bf Evaluation protocol.\enspace}
The held-out test split is used for the saturation curves and the centroid-geometry evaluations. For the recoverability and directional-CDNV analyses, we construct disjoint subsets A and B from the training data. The subsets contain 81,385 images each for CelebA, 2,048 images each for dSprites, and 14,400 images each for 3DShapes. The paired-view spectral basis, whitening transform, and recoverability estimates are computed on A; directional CDNV is then measured independently on B.

{\bf Evaluation preprocessing.\enspace}
All label-dependent quantities (recoverability, directional CDNV, NCC error, and linear-probe quantities) are measured on clean, single-view images. CelebA images are center-cropped to \(160\times160\) and resized to \(128\times128\) for the models trained from scratch. dSprites and 3DShapes are kept at their native \(64\times64\) resolution and are bilinearly upsampled to the input resolution of externally pretrained backbones when needed. Inputs are normalized with ImageNet-1K channel statistics. On CelebA, augmented pairs are used only to estimate the paired-view spectral basis. On dSprites and 3DShapes, the same deterministic resize-and-normalize transform is used both for paired-view spectral estimation and for clean feature extraction.

{\bf VICReg and W-MSE augmentations.\enspace}
We use the same view distribution for the CelebA-trained VICReg and W-MSE models. Each view is generated by a random resized crop with scale \((0.2,1.0)\) and aspect ratio \((3/4,4/3)\), resized to \(128\times128\), followed by horizontal flipping with probability \(0.5\). We apply color jitter with brightness \(0.4\), contrast \(0.4\), saturation \(0.2\), and hue \(0.05\) with probability \(0.8\), grayscale conversion with probability \(0.1\), and Gaussian blur with probability \(0.5\) using a \(7\times7\) kernel and \(\sigma\sim\mathcal U[0.1,2.0]\). We use a minimum crop scale of \(0.2\), rather than the ImageNet value \(0.08\), to preserve more facial structure. Unlike the original VICReg pipeline~\citep{bardes2022vicreg}, we do not use solarization and use the same blur probability for both views.

{\bf Pretrained encoders.\enspace}
For ImageNet-1K-pretrained VICReg, Barlow Twins, and I-JEPA, CelebA images are resized to a shorter side of 256 pixels and center-cropped to \(224\times224\). To estimate the paired-view spectral basis, we use an ImageNet-scale version of the augmentation pipeline above: random resized crops with scale \((0.08,1.0)\), color-jitter hue \(0.1\), grayscale probability \(0.2\), and a \(23\times23\) Gaussian blur kernel. For I-JEPA, these augmented pairs define a surrogate two-view kernel for analysis; they are not the masking-based views used during I-JEPA pretraining. For dSprites and 3DShapes, externally pretrained encoders receive bilinearly upsampled \(224\times224\) inputs.

{\bf Training details.\enspace}
We train VICReg~\citep{bardes2022vicreg} and W-MSE~\citep{pmlr-v139-ermolov21a} from scratch on the CelebA training split for 1000 epochs. Both use AdamW~\citep{loshchilov2018decoupled} with \(\beta_1=0.9\), \(\beta_2=0.999\), \(\epsilon=10^{-8}\), and weight decay \(10^{-6}\). The learning rate is scaled linearly with global batch size, \(\eta=\eta_{\mathrm{base}}B/256\), with a 10-epoch linear warm-up~\citep{goyal2017accurate} followed by cosine decay~\citep{loshchilov2016sgdr}.

For VICReg, the projection head is
\texttt{Linear(2048 $\rightarrow$ 2048, bias=False) $\rightarrow$ BN $\rightarrow$ ReLU $\rightarrow$ Linear(2048 $\rightarrow$ 2048)}.
The invariance, variance, and covariance coefficients are \(\lambda=25\), \(\mu=25\), and \(\nu=1\). We use a global batch size of 1024 across two NVIDIA H100 GPUs and \(\eta_{\mathrm{base}}=1.5\times10^{-4}\).

For W-MSE, the projection head is
\texttt{Linear(2048 $\rightarrow$ 2048, bias=False) $\rightarrow$ BN $\rightarrow$ ReLU $\rightarrow$ Linear(2048 $\rightarrow$ 128)}.
Following~\citep{pmlr-v139-ermolov21a}, embeddings from each view are split into sub-batches of 256 and whitened independently using the inverse Cholesky factor of the empirical covariance, with no covariance regularization (\(\epsilon=0\)). The loss is normalized mean-squared error between corresponding whitened embeddings. The two views are processed in one concatenated forward pass. We use a global batch size of 512 on one NVIDIA H100 and \(\eta_{\mathrm{base}}=10^{-3}\).

{\bf Architectures.\enspace}
Both CelebA-trained models use a standard ResNet-50~\citep{He_2016_CVPR}. At \(128\times128\) resolution, the final residual stage produces a \(4\times4\) feature map, which is globally averaged to a 2048-dimensional representation. Unless stated otherwise, representation analyses use these pooled backbone features rather than projection-head outputs.

We also evaluate public ImageNet-1K-pretrained VICReg and Barlow Twins~\citep{zbontar2021barlow} ResNet-50 encoders and an I-JEPA~\citep{DBLP:conf/cvpr/AssranDMBVRLB23} ViT-H/14 encoder, without fine-tuning. The ResNet-50 encoders produce 2048-dimensional pooled features. I-JEPA has 32 transformer layers, 16 attention heads, hidden dimension 1280, and MLP dimension 5120. At \(224\times224\), it produces 256 patch tokens and no classification token; we average the final-layer-normalized patch tokens to obtain a 1280-dimensional representation.

\section{Additional Results}
\label{app:add_results}

\subsection{Additional CelebA Results}
\label{app:celeba_extra}

This section reports the remaining CelebA results using the same protocol as Sec.~\ref{sec:exp}.

{\bf Semantic recoverability.\enspace}
Fig.~\ref{fig:app_vicreg_recoverability} (a) shows the rank-dependent recoverability curves for ImageNet-1K-pretrained VICReg. It reproduces the main-text pattern: recoverability varies substantially across attributes, the strongest tasks saturate at lower rank, and the learned representations remain above their random-initialization baselines.


\begin{figure}[t]
    \centering
    \begin{tabular}{ccc}
        \includegraphics[width=0.31\linewidth]{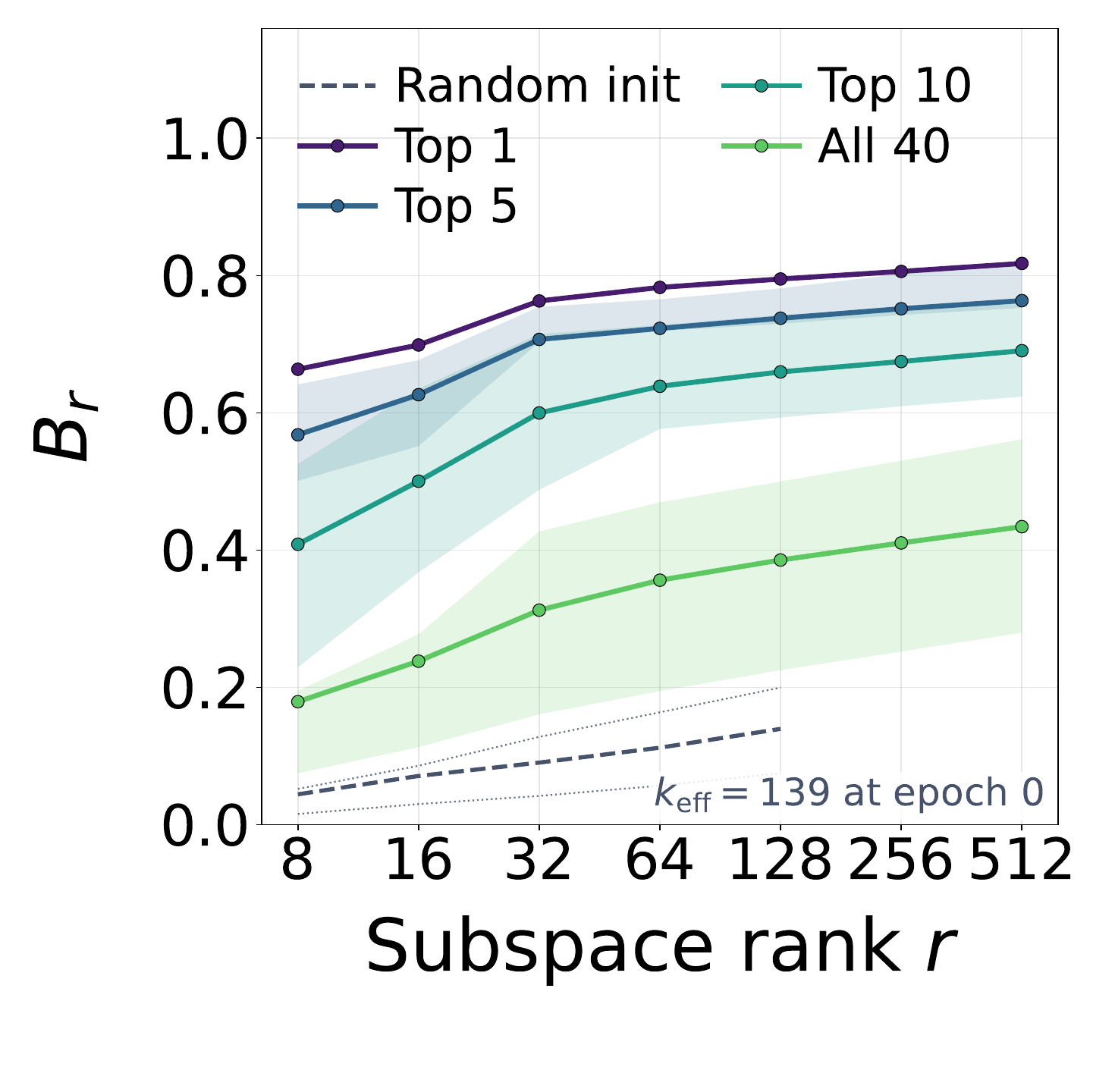} &
        \includegraphics[width=0.31\linewidth]{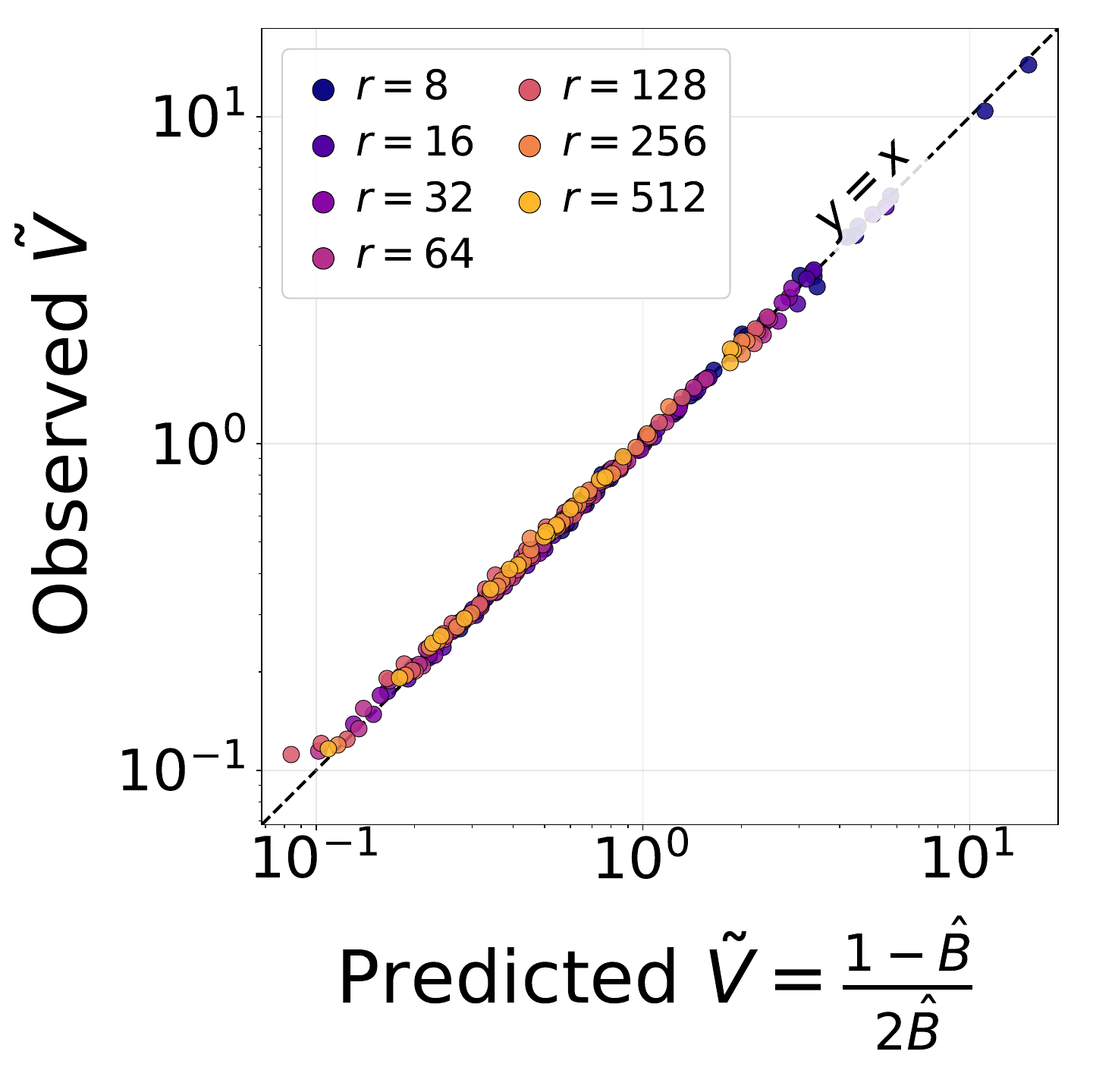} &
        \includegraphics[width=0.31\linewidth]{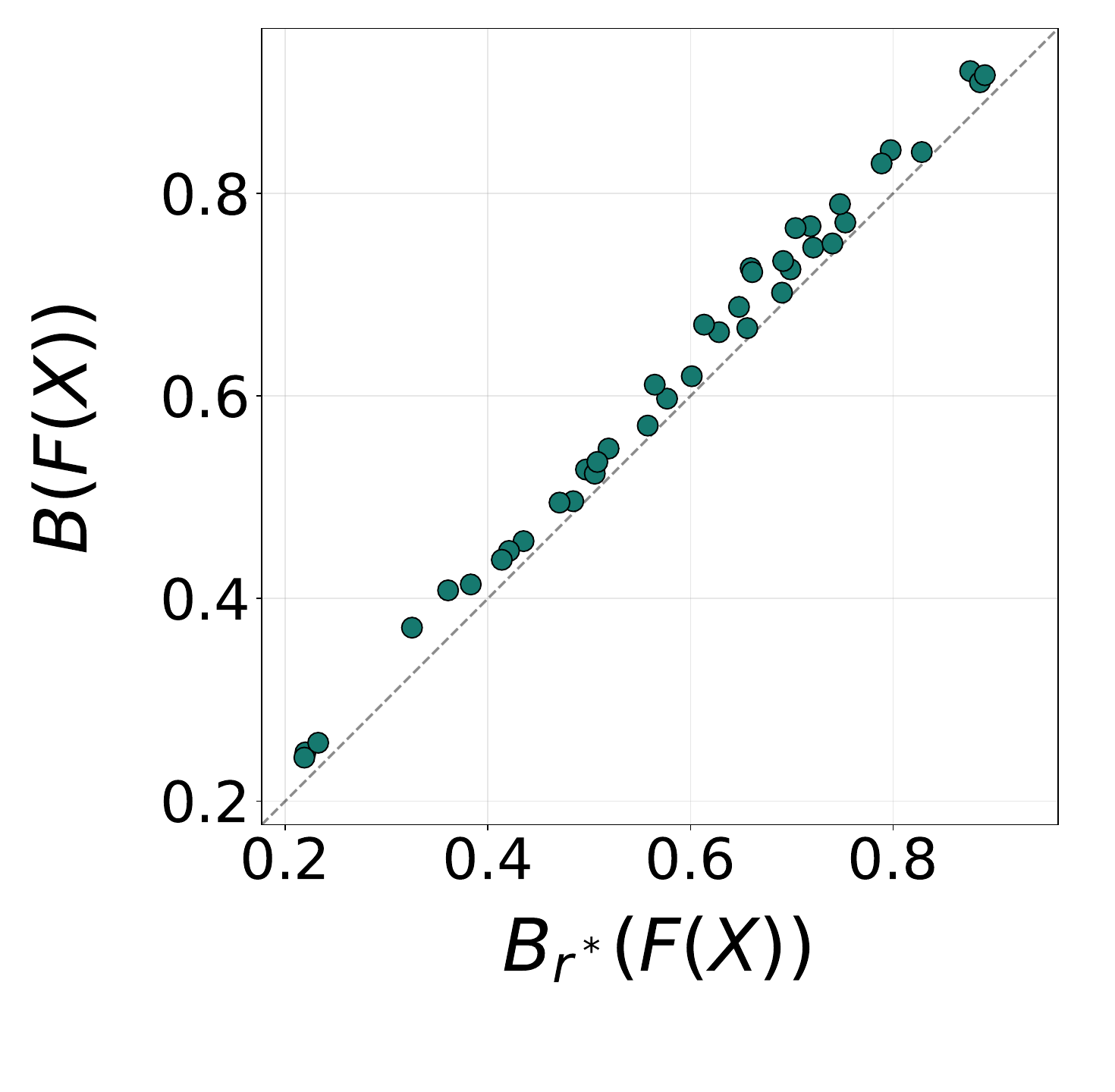} \\
        {\small {\bf (a) Captured posterior energy}} &
        {\small {\bf (b) Predicted vs.\ observed $\widetilde V$}} &
        {\small {\bf (c) Spectral identification}}
    \end{tabular}
    \caption{{\bf Additional recoverability results for VICReg-IM1K on CelebA. \enspace}
    (a) Captured posterior energy across representation ranks, using the same attribute groups and random-initialization reference as Fig.~\ref{fig:captured_Br}.
    (b) Directional CDNV predicted from recoverability estimated on split A versus directional CDNV measured independently on held-out split B; the dashed line is \(y=x\).
    (c) Recoverability measured directly from the learned representation versus its reconstruction from the paired-view spectral subspace; the dashed line denotes perfect agreement.}
    \label{fig:app_vicreg_recoverability}
\end{figure}

{\bf Directional geometry.\enspace}
Fig.~\ref{fig:app_vicreg_recoverability} (b) gives the held-out test of Prop.~\ref{prop:Br_probe_dcdnv}. Predictions are computed from split A and observed directional CDNV from split B. The agreement with \(y=x\) remains tight across attributes and ranks.


{\bf Few-shot transfer.\enspace}
Fig.~\ref{fig:app_Br_predicts_ncc} reports the remaining few-shot results. Across models, ranks, and shot counts, empirical NCC error decreases with recoverability and the bound tracks the same rank--shot tradeoff seen in the main text.

\begin{figure}[t]
    \centering
    \setlength{\tabcolsep}{2pt}
    \begin{tabular}{@{}ccc@{}}
         \includegraphics[width=0.31\linewidth]{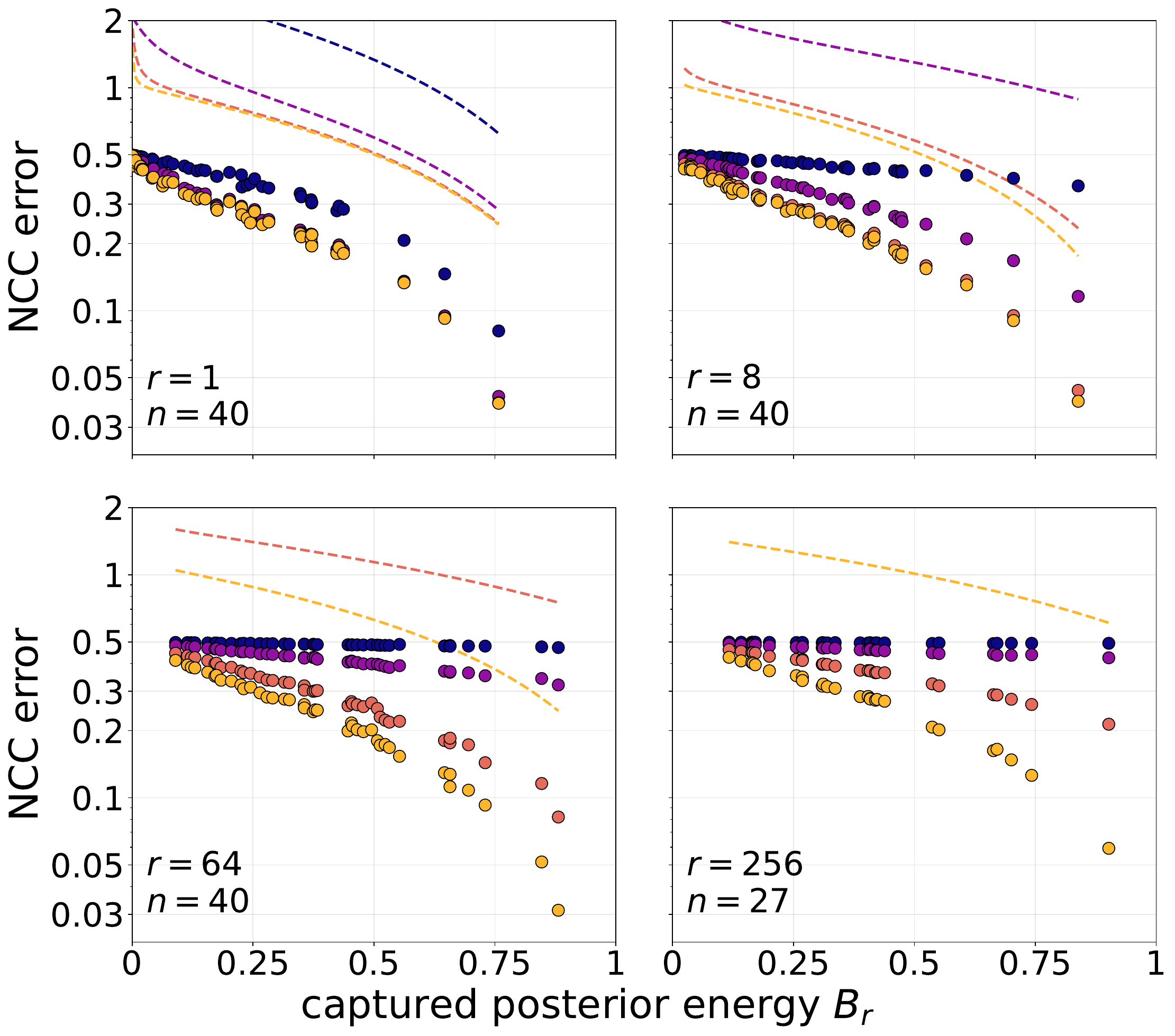} &
         \includegraphics[width=0.31\linewidth]{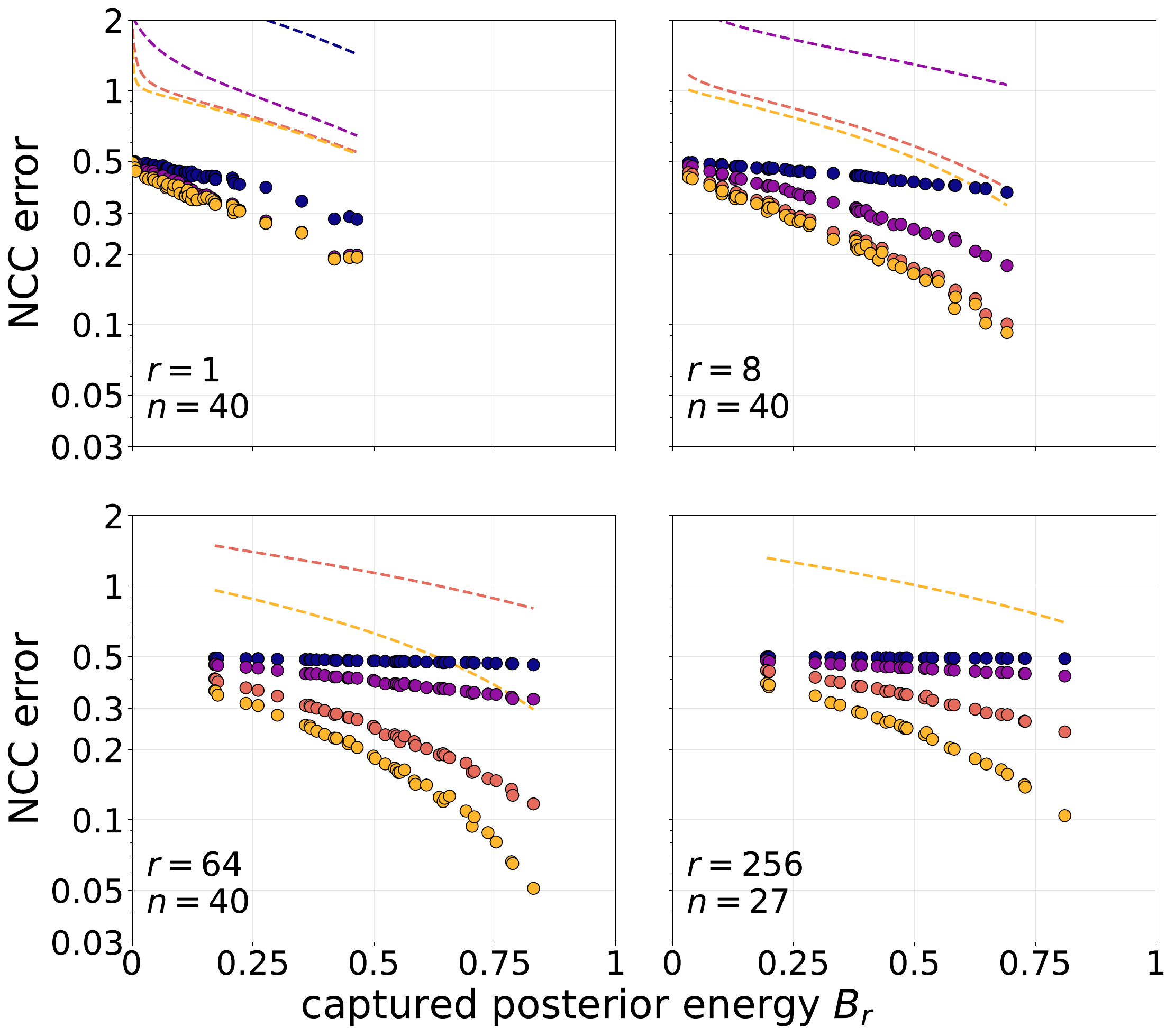} &
         \includegraphics[width=0.31\linewidth]{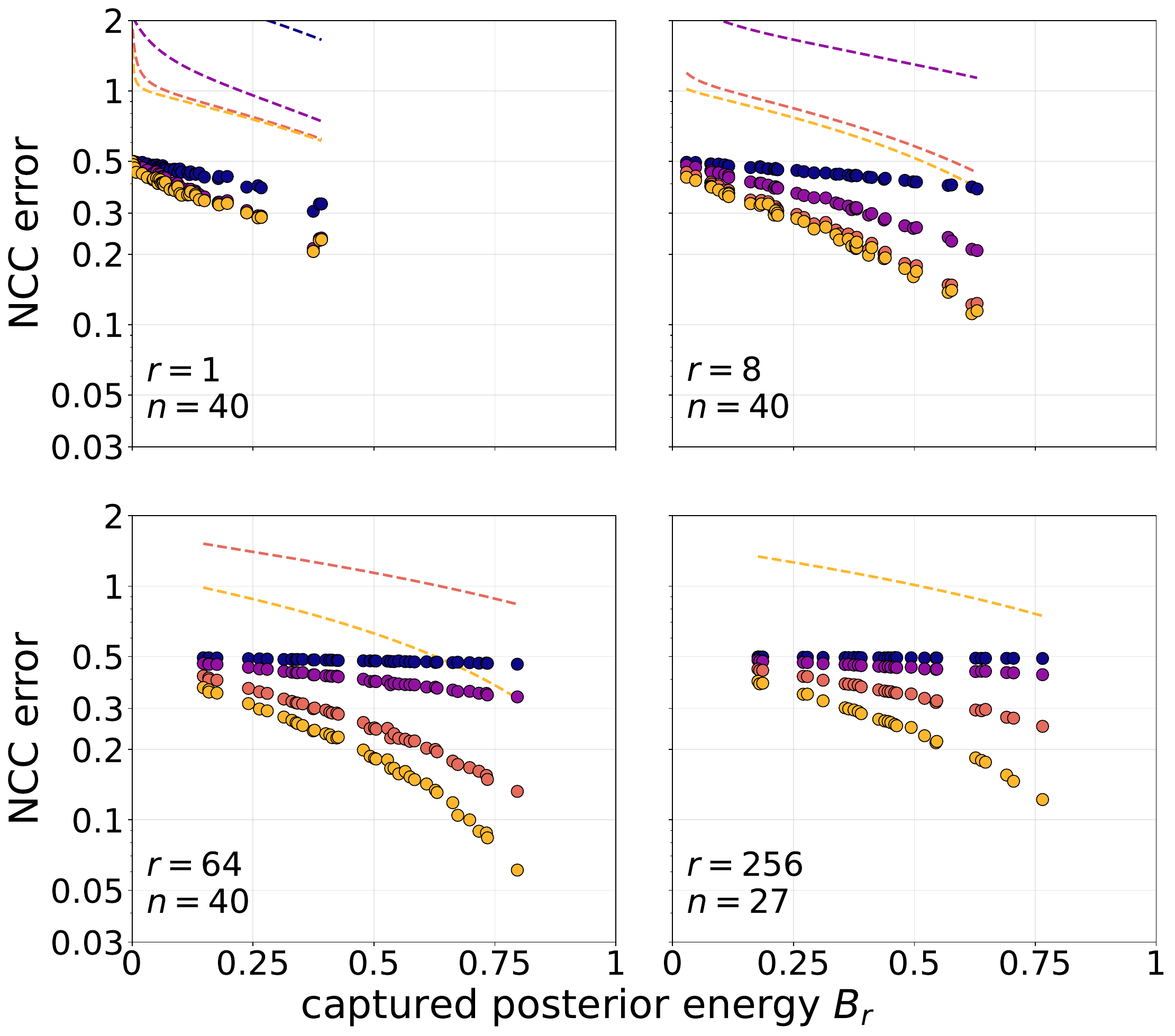} \\
         {\small {\bf (a) VICReg-CelebA}} &
         {\small {\bf (b) VICReg-IM1K}} &
         {\small {\bf (c) Barlow-IM1K}}
    \end{tabular}
    \caption{{\bf Captured posterior energy and few-shot NCC error. \enspace}
    Results for VICReg-CelebA, VICReg-IM1K, and Barlow Twins-IM1K across ranks \(r\) and support sizes \(m\). Dashed curves are the upper bound from Thm.~\ref{thm:ncc_bound_direct_via_Br}.}
    \label{fig:app_Br_predicts_ncc}
\end{figure}

{\bf Multitask geometry and spectral identification.\enspace}
Each encoder is evaluated at a label-free rank \(r^\star\), defined as the number of paired-view spectral directions whose eigenvalue is at least \(10^{-3}\) times the largest eigenvalue. This gives \(r^\star=946\) for W-MSE, 2046 for VICReg-CelebA, 763 for VICReg-IM1K, 806 for Barlow-IM1K, and 263 for I-JEPA-IM1K.

For each attribute triplet, the task-specific capture values \(B_{r^\star}^{(t)}\) and the task axes are estimated on the training split, weighting the eight joint-label cells equally; the centroids are then measured on the held-out test split. We write \(\overline B_{r^\star}:=\frac13\sum_{t=1}^3 B_{r^\star}^{(t)}\) for the triplet mean recoverability used in Fig.~\ref{fig:hyperrec_multitask}(c--d). Normalized centroid RMSE is the root-mean-square distance between observed and predicted centroids divided by the root-mean-square radius of the predicted box. Fig.~\ref{fig:hyperrec_multitask}(c--d) use all triplets drawn from the 26 attributes for which each binary class contains at least 10\% of the data and every joint-label cell contains at least 1000 training and 100 test examples. This yields 1646 triplets per encoder. No recoverability or orthogonality screening is applied to these quantitative panels. The curves are means in bins of width 0.05 containing at least 12 triplets; shaded regions show interquartile ranges. The illustrative cubes in Figs.~\ref{fig:hyperrec_multitask}(a--b) and \ref{fig:app_hyperrec_spectral_id}(a--b) are selected on the training split only from triplets satisfying \(\min_t B_{r^\star}^{(t)}\ge0.10\) and pairwise \(|\cos(u_i,u_j)|\le0.12\).

Fig.~\ref{fig:app_hyperrec_spectral_id} adds held-out centroid visualizations for VICReg-IM1K and Barlow Twins-IM1K, and Fig.~\ref{fig:app_vicreg_recoverability} (c) and Fig.~\ref{fig:bf_equals_br} (c) reports the corresponding spectral-identification plots for VICReg-IM1K and Barlow Twins-IM1K respectively. The latter compare recoverability measured in the full feature space with recoverability reconstructed from the label-free paired-view spectral subspace.

\begin{figure}[t]
    \centering
    \begin{tabular}{cc}
        \includegraphics[width=0.29\linewidth]{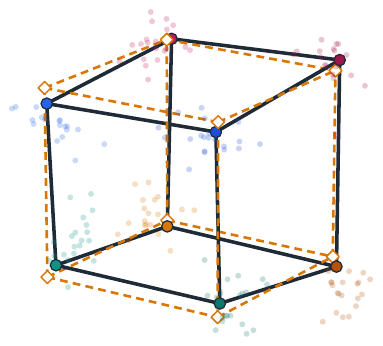} &
        \includegraphics[width=0.29\linewidth]{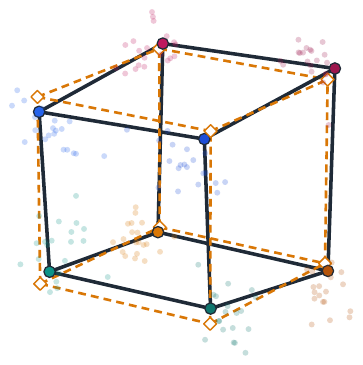} \\
        {\small {\bf (a) VICReg-IM1K}} &
        {\small {\bf (b) Barlow-IM1K}} \\[1mm]
    \end{tabular}
    \caption{{\bf Multitask centroid geometry on CelebA. \enspace}
    Held-out joint centroids and the hyperrectangles predicted by Thm.~\ref{thm:axis_and_centroids} for ImageNet-1K-pretrained VICReg and Barlow Twins. Experiment setup is same as Fig.~\ref{fig:hyperrec_multitask}}
    \label{fig:app_hyperrec_spectral_id}
\end{figure}

{\bf Comparison with prior transfer bounds.\enspace}
Table~\ref{tab:bound_comparison_male} compares Thm.~\ref{thm:ncc_bound_direct_via_Br} with earlier neural-collapse-based transfer bounds~\citep{galanti2023comparative,luthra2025selfsupervisedcontrastivelearningapproximately,luthra2026directionalneuralcollapseexplains}. On W-MSE trained on CelebA, the earlier bounds remain above \(0.5\) for all 40 attributes and all five shot counts shown. Our bound is non-vacuous in several rank--shot configurations and falls below \(0.5\) already at \(m=1\). The minimizing rank increases as more support examples become available, matching the rank--shot tradeoff in Thm.~\ref{thm:ncc_bound_direct_via_Br}.

\begin{table}[t]
    \centering
    \caption{\textbf{Comparison of few-shot transfer bounds on CelebA.}
    Bounds on NCC classification error for W-MSE trained on CelebA (``Male/Female'' attribute) . Our bound is minimized over the evaluated ranks separately for each shot count \(m\). Prior bounds use \(\ell_2\)-normalized full-dimensional features, whereas ours uses whitened top-\(r\) features. Values below \(0.5\) are non-vacuous relative to chance classification.}
    \label{tab:bound_comparison_male}
    \small
    \setlength{\tabcolsep}{5pt}
    \begin{tabular}{lccccc}
        \toprule
        \textbf{Bound} &
        $\boldsymbol{m=1}$ &
        $\boldsymbol{m=5}$ &
        $\boldsymbol{m=10}$ &
        $\boldsymbol{m=100}$ &
        $\boldsymbol{m=500}$ \\
        \midrule
        \citep{galanti2023comparative}
            & 2923.1 & 709.9 & 433.2 & 184.3 & 162.1 \\
        \citep{luthra2025selfsupervisedcontrastivelearningapproximately}
            & 138.1 & 32.1 & 20.6 & 6.70 & 3.53 \\
        \citep{luthra2026directionalneuralcollapseexplains}
            & 1544.4 & 282.5 & 126.7 & 10.27 & 2.05 \\
        \midrule
        \textbf{Our bound}
            & \textbf{0.490} & \textbf{0.253} & \textbf{0.223}
            & \textbf{0.190} & \textbf{0.161} \\
        \bottomrule
    \end{tabular}
\end{table}

\subsection{Results on Synthetic Datasets}
\label{app:synthetic_results}

We use dSprites and 3DShapes to check that the same recoverability--geometry relationships are visible in controlled factorized data. Figs.~\ref{fig:app_synth_captured_Br}--\ref{fig:app_synth_hyperrec} report rank-dependent recoverability, the directional-CDNV identity, and held-out multitask centroid geometry. The displayed encoders are ImageNet-1K-pretrained VICReg, I-JEPA, and Barlow Twins. Across both datasets, the qualitative behavior matches the CelebA results.

\begin{figure}[t]
    \centering
    \setlength{\tabcolsep}{2pt}
    \begin{tabular}{@{}ccc@{}}
    \multicolumn{3}{c}{
            \includegraphics[width=0.70\linewidth]{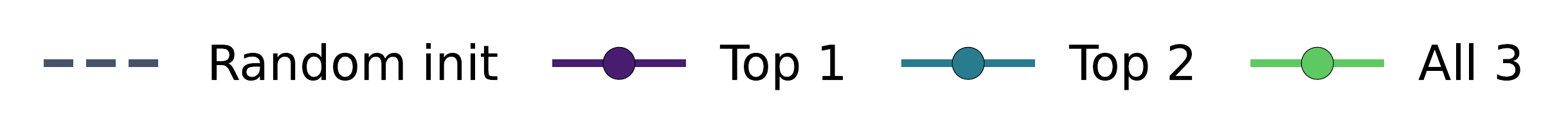}
        } \\[-1mm]
        \includegraphics[width=0.31\linewidth]{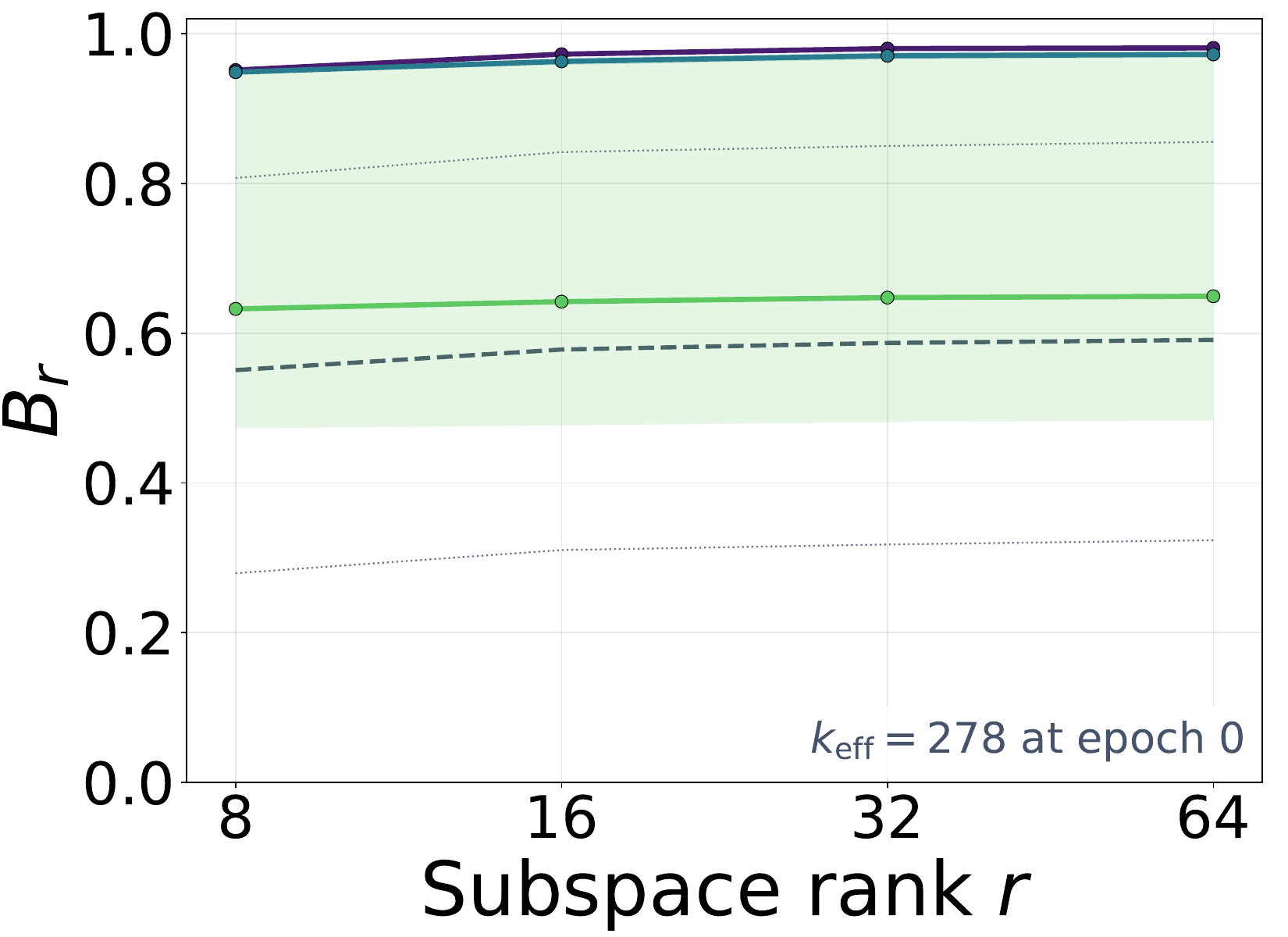} &
        \includegraphics[width=0.31\linewidth]{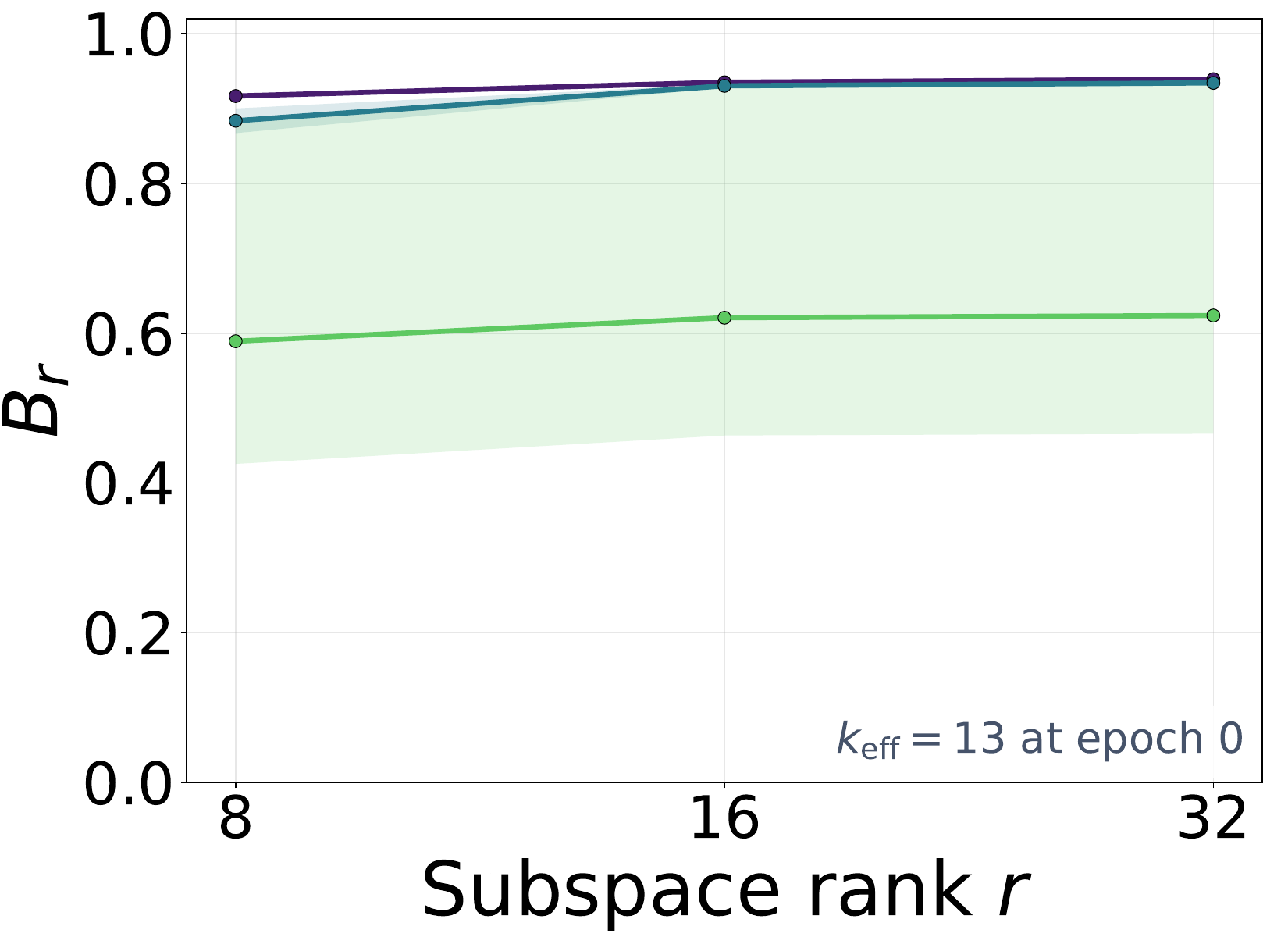} &
        \includegraphics[width=0.31\linewidth]{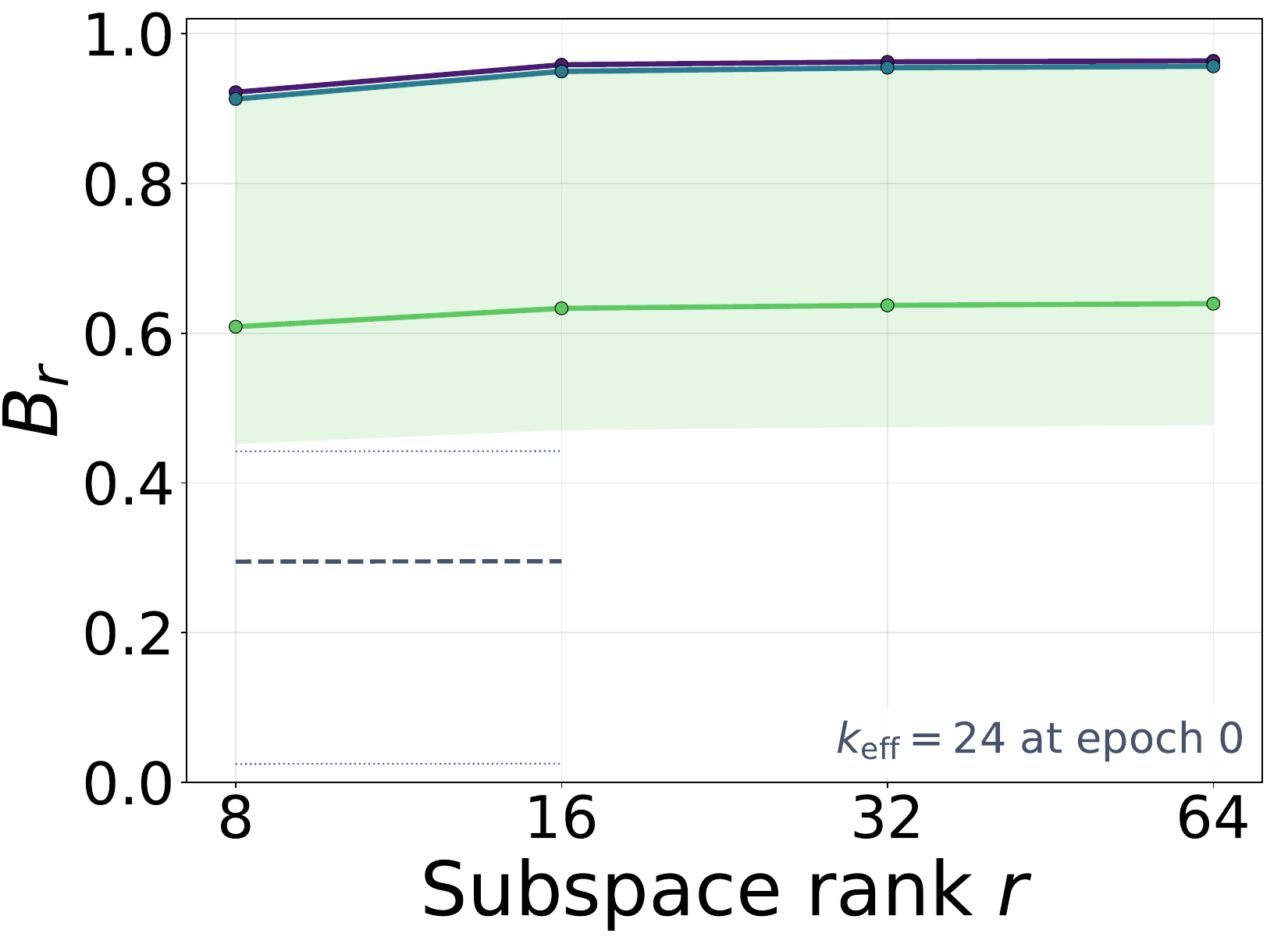} \\
        {\small {\bf (a) VICReg (dSprites)}} &
        {\small {\bf (b) I-JEPA (dSprites)}} &
        {\small {\bf (c) Barlow Twins (dSprites)}} \\[1mm]

        \multicolumn{3}{c}{
            \includegraphics[width=0.70\linewidth]{figures/Br_saturation/dSprites/Br_saturation_legend.pdf}
        } \\[-1mm]
        \includegraphics[width=0.31\linewidth]{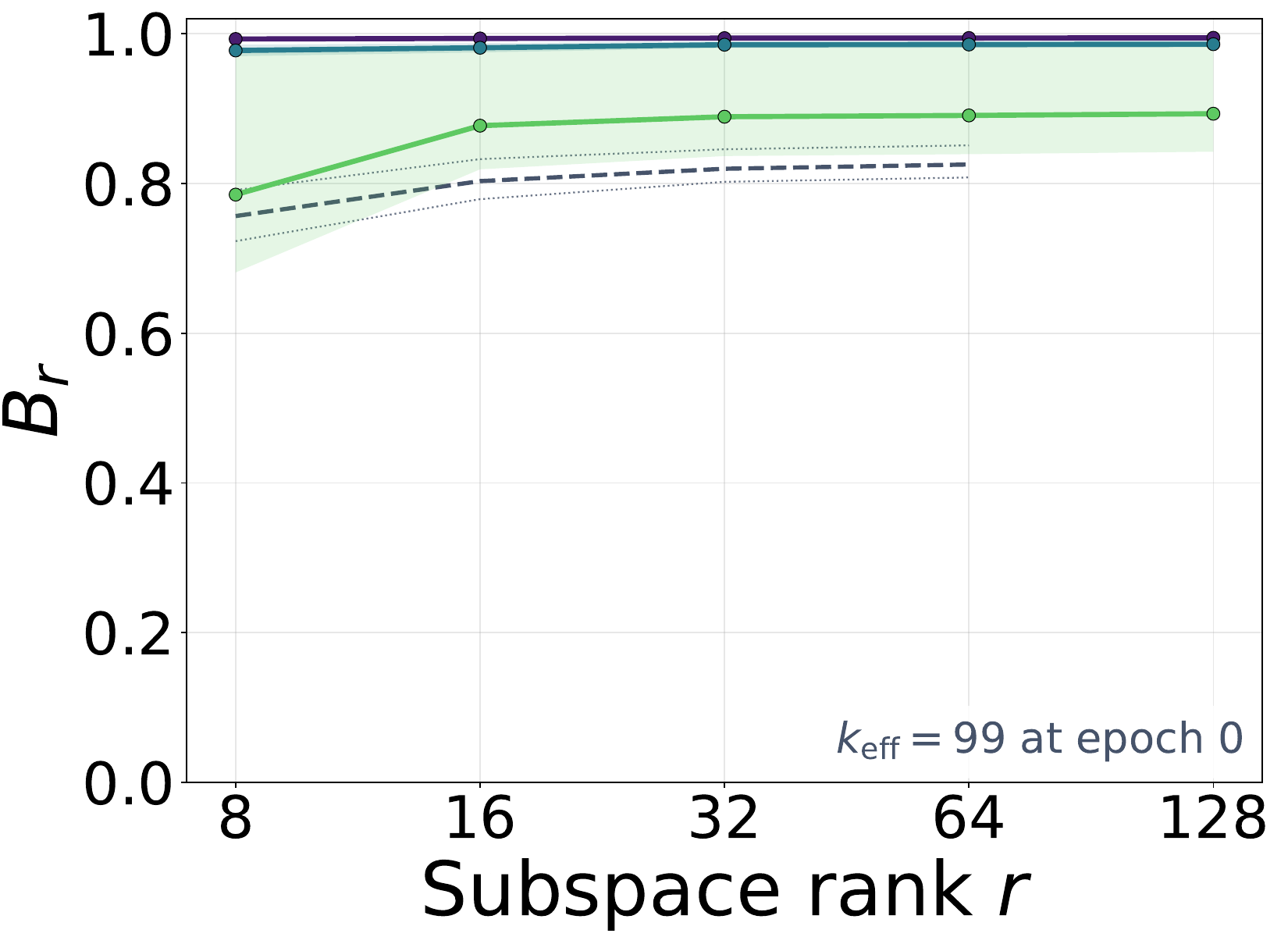} &
        \includegraphics[width=0.31\linewidth]{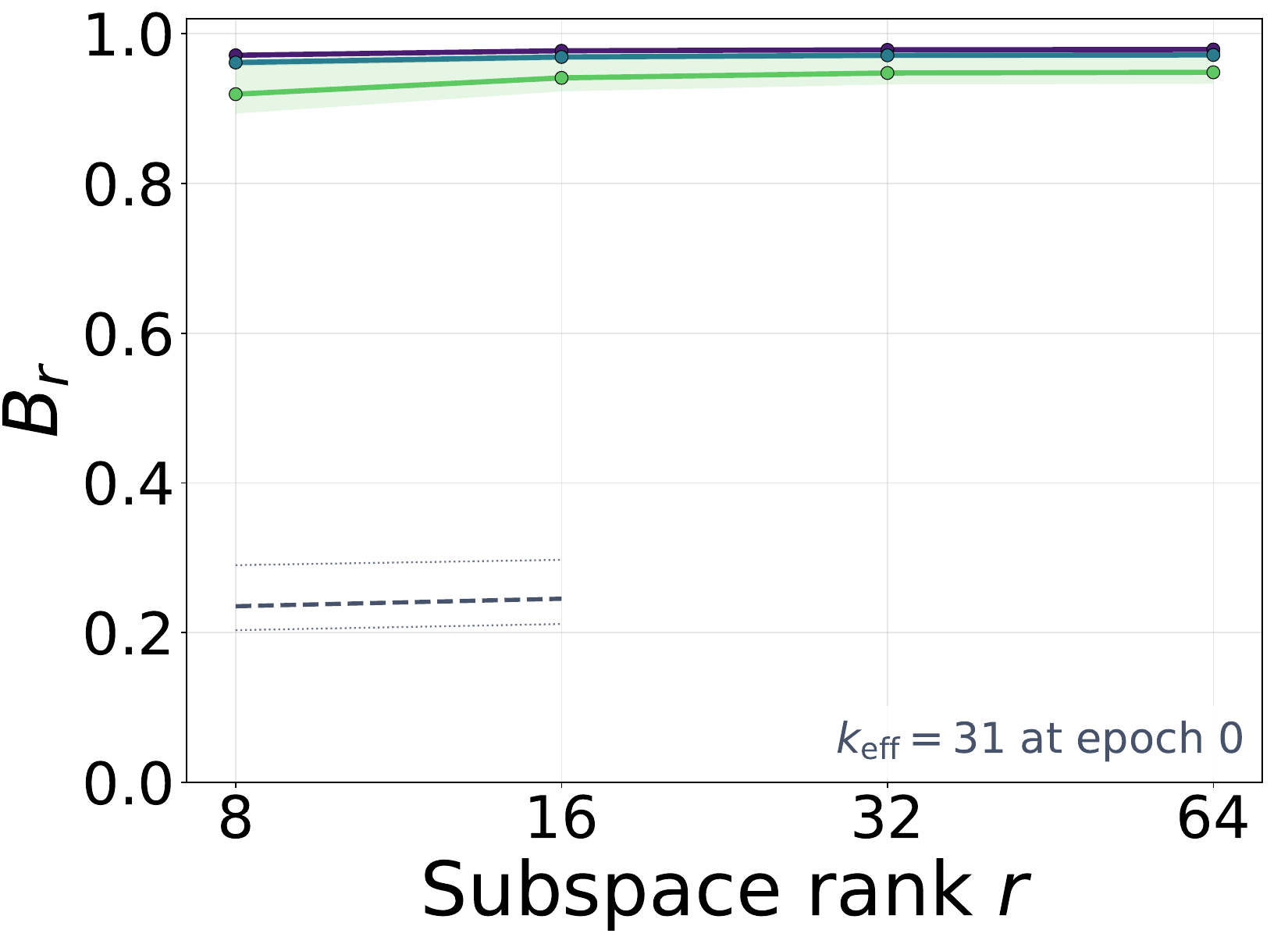} &
        \includegraphics[width=0.31\linewidth]{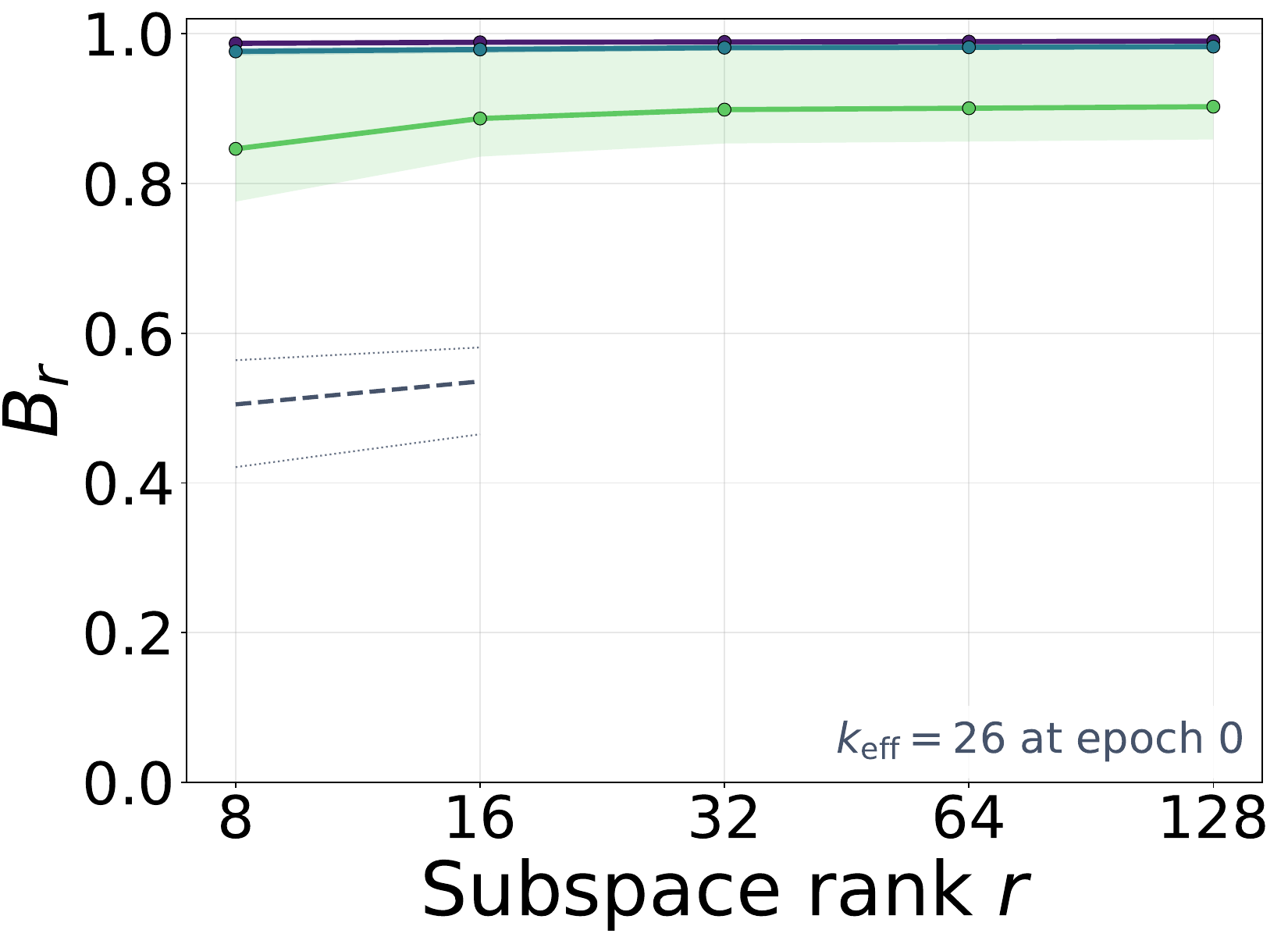} \\
        {\small {\bf (d) VICReg (3DShapes)}} &
        {\small {\bf (e) I-JEPA (3DShapes)}} &
        {\small {\bf (f) Barlow Twins (3DShapes)}}
    \end{tabular}
    \caption{{\bf Semantic recoverability on dSprites and 3DShapes. \enspace}
    Task-specific captured posterior energy \(B_r^{(t)}\) across representation ranks for ImageNet-1K-pretrained VICReg, I-JEPA, and Barlow Twins.}
    \label{fig:app_synth_captured_Br}
\end{figure}

\begin{figure}[t]
    \centering
    \setlength{\tabcolsep}{2pt}
    \begin{tabular}{@{}ccc@{}}
        \multicolumn{3}{c}{\includegraphics[width=0.62\linewidth]{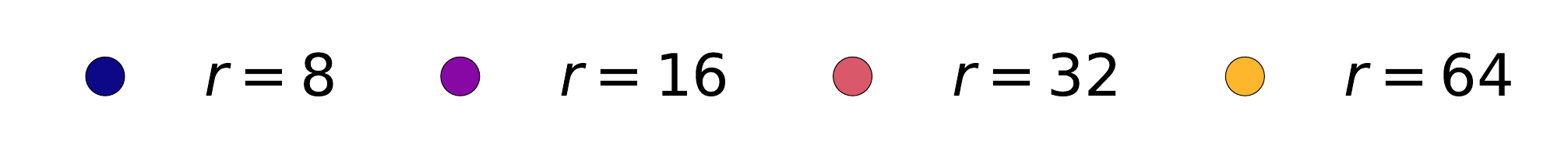}} \\[1mm]
        \includegraphics[width=0.31\linewidth]{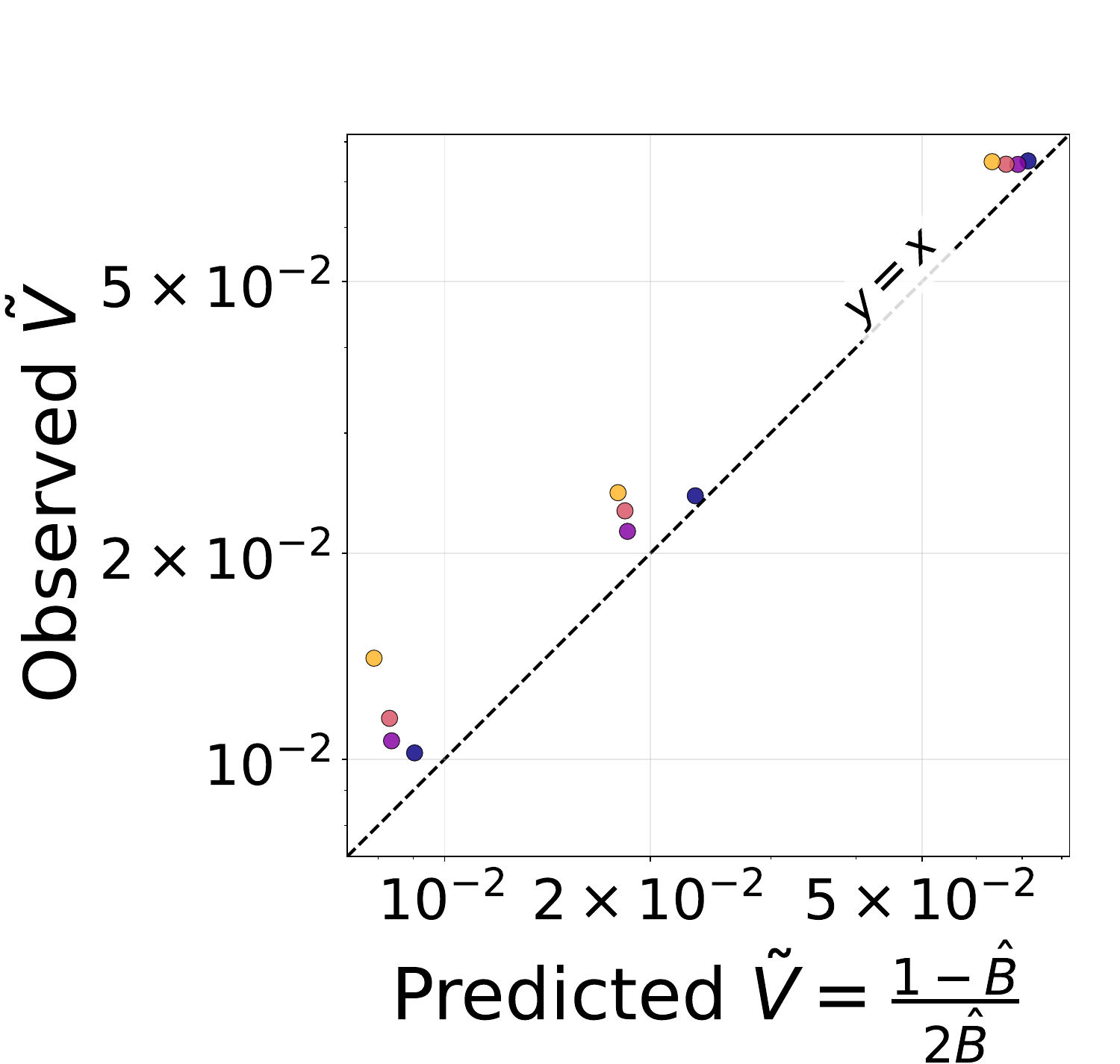} &
        \includegraphics[width=0.31\linewidth]{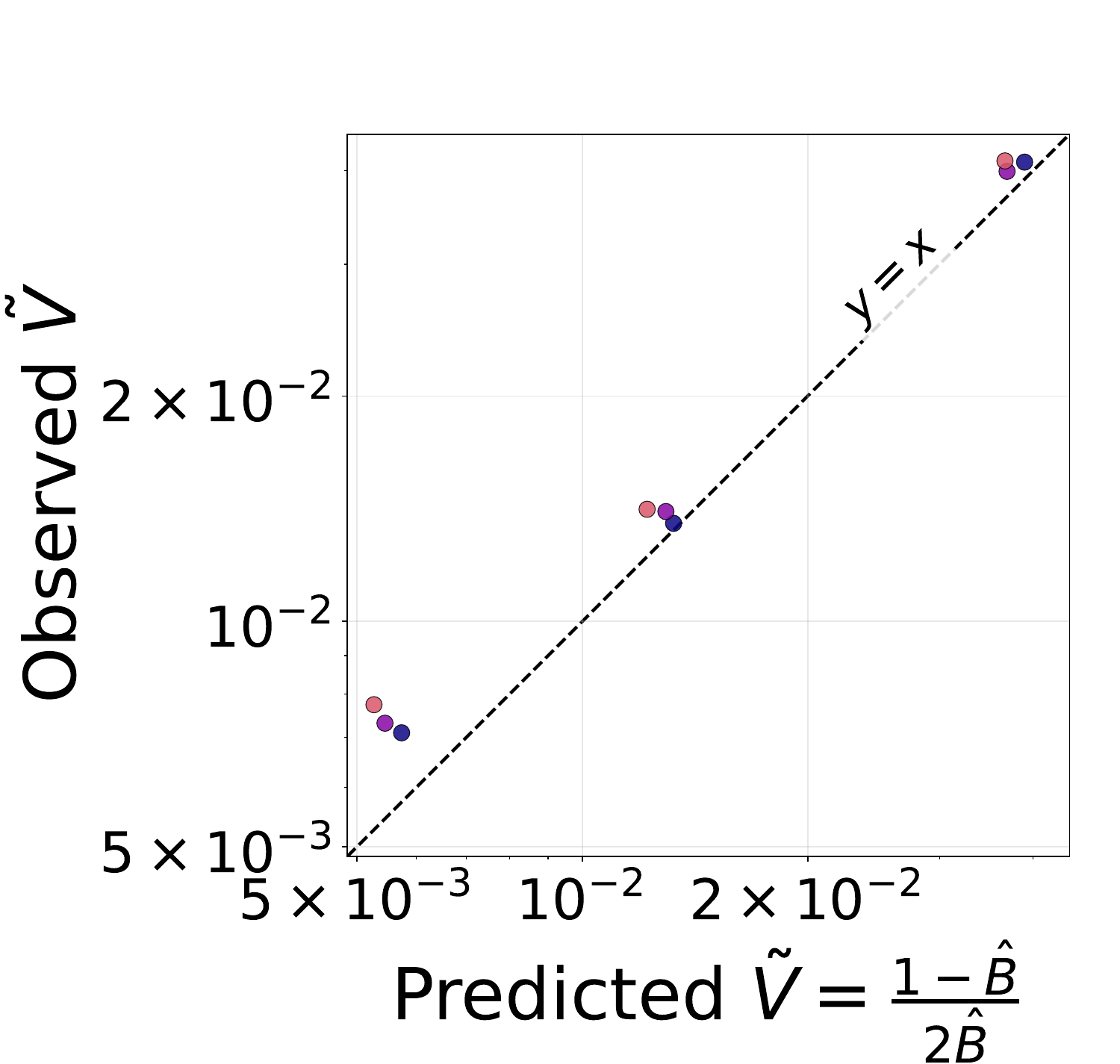} &
        \includegraphics[width=0.31\linewidth]{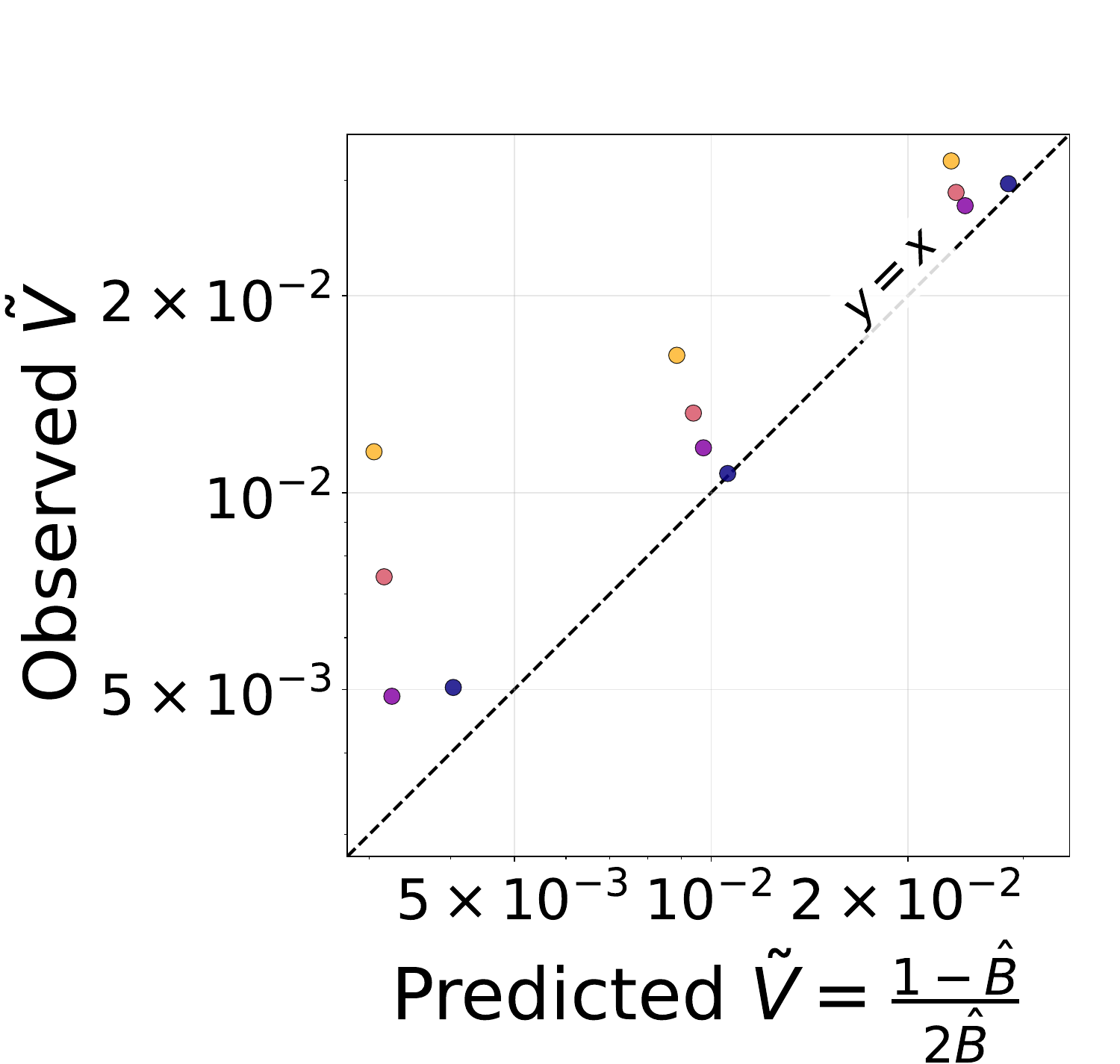} \\
        {\small {\bf (a) VICReg (dSprites)}} &
        {\small {\bf (b) I-JEPA (dSprites)}} &
        {\small {\bf (c) Barlow Twins (dSprites)}} \\[3mm]
        \multicolumn{3}{c}{\includegraphics[width=0.72\linewidth]{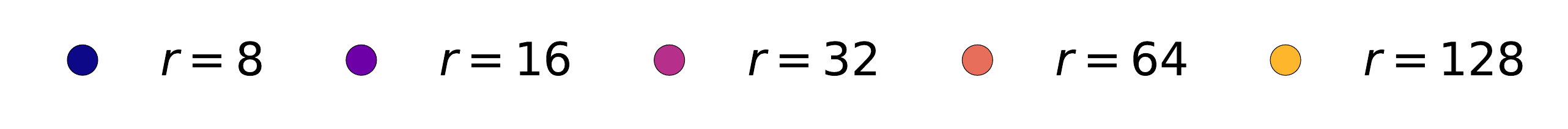}} \\[1mm]
        \includegraphics[width=0.31\linewidth]{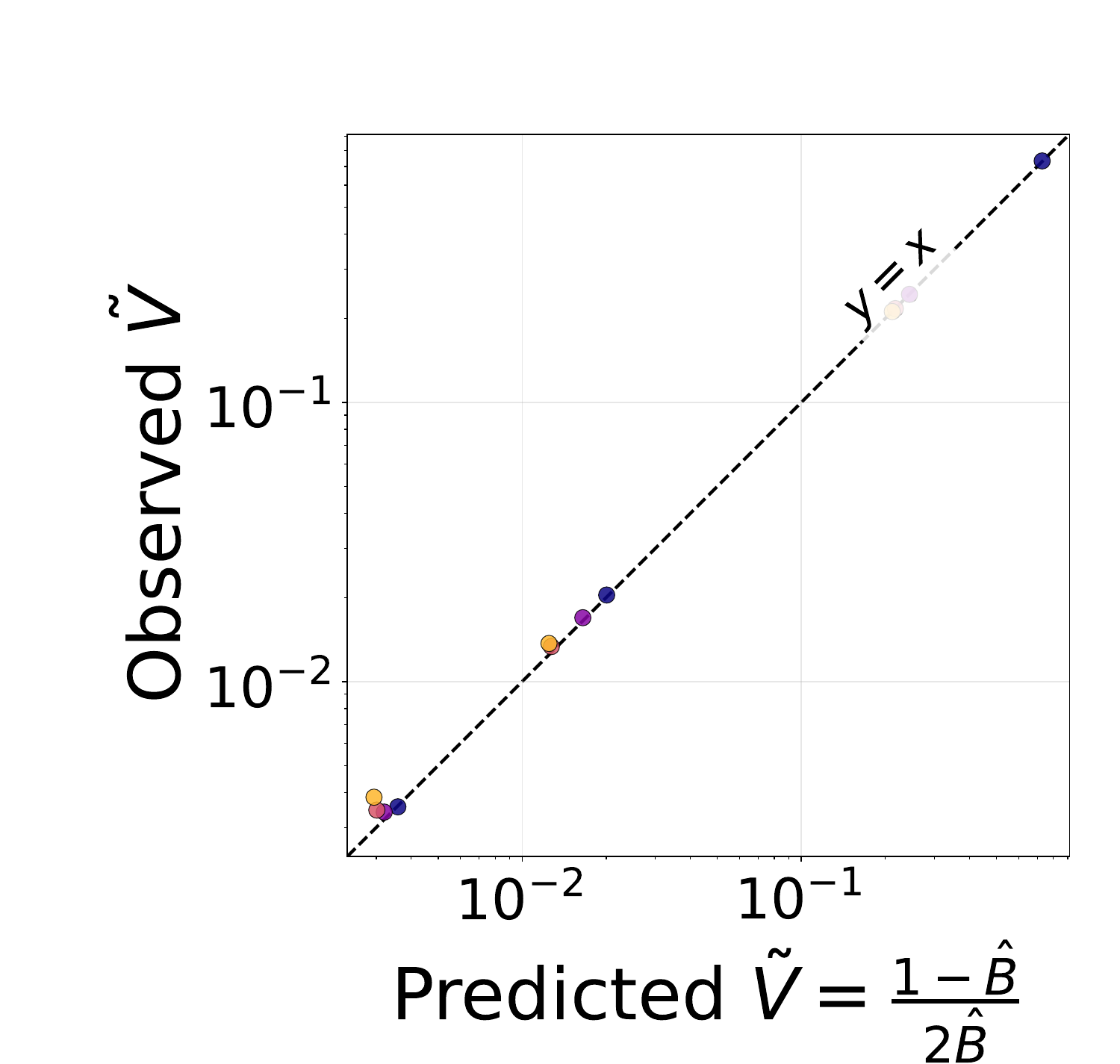} &
        \includegraphics[width=0.31\linewidth]{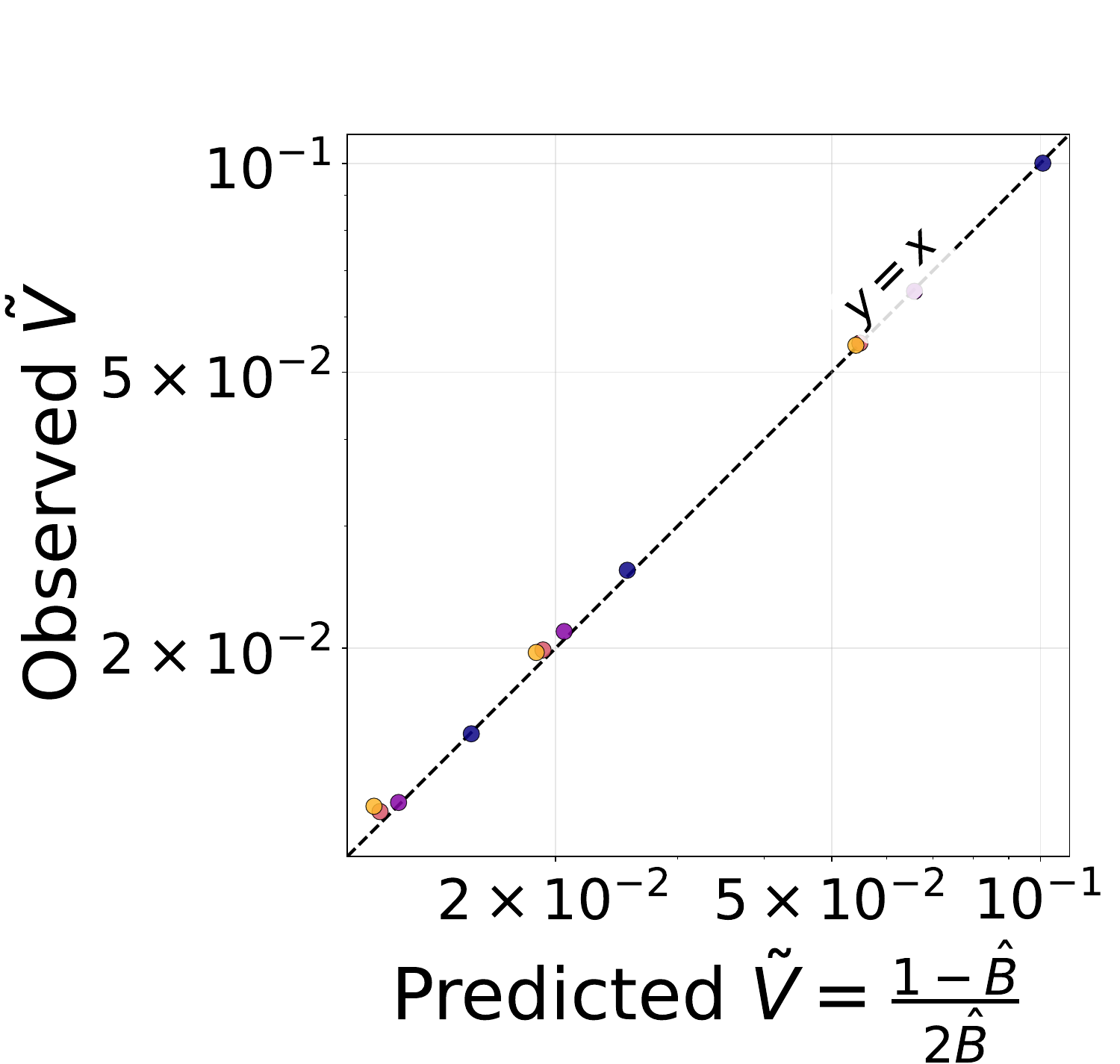} &
        \includegraphics[width=0.31\linewidth]{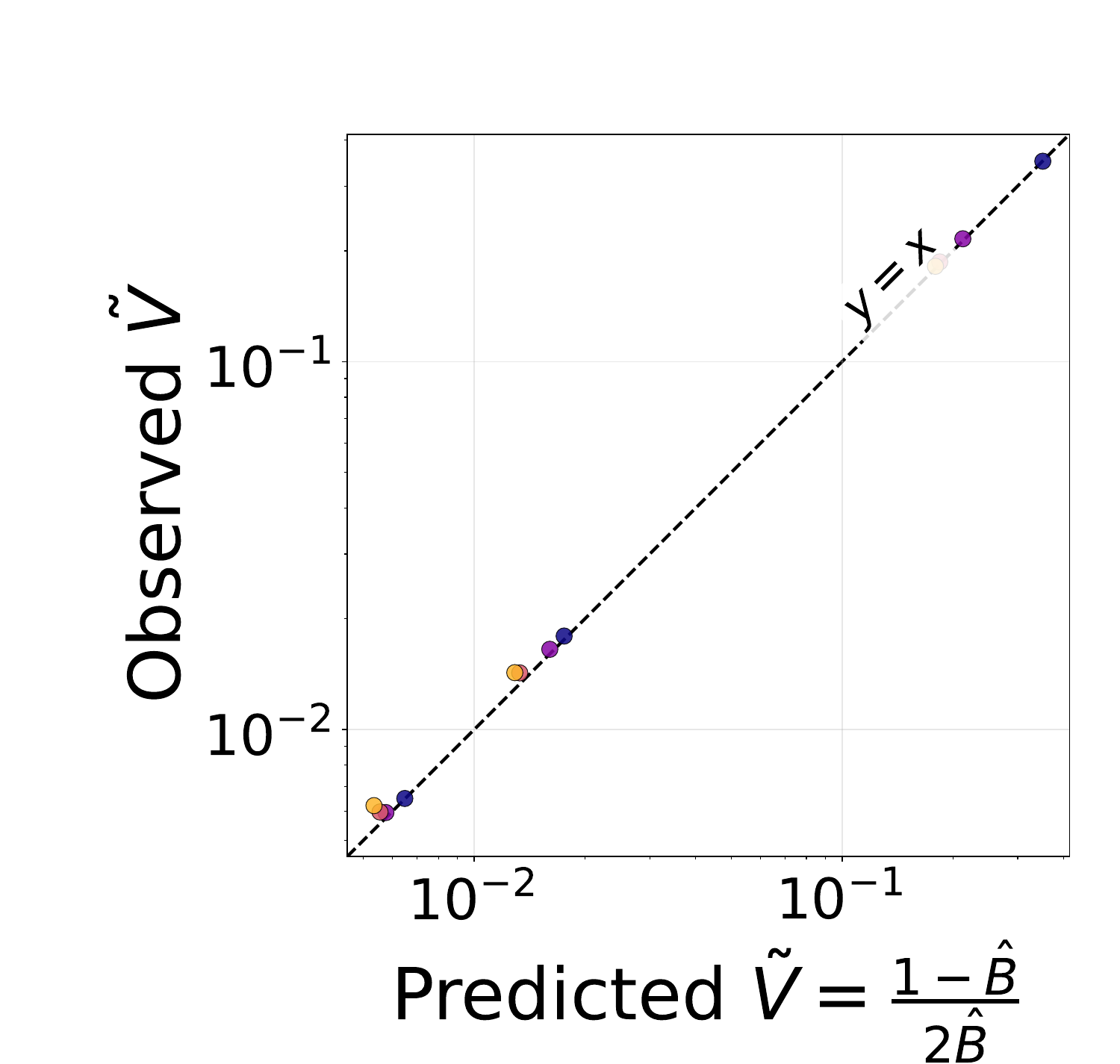} \\
        {\small {\bf (d) VICReg (3DShapes)}} &
        {\small {\bf (e) I-JEPA (3DShapes)}} &
        {\small {\bf (f) Barlow Twins (3DShapes)}}
    \end{tabular}
    \caption{{\bf Predicted versus observed directional CDNV on synthetic datasets.\enspace}
    Results for VICReg, I-JEPA, and Barlow Twins on dSprites and 3DShapes. Each point is
    one attribute at one rank $r$; the prediction is computed on split A and the observed
    directional CDNV on held-out split B. Dashed lines show the theoretical prediction $y=x$.}
    \label{fig:app_synth_dircollapse}
\end{figure}

\begin{figure}[t]
    \centering
    \setlength{\tabcolsep}{2pt}
    \begin{tabular}{@{}ccc@{}}
        \includegraphics[width=0.31\linewidth]{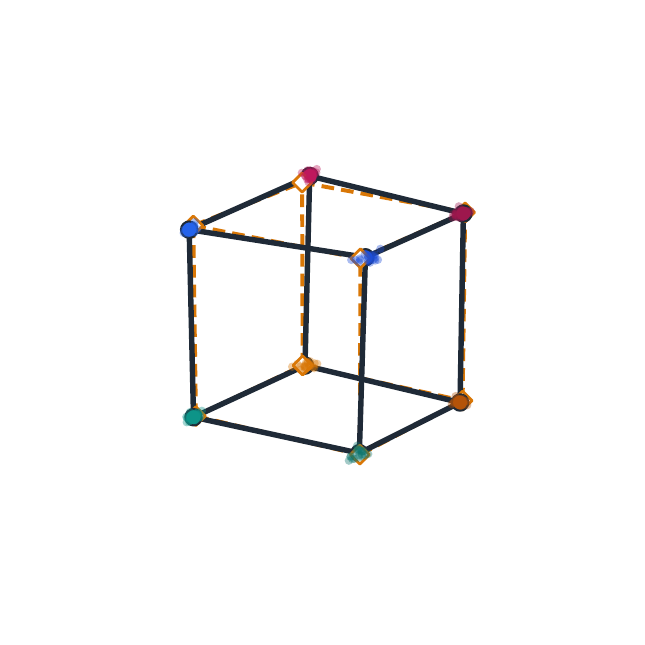} &
        \includegraphics[width=0.31\linewidth]{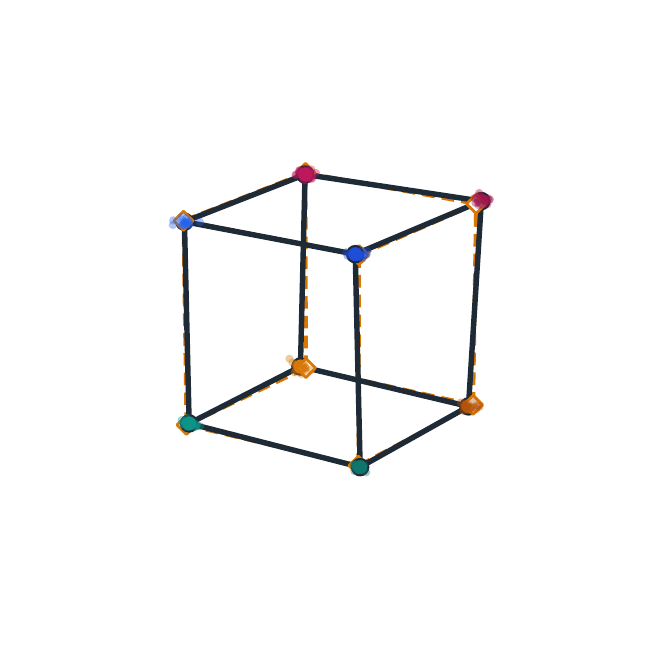} &
        \includegraphics[width=0.31\linewidth]{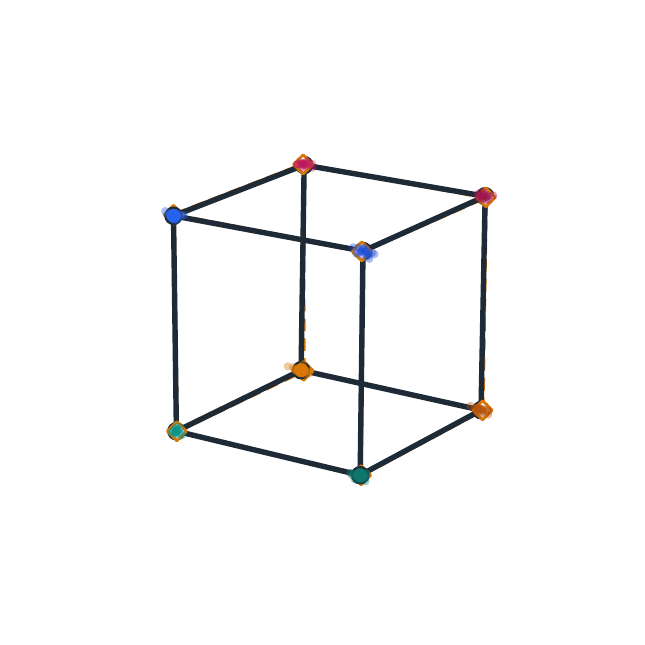} \\
        {\small {\bf (a) VICReg (dSprites)}} &
        {\small {\bf (b) I-JEPA (dSprites)}} &
        {\small {\bf (c) Barlow Twins (dSprites)}} \\[1mm]
        \includegraphics[width=0.31\linewidth]{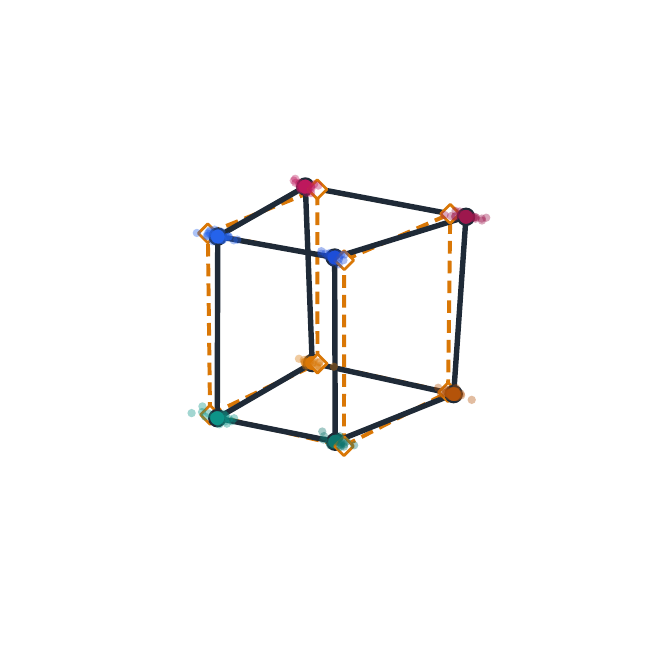} &
        \includegraphics[width=0.31\linewidth]{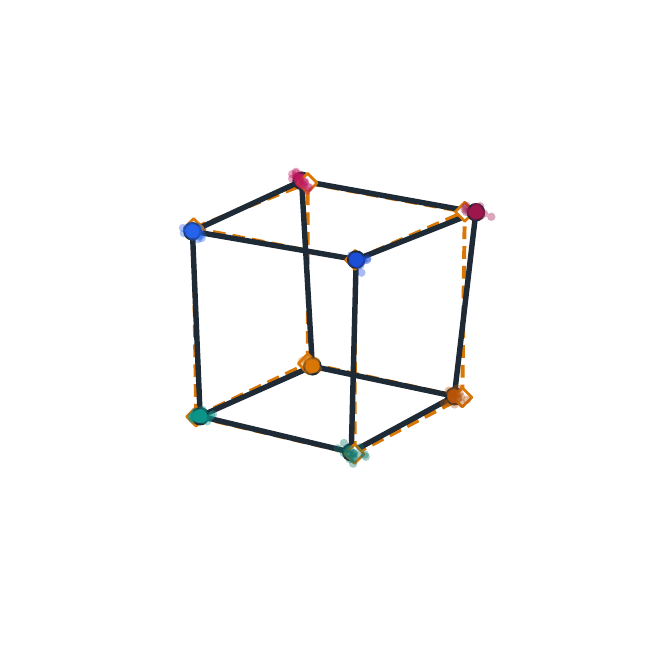} &
        \includegraphics[width=0.31\linewidth]{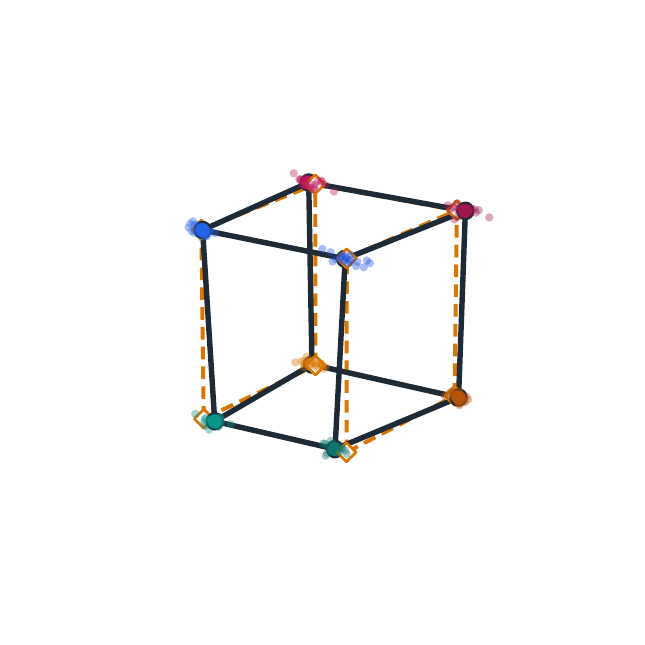} \\
        {\small {\bf (d) VICReg (3DShapes)}} &
        {\small {\bf (e) I-JEPA (3DShapes)}} &
        {\small {\bf (f) Barlow Twins (3DShapes)}}
    \end{tabular}
    \caption{{\bf Multitask centroid geometry on synthetic datasets. \enspace}
    Held-out joint centroids compared with the hyperrectangles predicted by Thm.~\ref{thm:axis_and_centroids}.}
    \label{fig:app_synth_hyperrec}
\end{figure}

\section{Covariance-Regularized SSL Objectives}
\label{app:soft_whitening}

In the main text we studied the population surrogate \eqref{eq:ssl_pop_obj_multi}, which enforces strict whitening
\(
\E[F(X)F(X)^\top]=I_r
\).
This yields an exact strict-whitening theory in which the main directional geometry is governed by the single semantic capture parameter
\(
B_r=\|\Pi_r\eta\|_{L^2(P_X)}^2
\).
Modern SSL objectives typically do not enforce exact whitening, but instead penalize covariance deviations through variance/covariance regularizers, as in Barlow Twins and VICReg. In this appendix we show that a natural covariance-regularized surrogate preserves the same spectral subspace identified by the two-view operator \(T\) in Prop.~\ref{prop:same_instance_ssl_posterior}, but the resulting geometry is no longer governed by a single scalar. Instead, three explicit quantities \((B_\gamma,C_\gamma,S_\gamma)\) appear, with the strict-whitening formulas recovered as \(\gamma\to\infty\).

\subsection{Population covariance-regularized objective}
\label{app:soft_whitening_obj}

We consider the following population surrogate:
\begin{equation}
\label{eq:ssl_pop_obj_reg}
\max_{F:\cX\to\R^r}
\ \E\!\left[\langle F(X^{(1)}),F(X^{(2)})\rangle\right]
-\frac{\gamma}{2}\left\|\E[F(X)F(X)^\top]-I_r\right\|_F^2
\quad
\text{s.t.}
\quad
\E[F(X)]=0,
\end{equation}
where \(\gamma>0\) controls the strength of the covariance penalty.
As \(\gamma\to\infty\), the penalty enforces \(\E[F(X)F(X)^\top]\to I_r\), recovering the strict whitening setting.
The centering constraint \(\E[F(X)]=0\) models the fact that practical implementations explicitly center features and, in a population analysis, can be imposed without loss of generality by working with centered features
\(
\bar F(X):=F(X)-\E[F(X)]
\)
and rewriting the objective in terms of \(\bar F\).

\subsection{Operator form and canonical modes}
\label{app:soft_whitening_operator}

Recall the two-view conditional expectation operator $(Tf)(x):=\E[f(X^{(2)})\mid X^{(1)}=x]$, $T:L^2(P_X)\to L^2(P_X)$. For any \(f,g\in L^2(P_X)\), we have
\begin{equation}
\label{eq:app_alignment_quadratic_form}
\E\big[f(X^{(1)})g(X^{(2)})\big]
=
\langle f,Tg\rangle_{L^2(P_X)}.
\end{equation}
In particular,
\[
\E\big[f(X^{(1)})f(X^{(2)})\big]
=
\langle f,Tf\rangle_{L^2(P_X)}.
\]

Let \((\lambda_j,\psi_j)_{j\ge1}\) be the orthonormal eigenbasis of \(T\) on \(L_0^2(P_X)\) from Prop.~\ref{prop:same_instance_ssl_posterior}, ordered as
\[
\lambda_1\ge\lambda_2\ge\cdots\ge 0.
\]
Thus
\[
T\psi_j=\lambda_j\psi_j,
\qquad
\langle \psi_i,\psi_j\rangle_{L^2(P_X)}=\delta_{ij},
\qquad
\E[\psi_j(X)]=0.
\]
These are exactly the canonical observable modes identified in Prop.~\ref{prop:same_instance_ssl_posterior}. This appendix keeps the same operator and asks how covariance regularization rescales its top-\(r\) \(T\)-eigenspace.

\subsection{Structure of the population optimum}
\label{app:soft_whitening_structure}

We show that covariance regularization preserves the \(T\)-selected spectral subspace and only introduces mode-dependent scalings within that subspace.

\begin{restatable}[Structure of the regularized optimum]{proposition}{regStructure}
\label{prop:regularized_structure}
Assume \(\lambda_r>0\) and \(\lambda_r>\lambda_{r+1}\).
Then any population maximizer \(F^\star_\gamma\) of \eqref{eq:ssl_pop_obj_reg} can be written as
\begin{equation}
\label{eq:Fgamma_structure}
F^\star_\gamma(x)
=
R\,\big(s_1\psi_1(x),\dots,s_r\psi_r(x)\big),
\end{equation}
where \(R\in\R^{r\times r}\) is orthogonal and \(s_1,\dots,s_r\ge0\) are scalars.
Moreover, for the quadratic penalty in \eqref{eq:ssl_pop_obj_reg}, the optimal scalings satisfy
\[
s_j^2 = 1+\frac{\lambda_j}{\gamma},
\qquad
j=1,\dots,r.
\]
\end{restatable}

Prop.~\ref{prop:regularized_structure} implies that the regularized optimum can be reduced to the whitened setting by an explicit population rewhitening step, and therefore inherits the same multitask geometric structure established in the main text. Let \(F^\star_\gamma\) be a population maximizer of \eqref{eq:ssl_pop_obj_reg}, and define
\[
\Sigma_\gamma:=\Cov(F^\star_\gamma(X)),
\qquad
\widetilde F_\gamma(x):=\Sigma_\gamma^{-1/2}F^\star_\gamma(x).
\]
For the maximizer in Prop.~\ref{prop:regularized_structure}, we have \(s_j^2=1+\lambda_j/\gamma>0\), hence
\(\Sigma_\gamma\succ 0\) and \(\Sigma_\gamma^{-1/2}\) is well-defined.
Then \(\widetilde F_\gamma\) is feasible for the strictly whitened objective \eqref{eq:ssl_pop_obj_multi} and is population-optimal. Indeed, by Prop.~\ref{prop:regularized_structure} we may write
\[
F^\star_\gamma(x)
=
R\,\mathrm{diag}(s_1,\dots,s_r)\,
\big(\psi_1(x),\dots,\psi_r(x)\big),
\]
so
\[
\Sigma_\gamma
=
R\,\mathrm{diag}(s_1^2,\dots,s_r^2)\,R^\top,
\qquad
\Sigma_\gamma^{-1/2}
=
R\,\mathrm{diag}(s_1^{-1},\dots,s_r^{-1})\,R^\top,
\]
and therefore
\[
\widetilde F_\gamma(x)
=
R\,\big(\psi_1(x),\dots,\psi_r(x)\big).
\]
By Prop.~\ref{prop:same_instance_ssl_posterior}, any representation of the form
\(
R(\psi_1,\dots,\psi_r)
\)
is feasible for \eqref{eq:ssl_pop_obj_multi} and achieves the optimal value \(\sum_{j=1}^r \lambda_j\), hence \(\widetilde F_\gamma\) is a whitened population optimum.

Consequently, all main-text results that rely only on the whitened-optimum structure, including the multitask near-orthogonality and factorial-centroid geometry in Thm.~\ref{thm:axis_and_centroids}, apply verbatim to \(\widetilde F_\gamma\). In the original feature coordinates of \(F^\star_\gamma\), this same structure is obtained by applying the linear map \(\Sigma_\gamma^{1/2}\), so the hyperrectangle in rewhitened probe coordinates becomes a parallelepiped in the raw \(F^\star_\gamma\) coordinates. This highlights that covariance regularization preserves the operator-selected semantic organization, but can distort Euclidean geometry unless one explicitly rewhitens features, or equivalently uses a Mahalanobis nearest-centroid rule with metric \(\Sigma_\gamma^{-1}\).

\subsection{How the CDNV formulas change}
\label{app:soft_whitening_cdnv}

Let \(Y=y(C)\in\{\pm1\}\) be balanced and define the posterior score
\[
\eta(x)=\E[Y\mid X=x].
\]
Write the posterior directly in the \(T\)-eigenbasis:
\[
\eta=\sum_{j\ge1}\beta_j\psi_j,
\qquad
\beta_j:=\langle \eta,\psi_j\rangle_{L^2(P_X)}.
\]
Then, by Cor.~\ref{cor:posterior_decomp_ssl_basis},
\[
\Pi_r\eta=\sum_{j=1}^r \beta_j\psi_j,
\qquad
B_r=\|\Pi_r\eta\|_{L^2(P_X)}^2=\sum_{j=1}^r \beta_j^2.
\]

Define the scalars
\begin{align}
\label{eq:soft_scalars}
B_\gamma
&:=\sum_{j=1}^r \beta_j^2 s_j^2,
&
C_\gamma
&:=\sum_{j=1}^r \beta_j^2 s_j^4,
&
S_\gamma
&:=\sum_{j=1}^r s_j^2.
\end{align}

\begin{restatable}{proposition}{regCDNV}
\label{prop:regularized_cdnv}
Let \(F^\star_\gamma\) be any population maximizer of \eqref{eq:ssl_pop_obj_reg}, written as in \eqref{eq:Fgamma_structure}.
Let \(\mu_\pm\) and \(\Sigma_\pm\) be the class means and class covariances defined as in the main text. Then \(\mu_-=-\mu_+\) and \(\Delta=\mu_+-\mu_-=2\mu_+\).
Moreover, when \(B_\gamma>0\), the CDNV and directional CDNV defined in \eqref{eq:problem_dcdnv_vector_binary} satisfy
\begin{equation}
\label{eq:regularized_cdnv_laws}
\tilde V_{F^\star_\gamma}
=
\frac12\left(\frac{C_\gamma}{B_\gamma^2}-1\right),
\qquad
V_{F^\star_\gamma}
=
\frac12\left(\frac{S_\gamma}{B_\gamma}-1\right).
\end{equation}
When \(B_\gamma=0\), we set both quantities to \(+\infty\) by the same convention as in the main text.
\end{restatable}

For \(B_\gamma>0\), Prop.~\ref{prop:regularized_cdnv} can be plugged directly into a general CDNV-based few-shot guarantee. Namely, with \(F=F^\star_\gamma\), substituting \eqref{eq:regularized_cdnv_laws} yields the explicit \(\gamma\)-dependent bound
\begin{equation}
\label{eq:app_luthra_bound_gamma}
\mathrm{err}^{\mathrm{NCC}}_{m}(F^\star_\gamma)
\le
4\Big(\frac{C_\gamma}{B_\gamma^2}-1\Big)
+\frac{8}{\sqrt m}\sqrt{\frac12\Big(\frac{S_\gamma}{B_\gamma}-1\Big)}
+\Big(\frac{8}{\sqrt m}+\frac{4}{m}\Big)\frac12\Big(\frac{S_\gamma}{B_\gamma}-1\Big).
\end{equation}
In particular, this bound depends on \(\gamma\) only through the three scalars \(B_\gamma,C_\gamma,S_\gamma\) defined in \eqref{eq:soft_scalars}.

Thus \(\gamma\) controls a concrete tradeoff in \eqref{eq:app_luthra_bound_gamma}. The leading directional term depends on \(C_\gamma/B_\gamma^2\), while the finite-\(m\) terms depend on \(S_\gamma/B_\gamma\). Under the quadratic penalty
\(
s_j^2=1+\lambda_j/\gamma
\),
decreasing \(\gamma\) inflates all selected modes, with larger inflation on more stable modes. This increases total variance through \(S_\gamma\), but it also increases signal strength along the task-relevant direction through \(B_\gamma\) and \(C_\gamma\). The resulting few-shot behavior is therefore governed by how the posterior coefficients \(\beta_j\) overlap with the modes that are most amplified by the regularizer.

In the strict-whitening limit \(s_j\to 1\), we recover
\[
B_\gamma\to B_r,
\qquad
C_\gamma\to B_r,
\qquad
S_\gamma\to r,
\]
and therefore
\[
\tilde V_{F^\star_\gamma}\to \frac{1-B_r}{2B_r},
\qquad
V_{F^\star_\gamma}\to \frac{r-B_r}{2B_r},
\]
exactly as in Prop.~\ref{prop:Br_probe_dcdnv}.

\section{Prediction-Based SSL Objectives as Reduced-Rank Regression}
\label{app:rrr_byol_jepa}

This appendix gives a simple population-level connection between prediction-based SSL objectives, such as BYOL, SimSiam, and JEPA, and classical reduced-rank regression / partial least squares. We do not aim to model algorithmic details such as stop-gradient or EMA teachers. Instead, we study a natural whitened population surrogate and show that its optimizer selects the same \(T\)-spectral subspace as \eqref{eq:ssl_pop_obj_multi}. As a result, the main-text conclusions that depend only on the selected subspace, and therefore on the induced capture value \(B_r\), carry over directly to this prediction-based setting.

\subsection{A stylized whitened prediction surrogate}

Let \(F:\cX\to\R^r\) be a shared representation used on both views, and let \(W\in\R^{r\times r}\) be a linear predictor.
Consider the population objective
\begin{equation}
\label{eq:rrr_pop_obj}
\min_{\substack{F:\cX\to\R^r \\ W\in\R^{r\times r}}}
\ \E\Big[\big\|F(X^{(2)})-W F(X^{(1)})\big\|_2^2\Big]
\quad\text{s.t.}\quad
\E[F(X)]=0,\ \ \E[F(X)F(X)^\top]=I_r,
\end{equation}
where \((X^{(1)},X^{(2)})\) is a same-instance pair with conditionally i.i.d.\ views as in Sec.~\ref{subsec:latent_instance_model}.

Define
\begin{equation}
\label{eq:rrr_cross_cov}
C_{12}(F):=\E\big[F(X^{(1)})F(X^{(2)})^\top\big],
\qquad
C_{21}(F):=C_{12}(F)^\top.
\end{equation}

\begin{proposition}[RRR reduction to a spectral objective]
\label{prop:rrr_reduction}
Fix any feasible \(F\) in \eqref{eq:rrr_pop_obj}. The unique minimizer over \(W\) is
\[
W^\star(F)=C_{21}(F)=C_{12}(F)^\top,
\]
and the resulting minimum equals
\begin{equation}
\label{eq:rrr_reduction_value}
\min_W\ \E\Big[\big\|F(X^{(2)})-W F(X^{(1)})\big\|_2^2\Big]
=
r-\|C_{21}(F)\|_F^2.
\end{equation}
Hence minimizing \eqref{eq:rrr_pop_obj} is equivalent to maximizing \(\|C_{21}(F)\|_F^2\) over feasible \(F\).
\end{proposition}

\begin{proof}
Using \(\E[F(X)F(X)^\top]=I_r\),
\[
\E\|F(X^{(2)})-WF(X^{(1)})\|_2^2
=r-2\Tr(WC_{12}(F))+\|W\|_F^2
=r+\|W-C_{21}(F)\|_F^2-\|C_{21}(F)\|_F^2.
\]
Thus the unique minimizer is \(W^\star(F)=C_{21}(F)=C_{12}(F)^\top\), and the minimum is \(r-\|C_{21}(F)\|_F^2\).
\end{proof}

\subsection{Selected subspace and relationship to the alignment objective}

For scalar \(f,g\in L^2(P_X)\) with \(\E[f(X)]=\E[g(X)]=0\), the same-instance model gives
\[
\E[f(X^{(2)})g(X^{(1)})]
=
\langle f,Tg\rangle_{L^2(P_X)}.
\]
Thus, if \(F=(f_1,\dots,f_r)\) is feasible and the coordinates \(\{f_k\}\) are orthonormal in \(L^2(P_X)\), then
\begin{equation}
\label{eq:rrr_matrix_L}
\big(C_{21}(F)\big)_{ij}
=
\E[f_i(X^{(2)})f_j(X^{(1)})]
=
\langle f_i,Tf_j\rangle_{L^2(P_X)}.
\end{equation}
Consequently, \(\|C_{21}(F)\|_F^2\) is the squared Hilbert--Schmidt norm of the compression \(P_UTP_U\), where \(U=\mathrm{span}\{f_1,\dots,f_r\}\subset L_0^2(P_X)\).

Let \((\lambda_j,\psi_j)_{j\ge1}\) be the orthonormal eigenbasis of \(T\) from Prop.~\ref{prop:same_instance_ssl_posterior}, with \(\lambda_1\ge\lambda_2\ge\cdots\ge 0\).

\begin{proposition}[Prediction objective selects the top-\(r\) operator subspace]
\label{prop:rrr_top_subspace}
Assume \(\lambda_r>0\) and \(\lambda_r>\lambda_{r+1}\).
Then the maximum of \(\|C_{21}(F)\|_F^2\) over feasible \(F\) equals \(\sum_{j=1}^r \lambda_j^2\), and any maximizer \(F^\star\) has coordinates spanning \(\mathrm{span}\{\psi_1,\dots,\psi_r\}\).
In particular, there exists an orthogonal \(R\in\R^{r\times r}\) such that
\begin{equation}
\label{eq:rrr_opt_form}
F^\star(x)=R\,(\psi_1(x),\dots,\psi_r(x)).
\end{equation}
\end{proposition}

\begin{proof}
Let \(U=\mathrm{span}\{f_1,\dots,f_r\}\) and let \(P_U\) be the orthogonal projector onto \(U\). In the orthonormal basis \(\{f_j\}\), \(C_{21}(F)\) is the matrix of \(P_UTP_U\), and hence
\[
\|C_{21}(F)\|_F^2
=\|P_UTP_U\|_{\mathrm{HS}}^2
\le \|TP_U\|_{\mathrm{HS}}^2
=\Tr(P_UT^2)
\le \sum_{j=1}^r\lambda_j^2.
\]
The last inequality is Ky Fan's variational principle applied to the positive semidefinite operator \(T^2\). Equality is attained for \(U=\mathrm{span}\{\psi_1,\dots,\psi_r\}\), and the eigengap \(\lambda_r>\lambda_{r+1}\) makes this maximizing subspace unique.
\end{proof}

\subsection{Consequences for recoverability, CDNV, and downstream guarantees}

Prop.~\ref{prop:rrr_top_subspace} shows that the whitened prediction surrogate \eqref{eq:rrr_pop_obj} selects the same \(r\)-dimensional \(T\)-eigenspace as \eqref{eq:ssl_pop_obj_multi}. In particular, any maximizer \(F^\star\) of \eqref{eq:rrr_pop_obj} has coordinates spanning \(\mathrm{span}\{\psi_1,\dots,\psi_r\}\), so \(F^\star\) is feasible for \eqref{eq:ssl_pop_obj_multi}. Since the strictly whitened alignment objective \eqref{eq:ssl_pop_obj_multi} depends only on the chosen \(r\)-dimensional subspace and is maximized by any orthonormal basis of the top-\(r\) eigenspace, as shown in Prop.~\ref{prop:same_instance_ssl_posterior}, \(F^\star\) is also population-optimal for \eqref{eq:ssl_pop_obj_multi} and achieves the optimal value \(\sum_{j=1}^r \lambda_j\). Therefore, all main-text results that apply to whitened population optima of \eqref{eq:ssl_pop_obj_multi} apply directly to \(F^\star\).

In particular, for any downstream binary label \(Y=y(C)\) with posterior score \(\eta(x)=\E[Y\mid X=x]\), writing
\[
\eta=\sum_{j\ge1}\beta_j\psi_j,
\]
the capture value remains
\[
B_r=\sum_{j=1}^r \beta_j^2,
\]
and the closed-form directional CDNV and CDNV laws for whitened optima from Prop.~\ref{prop:Br_probe_dcdnv} hold unchanged for \(F^\star\). Consequently, the same few-shot NCC guarantees from Thm.~\ref{thm:ncc_bound_direct_via_Br} and the same multitask geometry conclusions from Thm.~\ref{thm:axis_and_centroids} follow immediately for this stylized prediction-based surrogate.

\newpage

\section{Proofs}

\subsection{Proofs of the Results in the Main Text}

\cdnv*

\begin{proof}
Let \(\mu_s:=\E[F(X)\mid Y=s]\) and \(\Sigma_s:=\Cov(F(X)\mid Y=s)\) for \(s\in\{\pm1\}\), and set \(\Delta:=\mu_+-\mu_-\) with \(u:=\Delta/\|\Delta\|_2\) when \(\Delta\neq 0\). Balancedness of \(Y\) and \(\E[F(X)]=0\) give \(\tfrac12\mu_++\tfrac12\mu_-=0\), so \(\mu_-=-\mu_+\) and \(\Delta=2\mu_+\). Moreover, \(\E[YF(X)]=\tfrac12\mu_+-\tfrac12\mu_-=\mu_+\), and by iterated expectation this also equals \(\E[\eta(X)F(X)]\).

Write \(F=(f_1,\dots,f_r)\). Centering and whitening imply that \(\{f_1,\dots,f_r\}\) is orthonormal in \(L^2(P_X)\), so \(\cS_F=\mathrm{span}\{f_1,\dots,f_r\}\) is an \(r\)-dimensional subspace and \((\mu_+)_i=\E[\eta(X)f_i(X)]=\langle\eta,f_i\rangle\). Hence \(\Pi_{\cS_F}\eta=\sum_i(\mu_+)_if_i\), and by orthonormality \(\|\mu_+\|_2^2=\|\Pi_{\cS_F}\eta\|_{L^2}^2=B(F)\), giving \(\|\Delta\|_2^2=4B(F)\). If \(B(F)=0\) then \(\Delta=0\) and both ratios are \(+\infty\); assume \(B(F)>0\) below.

By the law of total covariance and whitening, \(I_r=\tfrac12(\Sigma_++\Sigma_-)+\Cov(\mu_Y)\), where \(\mu_Y\in\{\mu_+,\mu_-\}\) uniformly. Since \(\mu_-=-\mu_+\), \(\Cov(\mu_Y)=\mu_+\mu_+^\top\), yielding
\begin{equation}
\label{eq:proof_Ssum}
\Sigma_++\Sigma_-=2(I_r-\mu_+\mu_+^\top).
\end{equation}
Taking traces in \eqref{eq:proof_Ssum} gives \(\Tr(\Sigma_+)+\Tr(\Sigma_-)=2(r-B(F))\), so \(V_F=(r-B(F))/(2B(F))\) after dividing by \(\|\Delta\|_2^2=4B(F)\). For the directional ratio, \(u=\mu_+/\|\mu_+\|_2\) gives \(u^\top(\Sigma_++\Sigma_-)u=2(1-\|\mu_+\|_2^2)=2(1-B(F))\) by \eqref{eq:proof_Ssum}, so \(\tilde V_F=(1-B(F))/(2B(F))\).
\end{proof}

\duality*

\begin{proof}
Let \(Z:=F(X)\), so that \(\E[Z]=0\), \(\E[ZZ^\top]=I_r\), and \(\E[Y]=0\). Expanding the squared loss and using these identities gives
\[
\cR(w,b):=\E\big[(Y-(w^\top Z+b))^2\big]
=
1-2w^\top\E[YZ]+\|w\|_2^2+b^2,
\]
a strictly convex quadratic in \((w,b)\) whose unique minimizer is \(w^\star=\E[YZ]\) and \(b^\star=0\). Let \(\mu_s:=\E[Z\mid Y=s]\). Balancedness gives \(\E[YZ]=\tfrac12(\mu_+-\mu_-)=\tfrac12\Delta\) and \(\E[Z]=\tfrac12\mu_++\tfrac12\mu_-=0\), so \(\mu_-=-\mu_+\) and hence \(\Delta=2\mu_+=2w^\star\).
\end{proof}

\nccBound*

\begin{proof}
If \(B(F)=0\), the stated bound is trivial since its right-hand side is at least \(1\). Assume henceforth that \(B(F)>0\). Let \(Z:=F(X)\) and \(w:=\E[YZ]\), so \(\E[Z]=0\), \(\E[ZZ^\top]=I_r\), and \(\|w\|_2^2=B(F)\). Balancedness and \(\E[Z]=0\) give \(\mu_\pm:=\E[Z\mid Y=\pm1]=\pm w\), hence \(\Delta=2w\) and \(\|\Delta\|_2^2=4B(F)\). As in the proof of Prop.~\ref{prop:Br_probe_dcdnv}, the law of total covariance yields
\begin{equation}
\label{eq:sigma_sum_generic}
\Sigma_++\Sigma_-=2(I_r-ww^\top),
\end{equation}
where \(\Sigma_s:=\Cov(Z\mid Y=s)\); taking traces gives \(\Tr(\Sigma_+)+\Tr(\Sigma_-)=2(r-B(F))\), and contracting with \(u:=w/\|w\|_2\) gives \(u^\top(\Sigma_++\Sigma_-)u=2(1-B(F))\).

{\bf NCC rule as a margin.\enspace} Draw an \(m\)-shot support set with empirical class means \(\hat\mu_\pm:=\tfrac1m\sum_{i=1}^m Z_{\pm,i}\), and set \(\hat w:=(\hat\mu_+-\hat\mu_-)/2\) and \(c:=(\hat\mu_++\hat\mu_-)/2\). The NCC rule is \(g_{S_m}^{\mathrm{NCC}}(z)=\sign(\langle\hat w,z-c\rangle)\), since its decision boundary is the perpendicular bisector of \(\hat\mu_\pm\). Define the signed margin \(M_{S_m}:=2Y\langle\hat w,Z-c\rangle\). Regardless of how ties are broken, \(\{g_{S_m}^{\mathrm{NCC}}(Z)\neq Y\}\subseteq\{M_{S_m}\le 0\}\).

{\bf Cantelli bound conditional on the support set.\enspace} Since \((Z,Y)\) is independent of \(S_m\) with \(\E[YZ]=w\), \(\E[Z]=0\), \(\E[ZZ^\top]=I_r\),
\[
\E[M_{S_m}\mid S_m]=2\langle\hat w,w\rangle,
\qquad
\Var(M_{S_m}\mid S_m)=4\big(\|\hat w\|_2^2+\langle\hat w,c\rangle^2-\langle\hat w,w\rangle^2\big).
\]
On the event \(G:=\{\langle\hat w,w\rangle>0\}\), Cantelli's inequality gives \(\Prob(M_{S_m}\le 0\mid S_m)\le 1-f(S_m)\), where
\[
f(S_m):=
\begin{cases}
\dfrac{\langle\hat w,w\rangle^2}{\|\hat w\|_2^2+\langle\hat w,c\rangle^2}, & \hat w\neq0,\\[1mm]
0, & \hat w=0.
\end{cases}
\]
Since \(0\le f(S_m)\le 1\), bounding \(f(S_m)\mathbf 1_G\ge f(S_m)-\mathbf 1_{G^c}\) and taking expectations gives
\begin{equation}
\label{eq:key_reduce}
\err_m^{\mathrm{NCC}}(F)\le 1-\E[f(S_m)]+\Prob(G^c).
\end{equation}

{\bf Lower bound on \(\E[f(S_m)]\).\enspace} When \(\hat w\neq 0\), set \(p:=\langle\hat w,w\rangle^2/\|\hat w\|_2^2=\|P_{\mathrm{span}(\hat w)}w\|_2^2\); then \(f(S_m)=p/(1+\langle\hat w/\|\hat w\|_2,c\rangle^2)\). The inequality \(x/(1+y)\ge x-y\) for \(0\le x\le 1\), \(y\ge 0\), combined with \(p\le\|w\|_2^2=B(F)\le 1\), yields \(f(S_m)\ge p-\|c\|_2^2\). Bounding the distance from \(w\) to \(\mathrm{span}(\hat w)\) by \(\|w-\hat w\|_2\) gives \(p\ge B(F)-\|\hat w-w\|_2^2\); when \(\hat w=0\) the inequality \(f(S_m)\ge B(F)-\|\hat w-w\|_2^2-\|c\|_2^2\) still holds since then \(\|w-\hat w\|_2^2=B(F)\). Writing \(\hat w-w=(\varepsilon_+-\varepsilon_-)/2\) and \(c=(\varepsilon_++\varepsilon_-)/2\) with \(\varepsilon_\pm:=\hat\mu_\pm-\mu_\pm\), class independence gives
\[
\E\|\hat w-w\|_2^2=\E\|c\|_2^2=\frac{\Tr(\Sigma_+)+\Tr(\Sigma_-)}{4m}=\frac{r-B(F)}{2m},
\]
so \(\E[f(S_m)]\ge B(F)-(r-B(F))/m\).

\emph{Bound on \(\Prob(G^c)\).} Since \(\langle\hat w,w\rangle=B(F)+\langle\hat w-w,w\rangle\), we have \(\E[\langle\hat w,w\rangle]=B(F)\), and
\[
\Var(\langle\hat w,w\rangle)=w^\top\Cov(\hat w-w)w=\frac{B(F)}{4m}\,u^\top(\Sigma_++\Sigma_-)u=\frac{B(F)(1-B(F))}{2m}
\]
using \eqref{eq:sigma_sum_generic}. Cantelli gives \(\Prob(G^c)\le(1-B(F))/(1-B(F)+2mB(F))\).

Combining this with \eqref{eq:key_reduce} and the lower bound on \(\E[f(S_m)]\) yields \(\err_m^{\mathrm{NCC}}(F)\le 1-B(F)+(r-B(F))/m+(1-B(F))/(1-B(F)+2mB(F))\).
\end{proof}

\rectangle*

\begin{proof}
Write \(F=(f_1,\dots,f_r)\). Centering and whitening make \(\{f_1,\dots,f_r\}\) orthonormal in \(L^2(P_X)\), so \(\cS_F:=\mathrm{span}\{f_1,\dots,f_r\}\) is \(r\)-dimensional with orthogonal projector \(\Pi_{\cS_F}\). For each task \(t\), let \(w_t:=\E[Y_tF(X)]\in\R^r\) with coordinates \((w_t)_\ell=\E[\eta_t(X)f_\ell(X)]=\langle\eta_t,f_\ell\rangle\), so \(\Pi_{\cS_F}\eta_t=\sum_\ell(w_t)_\ell f_\ell\) and hence
\begin{equation}
\label{eq:proof_wt_BF}
\|w_t\|_2^2=\|\Pi_{\cS_F}\eta_t\|_{L^2}^2=B^{(t)}(F).
\end{equation}
Write \(\eta_t^\perp:=(I-\Pi_{\cS_F})\eta_t\), so \(\|\eta_t^\perp\|_{L^2}^2=\|\eta_t\|_{L^2}^2-B^{(t)}(F)\).

{\bf Near-orthogonality.\enspace} By \eqref{eq:proof_wt_BF}, \(u_i^\top u_j=\langle\Pi_{\cS_F}\eta_i,\Pi_{\cS_F}\eta_j\rangle/\sqrt{B^{(i)}(F)B^{(j)}(F)}\). Since \(\langle\Pi_{\cS_F}\eta_i,\Pi_{\cS_F}\eta_j\rangle=\langle\eta_i,\eta_j\rangle-\langle\eta_i^\perp,\eta_j^\perp\rangle\), Cauchy--Schwarz on the residual term gives
\begin{equation}
\label{eq:proof_inner_split}
\big|\langle\Pi_{\cS_F}\eta_i,\Pi_{\cS_F}\eta_j\rangle\big|
\le
|\langle\eta_i,\eta_j\rangle|+\sqrt{(\|\eta_i\|_{L^2}^2-B^{(i)}(F))(\|\eta_j\|_{L^2}^2-B^{(j)}(F))}.
\end{equation}
To bound the first term, set \(\xi_t:=Y_t-\eta_t(X)\); then \(\E[\xi_t\mid X]=0\) and \(\E[\xi_t^2]=\varepsilon_t\). Expanding \(Y_iY_j=(\eta_i+\xi_i)(\eta_j+\xi_j)\) and using conditional mean zero on the cross terms, \(\rho_{ij}=\E[\eta_i(X)\eta_j(X)]+\E[\xi_i\xi_j]\), so \(|\langle\eta_i,\eta_j\rangle|\le|\rho_{ij}|+\sqrt{\varepsilon_i\varepsilon_j}\) by Cauchy--Schwarz. Substituting into \eqref{eq:proof_inner_split} yields the stated bound on \(|u_i^\top u_j|\); no independence among the \(Y_t\) was used.

{\bf Centroid geometry.\enspace} Let \(Z(X)=(Z_1(X),\dots,Z_k(X))\) and \(m_Y^{(t)}:=\E[Z_t(X)\mid Y]\). By \eqref{eq:proof_wt_BF}, \(\E[Y_tZ_t(X)]=u_t^\top w_t=\|w_t\|_2=\sqrt{B^{(t)}(F)}\), and iterated expectation gives \(\E[Y_tm_Y^{(t)}]=\sqrt{B^{(t)}(F)}\). Expanding the square and using \(Y_t^2=1\),
\[
\E\big[(m_Y^{(t)}-\sqrt{B^{(t)}(F)}\,Y_t)^2\big]
=\E[(m_Y^{(t)})^2]-B^{(t)}(F)
\le\E[Z_t^2]-B^{(t)}(F)
=1-B^{(t)}(F),
\]
where Jensen gives the inequality and whitening gives \(\E[Z_t^2]=u_t^\top I_r u_t=1\); this step also used no joint structure of the \(Y_t\). Summing over \(t\) yields the claimed centroid bound.
\end{proof}

\posterior*

\begin{proof}
Write \(F=(f_1,\dots,f_r)\). The constraints
\(\E[F(X)]=0\) and \(\E[F(X)F(X)^\top]=I_r\) are equivalent to requiring
\(\{f_1,\dots,f_r\}\) to be an orthonormal family in \(L_0^2(P_X)\).
Using \eqref{eq:app_alignment_quadratic_form}, the objective is therefore
\[
\E\!\left[\langle F(X^{(1)}),F(X^{(2)})\rangle\right]
=
\sum_{\ell=1}^r \langle f_\ell,Tf_\ell\rangle_{L^2(P_X)}.
\]

Expand \(f_\ell=\sum_{j\ge1}a_{\ell j}\psi_j\) in the orthonormal eigenbasis of
\(T\), and define \(s_j:=\sum_{\ell=1}^r a_{\ell j}^2\). Since
\(T\psi_j=\lambda_j\psi_j\),
\[
\sum_{\ell=1}^r \langle f_\ell,Tf_\ell\rangle
=
\sum_{j\ge1}\lambda_j s_j,
\qquad
0\le s_j\le1,
\qquad
\sum_{j\ge1}s_j=r.
\]
The inequalities follow because the coefficient matrix
\(A=(a_{\ell j})\) has orthonormal rows. Since
\(\lambda_1\ge\lambda_2\ge\cdots\), the maximum is attained by
\(s_1=\cdots=s_r=1\) and \(s_j=0\) for \(j>r\), with optimal value
\(\sum_{j=1}^r\lambda_j\).

Taking \(F_0=(\psi_1,\dots,\psi_r)\) attains this value, so
\(\psi_1,\dots,\psi_r\) are feasible and optimal. Finally, the eigengap
\(\lambda_r>\lambda_{r+1}\) makes the maximizing allocation unique.
Hence any optimal \(F\) has coordinate span
\(\mathrm{span}\{\psi_1,\dots,\psi_r\}\).
\end{proof}

\decomp*

\begin{proof}
By Prop.~\ref{prop:same_instance_ssl_posterior}, \(\cS_{F^\star}=\Ur\). Since \((\psi_j)_{j\ge1}\) is an orthonormal eigenbasis of \(T\) on \(L_0^2(P_X)\) and \(\eta\in L_0^2(P_X)\), write \(\eta=\sum_{j\ge1}\beta_j\psi_j\), where \(\beta_j=\langle\eta,\psi_j\rangle_{L^2(P_X)}\). Hence \(\Pi_r\eta=\sum_{j=1}^r\beta_j\psi_j\), and orthonormality gives \(B(F^\star)=\|\Pi_r\eta\|_{L^2(P_X)}^2=\sum_{j=1}^r\beta_j^2\).
\end{proof}

\subsection{Proofs for Appendix~\ref{app:soft_whitening}}

\regStructure*

\begin{proof}
Let \(G=\E[F(X)F(X)^\top]\). Both terms in \eqref{eq:ssl_pop_obj_reg} are invariant under orthogonal changes of feature coordinates, so diagonalize \(G\) and write the resulting coordinates as \(g_j=\sqrt{t_j}\,h_j\), where \(t_j\ge0\) and the nonzero \(h_j\)'s are orthonormal in \(L_0^2(P_X)\); complete them arbitrarily to an orthonormal family when some \(t_j=0\). Reorder so that \(t_1\ge\cdots\ge t_r\). The objective becomes
\[
\sum_{j=1}^r t_j\langle h_j,Th_j\rangle
-\frac{\gamma}{2}\sum_{j=1}^r(t_j-1)^2.
\]
By the weighted Ky Fan principle, \(\sum_j t_j\langle h_j,Th_j\rangle\le\sum_j t_j\lambda_j\). Hence the objective is at most
\[
\sum_{j=1}^r\left[t_j\lambda_j-\frac{\gamma}{2}(t_j-1)^2\right],
\]
whose \(j\)-th term is uniquely maximized at \(t_j=1+\lambda_j/\gamma\). Equality is attained by taking \(h_j=\psi_j\). The eigengap fixes the maximizing \(r\)-dimensional subspace, while rotations within eigenspaces of repeated eigenvalues are absorbed into an orthogonal output matrix. Thus every maximizer has the stated form with \(s_j^2=t_j=1+\lambda_j/\gamma\).
\end{proof}

\regCDNV*

\begin{proof}
Let \(D=\mathrm{diag}(s_1^2,\dots,s_r^2)\) and \(b=(s_1\beta_1,\dots,s_r\beta_r)^\top\). Then \(\Cov(F^\star_\gamma(X))=RDR^\top\), while balancedness and centering give \(\mu_+=Rb\), \(\mu_-=-\mu_+\), and \(\|\mu_+\|_2^2=B_\gamma\). By total covariance,
\[
\Sigma_++\Sigma_-=2\big(RDR^\top-\mu_+\mu_+^\top\big).
\]
Taking traces and using \(\|\Delta\|_2^2=4B_\gamma\) yields \(V_{F^\star_\gamma}=\tfrac12(S_\gamma/B_\gamma-1)\). For \(u=\mu_+/\sqrt{B_\gamma}\), we have \(u^\top RDR^\top u=C_\gamma/B_\gamma\), so contraction along \(u\) gives \(\tilde V_{F^\star_\gamma}=\tfrac12(C_\gamma/B_\gamma^2-1)\). The case \(B_\gamma=0\) follows from the stated convention.
\end{proof}

\end{document}